\documentclass[lettersize,journal]{IEEEtran}

\usepackage{tro}
\usepackage{acronym}

\usepackage{subfig}
\usepackage[short]{optidef}
\usepackage{amsmath, amssymb, amsfonts, mathtools}
\usepackage{tcolorbox}
\usepackage{tikz}
\usepackage{colortbl}   
\usepackage{float}
\usepackage{caption} 
\usepackage{cuted}
\usepackage{threeparttable} 

\usepackage{booktabs}   
\usepackage{threeparttable} 

\usetikzlibrary{shapes.geometric,shapes.misc}
\newcommand{\backgroundcolor}{\rowcolor[gray]{0.9}}

\newacro{KKT}{Karush-Kuhn-Tucker}
\newacro{NLP}{nonlinear program}
\newacro{QP}{quadratic program}
\newacro{LP}{linear program}
\newacro{CP}{convex program}
\newacro{NLP}{nonlinear program}
\newacro{IOC}{inverse optimal control}
\newacro{NMPC}{nonlinear model predictive control}
\newacro{SQP}{sequential quadratic programming}
\newacro{OCP}{optimal control problem}
\newacro{BB}{branch-and-bound}
\newacro{MPC}{model predictive control}
\newacro{FONC}{first order necessary condition}
\newacro{MDP}{Markov decision process}
\newacro{IP}{interior point}
\newacro{LOS}{line of sight}
\newacro{PnP}{perspective-n-point}
\newacro{MPCTC}{model predictive contour-tracking control}
\newacro{LA}{look-ahead}
\newacro{FOV}{field-of-view}
\newacro{PUM}{position uncertainty minimization}
\newacro{IMU}{inertial measurement units}
\newacro{FCU}{flight control unit}
\newacro{UAV}{uncrewed aerial vehicles}
\newacro{FPV}{first-person-view}
\newacro{SOTA}{state of the art}
\newacro{TOGT}{time-optimal gate-traversing}
\newacro{TOWP}{time-optimal waypoint-passing}
\newacro{MPCC}{model predictive contouring control}
\newacro{PDOP}{position dilution of precision}
\newacro{CRLB}{Cramer-Rao lower bound}
\newacro{FIM}{fisher information matrix}
\newacro{PDOP}{position dilution of precision}
\newacro{EKF}{extended Kalman filter}
\newacro{RL}{reinforcement learning}
\newacro{MOCAP}{motion capture}

\newtcolorbox{algbox}[2][]{%
	colframe=gray!15, 
	colback=gray!05, 
	coltitle=black, 
	fonttitle=\bfseries,
	colupper=black, 
	title=#2, 
	arc=1mm,
	boxrule=0.0mm, 
	left=0pt,
	right=0pt,
	top=-5pt,
	bottom=0pt,
	#1 
}

\IEEEtitleabstractindextext{
  \begin{center}
    \centering
    \includegraphics[width=\textwidth]{your_figure}
    \captionof{figure}{Your teaser image caption.}
    \label{fig:teaser}
  \end{center}

  \begin{abstract}
    Your abstract text goes here.
  \end{abstract}
}

\begin{document}


\title{Perception-Aware Time-Optimal Planning for Quadrotor Waypoint Flight}


\author{Chao Qin$^{1}$, Jiaxu Xing$^{2}$, Rudolf Reiter$^{2}$, Angel Romero$^2$, Yifan Lin$^{1}$, Hugh H.-T. Liu$^{1}$ and Davide Scaramuzza$^{2}$
\thanks{$^{1}$ 
Authors are with the University of Toronto Institute for Aerospace Studies (UTIAS), Canada}
\thanks{$^{2}$ The Robotics and Perception Group, University of
Zurich, Switzerland (https://rpg.ifi.uzh.ch).}
}

\markboth{IEEE Journal of \LaTeX\ Class Files, February~2026}%
{Shell \MakeLowercase{\textit{et al.}}: A Sample Article Using IEEEtran.cls for IEEE Journals}


\maketitle

\begin{strip}
    \centering
    \vspace{-60pt} 
    \begin{minipage}{\textwidth}
        \centering
        \tikzstyle{every node}=[font=\footnotesize]
        \begin{tikzpicture}[>=stealth]
            \node [inner sep=0pt, outer sep=0pt] (img1) at (0,0) {\includegraphics[width=5.8cm]{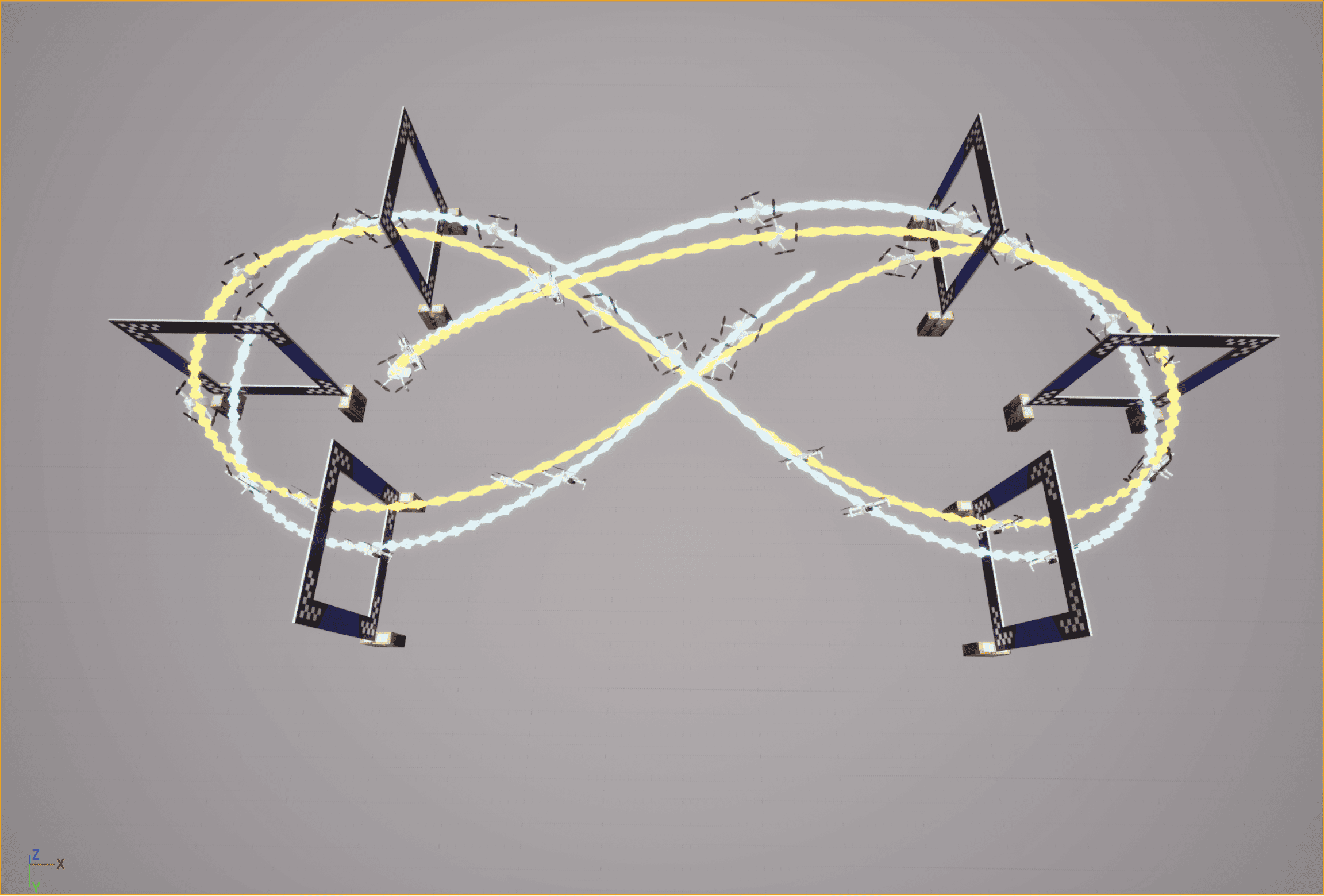}};
            \node [inner sep=0pt, outer sep=0pt, right=1mm of img1] (img2) {\includegraphics[width=5.8cm]{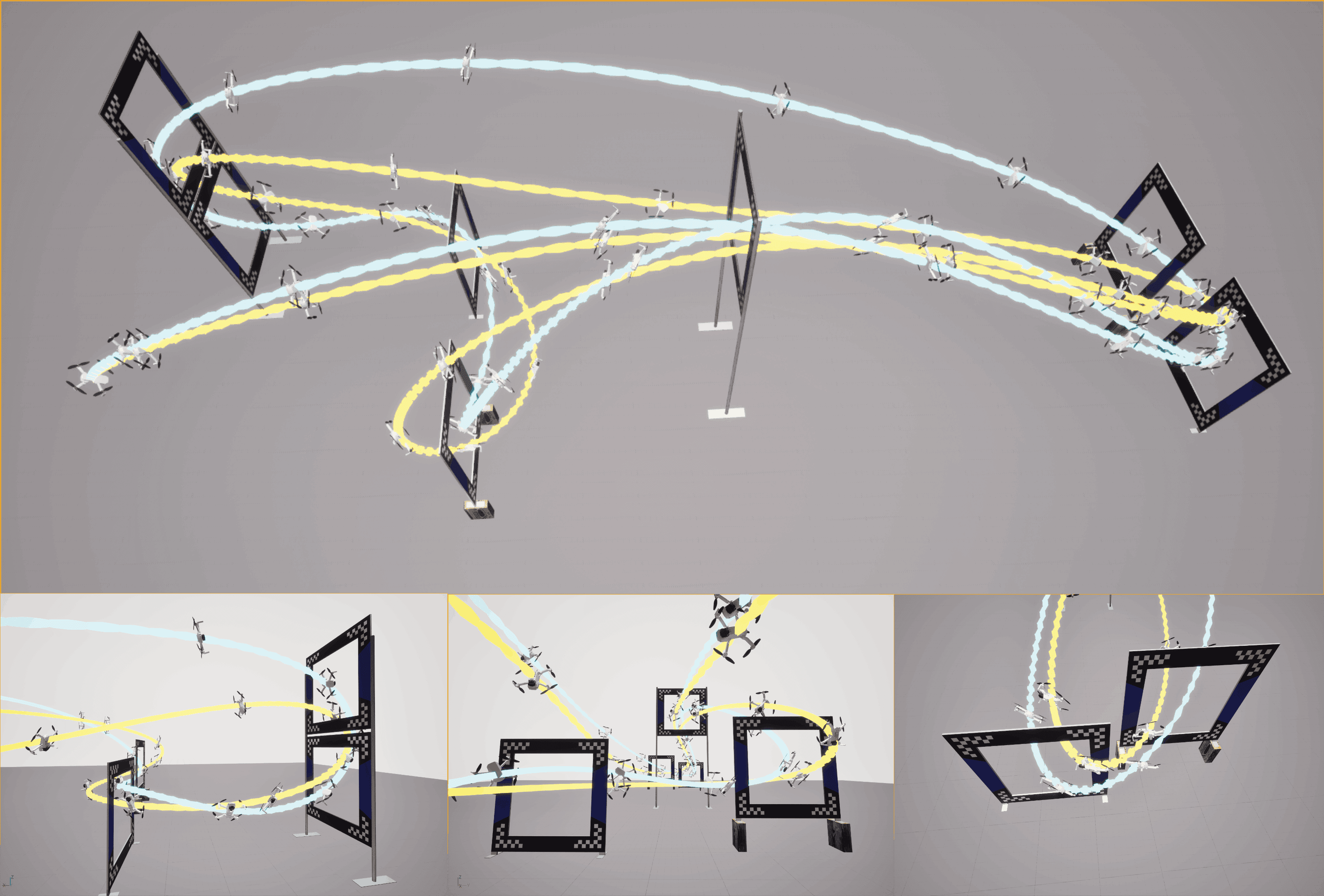}};
            \node [inner sep=0pt, outer sep=0pt, right=1mm of img2] (img3) {\includegraphics[width=5.8cm]{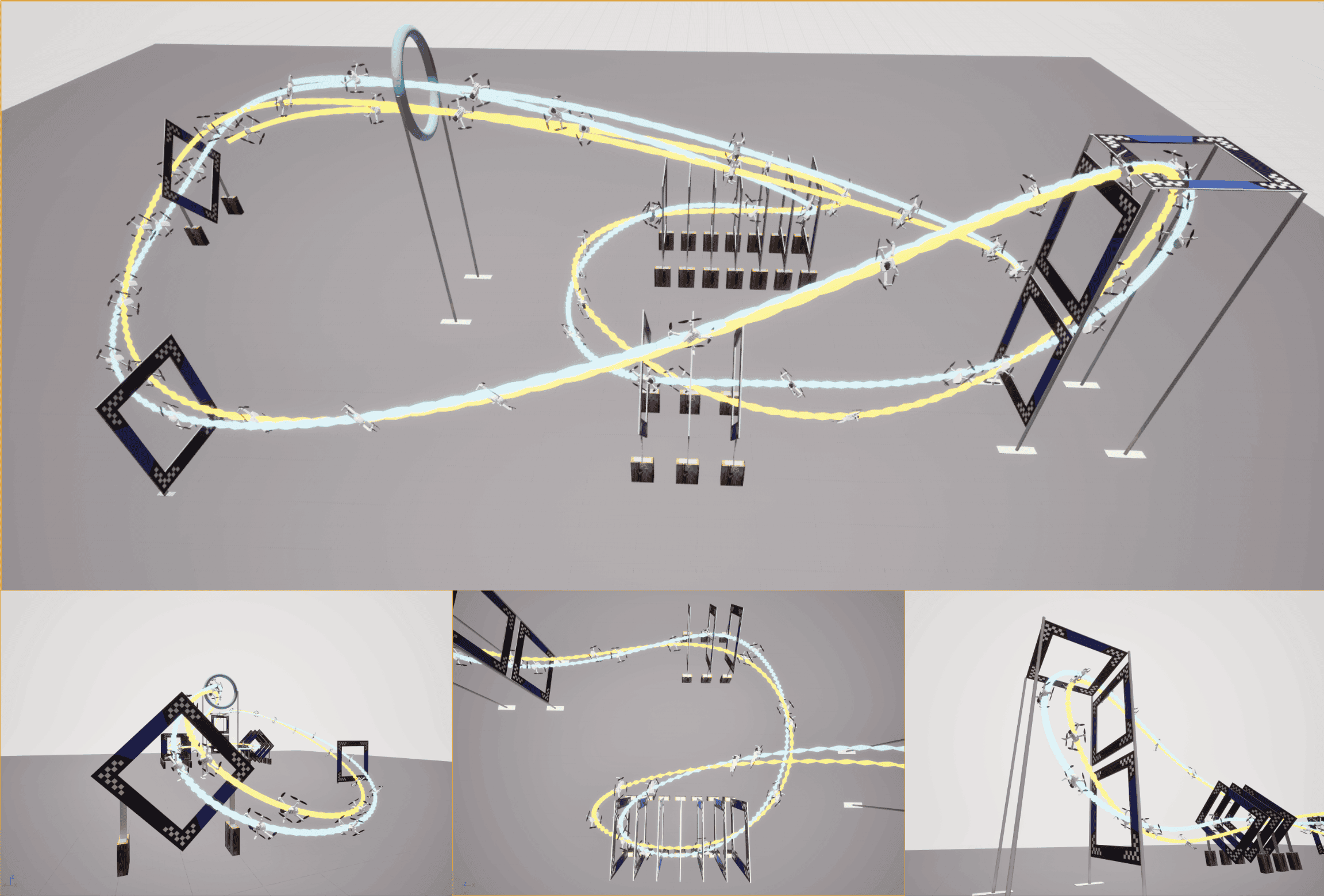}};

            \node[below=1mm of img1](text1){(a) Figure-8 Track \cite{xing2024multi}};
            \node[below=1mm of img2](text2){(b) TII Track \cite{bosello2025your}};
            \node[below=1mm of img3](text3){(c) Dive Track};
        \end{tikzpicture}
        \captionof{figure}{Our framework addresses the perception-aware, time-optimal planning problem for agile quadrotors. It supports both \ac{TOWP} (blue) and \ac{TOGT} (yellow) modes on a wide range of race tracks, and three perception objectives are proposed to ensure quality visual inputs. Tracks (a) and (b) demonstrate that our method can handle tracks defined in previous literature, while track (c) illustrates its superior flexibility in handling more complex scenarios where gates are in diverse shapes, sizes, and orientations.}
        \label{fig:figure8_tii_dive}
    \end{minipage}

\end{strip}



\begin{abstract}

Agile quadrotor flight pushes the limits of control, actuation, and onboard perception. While time-optimal trajectory planning has been extensively studied, existing approaches typically neglect the tight coupling between vehicle dynamics, environmental geometry, and the visual requirements of onboard state estimation. As a result, trajectories that are dynamically feasible may fail in closed-loop execution due to degraded visual quality. This paper introduces a unified time-optimal trajectory optimization framework for vision-based quadrotors that explicitly incorporates perception constraints alongside full nonlinear dynamics, rotor actuation limits, aerodynamic effects, camera field-of-view constraints, and convex geometric gate representations.
The proposed formulation solves minimum-time lap trajectories for arbitrary racetracks with diverse gate shapes and orientations, while remaining numerically robust and computationally efficient. We derive an information-theoretic position uncertainty metric to quantify visual state-estimation quality and integrate it into the planner through three perception objectives: position uncertainty minimization, sequential field-of-view constraints, and look-ahead alignment. This enables systematic exploration of the trade-offs between speed and perceptual reliability.
To accurately track the resulting perception-aware trajectories, we develop a model predictive contouring tracking controller that separates lateral and progress errors. Experiments demonstrate real-world flight speeds up to 9.8 m/s with 0.07 m average tracking error, and closed-loop success rates improved from 55\% to 100\% on a challenging Split-S course. The proposed system provides a scalable benchmark for studying the fundamental limits of perception-aware, time-optimal autonomous flight.
\end{abstract}

%
\section{Introduction}

Quadrotors are among the most agile \ac{UAV}s and have been widely deployed in inspection, delivery, and scene reconstruction and mapping tasks \cite{xing2023autonomous}. 
These applications benefit from the platform’s high maneuverability, precise control authority, and rich onboard sensing capabilities. 
\ac{FPV} Drone racing represents an extreme operating regime of these same capabilities, pushing both actuation and perception systems to their limits under aggressive maneuvers and strict real-time constraints \cite{hanover2024autonomous}. 
Beyond competition, racing provides a principled benchmark for quantifying the maximum achievable performance of autonomous aerial systems within known environments.

Understanding these performance limits requires solving the time-optimal trajectory planning problem. 
Establishing theoretical minimum lap times enables objective comparison across control strategies, whether model-based, learning-based, or optimization-based. 
In real-world vision-based systems, maximum achievable performance is not determined solely by time-optimal control. 
It is also influenced by perceptual constraints such as the minimum visual quality required for reliable localization and by geometric constraints imposed by gates and tunnels \cite{kaufmann2023champion, xing2024bootstrapping}. 
These coupled effects motivate the concept of \textit{perception-aware flight} \cite{li2021pcmpc,falanga2018pampc,song2022learning}. 
Additionally, environmental constraints such as narrow passages and complex gate geometries must be explicitly considered to guarantee collision-free and effort-efficient flight \cite{qin2023time}. 
Characterizing and optimizing these tightly coupled dynamics–perception–geometry interactions remains challenging due to their high dimensionality and nonlinear structure.

To bridge this gap, we present a unified time-optimal trajectory planning framework that explicitly incorporates the critical factors of vision-based autonomous flight. 
Our formulation jointly accounts for full quadrotor dynamics, rotor actuation limits, aerodynamic effects, camera \ac{FOV} constraints, geometric constraints, and perception objectives. 
As illustrated in Fig.~\ref{fig:figure8_tii_dive}, the proposed framework computes theoretical minimum lap trajectories across diverse racetracks, including those used in prior literature \cite{xing2024multi, bosello2025your}, as well as tracks with significantly more flexible gate configurations. 
Importantly, gates can be modeled as convex polygons or polytopes (Fig.~\ref{fig:gate_shapes}), enabling broad applicability. 
The resulting optimization is numerically stable and computationally efficient. 
To the best of our knowledge, this is the first framework capable of jointly handling all the aforementioned factors within a scalable formulation.

%
%
%

A second major contribution of this work is the formal verification and integration of three core perception objectives for vision-based autonomous aerial vehicles (Fig.~\ref{fig:perceptual_objectives}). 
We derive an information-theoretic metric for vision-based position uncertainty and incorporate it into trajectory optimization through:
\begin{itemize}
    \item \textbf{\ac{PUM}:} minimizing the theoretical lower bound on positioning uncertainty,
    \item \textbf{sequential \ac{FOV} constraints (FOV):} enforcing visibility of sequential targets via tunable soft penalties,
    \item \textbf{\ac{LA}:} aligning camera orientation with future trajectory segments, inspired by human piloting behavior \cite{pfeiffer2021human}.
\end{itemize}
These objectives can be applied individually or in combination to systematically balance speed and perceptual robustness. 
Closed-loop experiments demonstrate that incorporating perception objectives increases the success rate from $55\%$ to $100\%$ on the Split-S track.

%
%

\begin{figure*}[!htbp]
\centering
\tikzstyle{every node}=[font=\footnotesize]
\begin{tikzpicture}[>=stealth]

\node [inner sep=0pt, outer sep=0pt] (img1) at (0,0)
{\includegraphics[width=4.0cm]{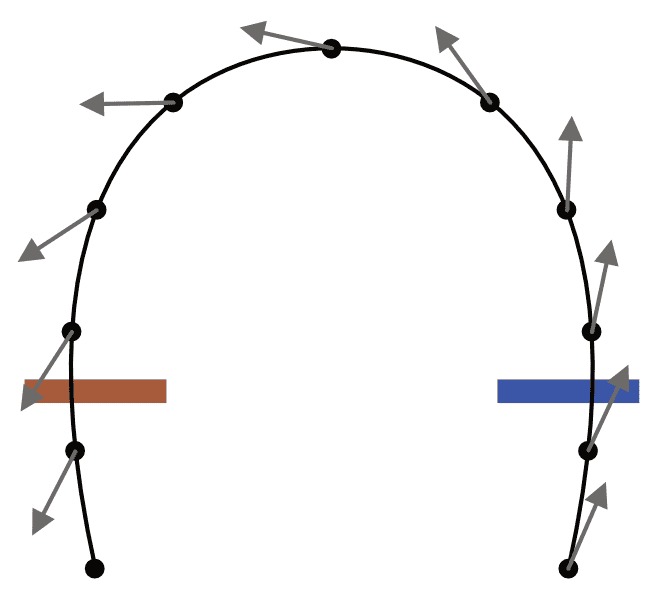}};
\node [inner sep=0pt, outer sep=0pt, right=1mm of img1] (img2) {\includegraphics[width=4.0cm]{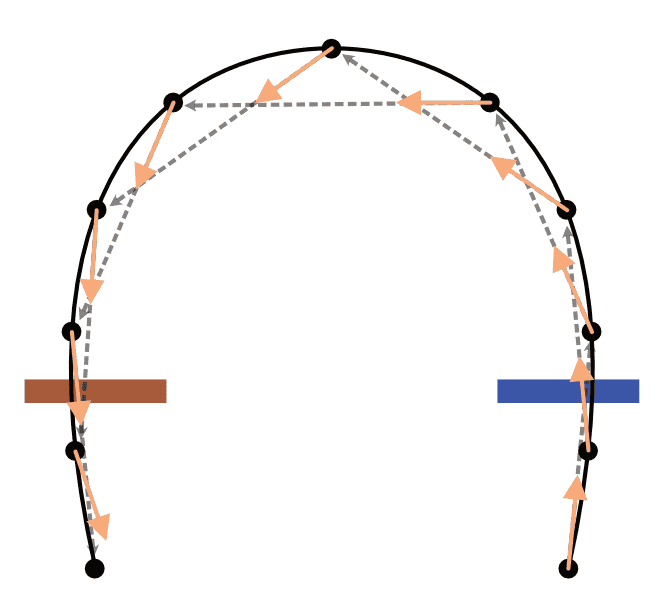}};
\node [inner sep=0pt, outer sep=0pt, right=1mm of img2] (img3) {\includegraphics[width=4.0cm]{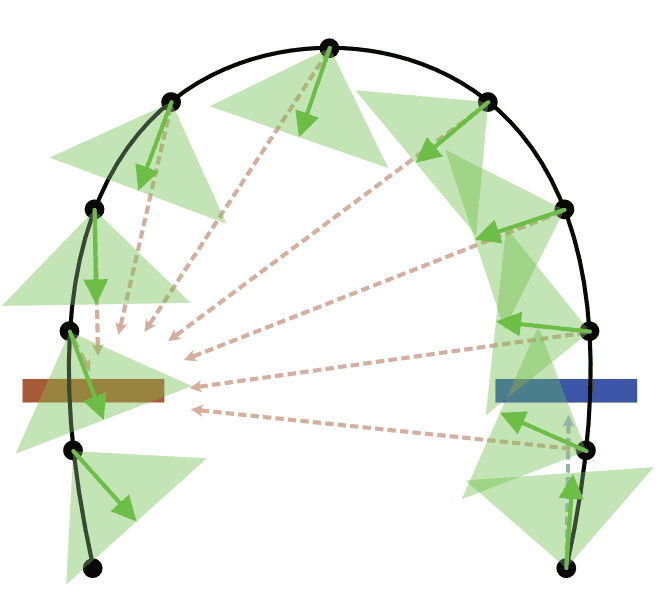}};
\node [inner sep=0pt, outer sep=0pt, right=1mm of img3] (img4) {\includegraphics[width=4.0cm]{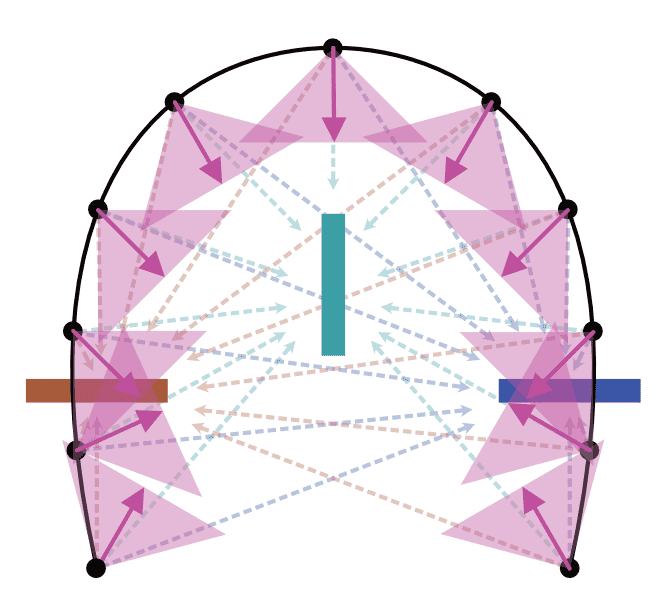}};

\node[below=1mm of img1](text1){(a) None};
\node[below=1mm of img2](text2){(b) LA};
\node[below=1mm of img3](text3){(c) FOV};
\node[below=1mm of img4](text2){(b) PUM};
\end{tikzpicture}
\caption{Illustrations of three perceptual objectives verified in the proposed framework. The
arrows represent the camera optical directions, and the triangles with filled
colors indicate the camera’s \ac{FOV}. (a) In the absence of specific perception objectives, the camera's orientation is purely governed by time-optimal control, which could be detrimental to visual perception; (b) \ac{LA} encourages the \ac{UAV} to look at a future waypoint; (c) \ac{FOV} keeps one or multiple target landmarks (e.g., the next racing gate) visible; (d) \ac{PUM} trades off time optimality and information-theoretic position uncertainty for robust closed-loop execution.}
\label{fig:perceptual_objectives}
\vspace{-1em}
\end{figure*}
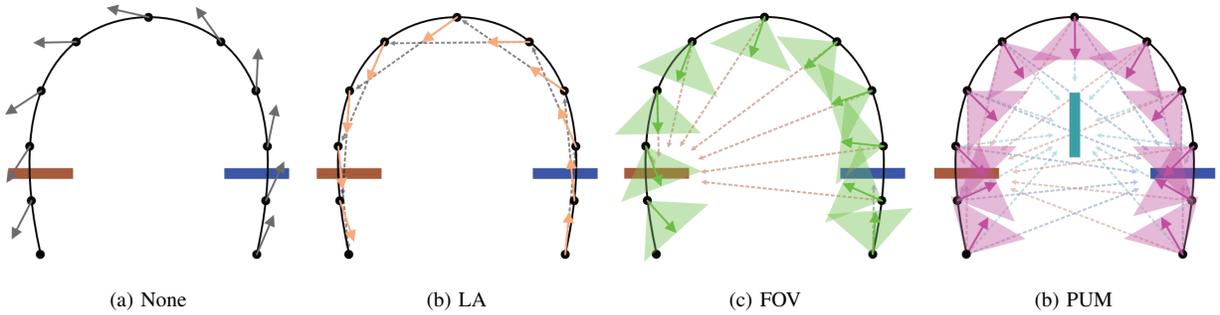

To accurately execute the resulting perception-aware trajectories, we further propose a \ac{MPCTC} that decomposes tracking error into longitudinal and lateral components, allowing independent weighting of path adherence and forward progression and thereby overcoming the corner-cutting problem~\cite{song2023reaching} of a standard \ac{MPC}. 
In real-world experiments at speeds up to $9.8\,\text{m/s}$, the controller achieves an average tracking error of $0.07\,\text{m}$ and a peak error of $0.23\,\text{m}$, ensuring reliable satisfaction of perception constraints.

%

The main contributions of this paper are summarized as follows:
\begin{itemize}
    \item A robust and efficient time-optimal quadrotor trajectory planner that jointly integrates nonlinear dynamics, actuation limits, aerodynamic effects, geometric constraints, and perception objectives.
    \item An information-theoretic derivation of vision-based position uncertainty and three perception-aware objectives that directly enhance visual robustness.
    \item A complete planning and control pipeline enabling accurate, perception-aware, time-optimal flight validated in extensive simulation and real-world experiments.
\end{itemize}

%

The remainder of this paper is structured as follows. 
Section~\ref{sec:related_works} reviews related work. 
Section~\ref{sec:problem_formulation} formulates the problem. 
Section~IV derives the position uncertainty metric. 
Section~\ref{sec:numerical_optimization} presents the trajectory optimization framework. 
Section~\ref{sec:mpctc} introduces the tracking controller. 
Section~\ref{sec:results} reports experimental results. 
Finally, Section~\ref{sec:conclusion} discusses limitations and connections to \ac{RL}.


\begin{figure}[b]
\centering
\includegraphics[width=0.35\textwidth]{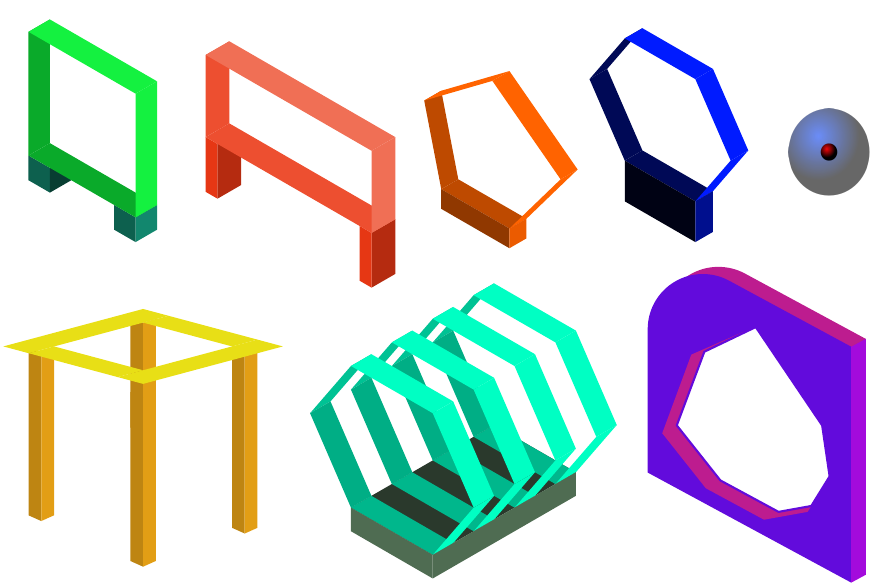}
\caption{The framework accommodates gates representable by convex shapes, which constitute the majority of gates used in drone racing \cite{qin2023time}.}\label{fig:gate_shapes}
\end{figure}

\section{Related Works}\label{sec:related_works}

\subsection{Time-Optimal Trajectory Planning}

There are two mainstream methods to plan time-optimal quadrotor trajectories: direct methods, i.e., discretization combined with \ac{NLP}~\cite{diehl2011numerical}, and polynomial approaches \cite{mellinger2011minimum}, which utilize the differentiable flatness property of quadrotors to formulate a polynomial trajectory.

\textbf{Direct methods.}
The most common formulation of direct methods is via multiple shooting~\cite {bock1984multiple}, which employs discrete-time trajectories to represent the quadrotor's controls and states as decision variables.
Foehn et al. \cite{foehn2021time} present the first time-optimal planner based on a multiple-shooting framework that addresses the quadrotor waypoint-flight problem in drone racing. However, this method is time-consuming, often requiring minutes or even hours to compute results. Zhou et al. \cite{zhou2023efficient} optimize the problem structure by explicitly assigning fixed waypoint constraints to specific shooting nodes, thereby reducing computation time. Shen et al. \cite{shen2024sequential} and the authors in~\cite{dong2026fast} reformulate the original problem into convex relaxations for computational efficiency. 
An alternative approach to enhance numerical tractability involves projecting the system dynamics onto a reference path and maximizing the path progress as a surrogate for time-optimality, as discussed in \cite{spedicato2017minimum, arrizabalaga2022towards, fork2023euclidean}. In addition, by slightly compromising model fidelity, Romero et al. \cite{romero2022time} and Teissing et al. \cite{teissing2024real} achieved real-time performance. The online replanned trajectory can seemingly be incorporated into certain time-optimal controllers \cite{romero2022model, kulic2024mpcc++, ji2020cmpcc}, formulated as \ac{MPCC} problems, thereby enabling dynamic obstacle avoidance. 

\textbf{Polynomial approach.}
The polynomial approach typically offers superior computational efficiency, taking merely seconds or milliseconds of runtime. Mellinger et al. \cite{mellinger2011minimum} introduced the well-known minimum-snap trajectory planner. 
Richter et al. \cite{richter2016polynomial} further eliminated waypoint constraints, significantly reducing computation time.
Meanwhile, Mueller et al. \cite{mueller2013computationally} derived an analytical solution for the coefficients of a polynomial that connects two states. 
Wang et al. \cite{wang2022geometrically} extended this concept to waypoint flight, enabling lightweight computation of dynamically feasible waypoint trajectories for quadrotors. Subsequently, Han et al. \cite{han2021fast}, Qin et al. \cite{qin2023time}, and Wang et al. \cite{wang2023polynomial} extended this framework to tackle various challenges in drone racing, including collision avoidance and the traversal of both static and dynamic gates. However, polynomial approaches often suffer from inherent smoothness, resulting in significantly longer flight times compared to discretization-based methods. To address this issue, Qin et al. \cite{qin2024time} proposed a solution that achieves 1--3\% relative suboptimality to the global optimal time by appropriately increasing the number of polynomial segments. They also provide an analytical study to justify the choice of the segment number.

\textbf{Other methods.}
While direct multiple shooting algorithms use gradient-based optimization techniques, Penicka et al. \cite{penicka2022minimum} introduced a sampling-based framework that accounts for the full quadrotor dynamics. Song et al. \cite{song2021autonomous} presented a deep \ac{RL} based approach to compute near-time-optimal trajectories for a fixed track layout. 
The authors in~\cite{reiter2025unifying} propose a multi-phase MPC with a polynomial model over a second long-term horizon, but consider only tracking problems.

\subsection{Percepion-Aware Flights}
Perception-awareness is crucial for mission success, as perception quality often determines the positioning robustness and precision. Previous literature handles two primary perception objectives. 

\textbf{Target visibility.}
The first objective emphasizes maintaining visibility of specific objects of interest, ensuring that the target remains within the \ac{LOS} \cite{falanga2018pampc}. This is particularly crucial for applications such as target following \cite{ji2022elastic, wang2023svpto, wang2021visibility} and inspection tasks \cite{zhao2024optimized, panigati2025drone, kim2023visibility}. 
For safe navigation in unknown and obstacle-rich environments, the interest is extended to unknown regions that might potentially put the vehicle in danger \cite{yu2022cpa, zhou2021raptor}.
Additionally, the target to keep in sight can be as simple as the next waypoint just to counteract unexpected scenarios \cite{song2023learning}.

\textbf{Uncertainty reduction.}
The other type of perception objective aims to enrich landmark observation from the imaging sensor and thereby reduce localization uncertainty \cite{zhang2020fisher, kaufmann2023champion, lim2023fisher, papachristos2019localization}. 
It was verified that this strategy can effectively mitigate position estimate drift in texture-less areas where effective feature matching is scarce~\cite{bartolomei2020perception}.

\textbf{Perception-aware time-optimal flights.}
When time is of the essence, it is desirable to optimize both trajectory duration and perception quality simultaneously. 
Murali et al. \cite{murali2019perception} proposed an aggressive trajectory planner that optimizes the total trajectory duration while maximizing the co-visibility of multiple known 3D landmarks across consecutive frames. 
However, their objective is not to operate the vehicle at its actuation limit, which results in suboptimal durations. 
Spasojevic et al. \cite{spasojevic2020perception} introduced an efficient time-optimal path parameterization algorithm for quadrotors under limited FOV constraints. 
Their planning is divided into two stages: a path-planning phase in a lower-dimensional space that ensures collision avoidance, followed by time-parameterization that respects the quadrotor's kinodynamic constraints. 
However, this method cannot guarantee that the resulting path is time-optimal in both temporal and spatial profiles, as the latter phase has no access to modify the given path. Additionally, it simplifies the camera's \ac{FOV} constraint as a symmetric conical region, which does not accurately represent real image boundaries. 


%
\section{Problem Formulation}\label{sec:problem_formulation}

This section formulates the time-optimal trajectory optimization problem. We first define the quadrotor's equation of motion and system constraints involving single-rotor thrust limits, geometrical constraints, and FOV constraints of the imaging sensor, followed by introducing three perception objectives.

\subsection{Quadrotor Dynamics}\label{subsec:dynamics}

The quadrotor's state space is described as:
\begin{align}
  \mathbf{x}^{\top}=\left[ {\mathbf{p}^{w}}^{\top},{\mathbf{q}_{wb}}^{\top},{\mathbf{v}^{w}}^{\top},{\boldsymbol{\omega}^{b}}^{\top}, {\mathbf{f}^{b}}^{\top}\right]
\end{align}
where $\mathbf{p}^{w}\in \mathbb{R}^3$ is the position in the world frame, $\mathbf{q}_{wb} \in SO(3)$ is the unit quaternion that describes the orientation of the platform from the body frame to the world frame, $\mathbf{v}^{w}\in \mathbb{R}^3$ is the linear velocity vector, $\boldsymbol{\omega}^{b}\in \mathbb{R}^3$ are the body rates, and $\mathbf{f}^{b}=[f_1, f_2, f_3, f_4]^T\in \mathbb{R}^4$ are the thrusts produced by each rotor. The quadrotor frame configuration is illustrated in Fig. \ref{fig:reference_frames}. We denote the collective thrust as $\mathbf{f}_T$ and the associated body torques as $\boldsymbol{\tau}^b$, which can be obtained by:
\begin{equation}\label{eq:collective_thrust}
\mathbf{f}_{T}\!=\!\left[\begin{array}{c}
  0\\
  0\\
  \sum f_{i}
  \end{array}\right] \text{and } \boldsymbol{\tau}^{b}\!=\!\left[\begin{array}{c}
    l/\sqrt{2}(f_{1}\!+\!f_{2}\!-\!f_{3}\!-\!f_{4})\\
    l/\sqrt{2}(\!-\!f_{1}\!+\!f_{2}\!+\!f_{3}\!-\!f_{4})\\
    c_{\tau}(f_{1}-f_{2}+f_{3}-f_{4})
    \end{array}\right]
\end{equation}
with the quadrotor's arm length $l$ and the rotor's torque constant $c_{\tau}$. For readability, we drop the frame indices (e.g., the superscript $w$ and $b$) as they are consistent throughout the description. In addition, we adopt a linear drag model \cite{faessler2017differential} to emulate the aerodynamic drags. Let $\mathbf{D}$ be the diagonal drag coefficients matrix. The equation of motion can be written as:
\begin{align}\label{eq:quadrotor_dynamics}
\begin{array}{ll}
  \dot{\mathbf{p}}=\mathbf{v} & \dot{\mathbf{q}}=\frac{1}{2}\boldsymbol{\Lambda}(\mathbf{q})\begin{bmatrix}\boldsymbol{\omega}\\
  0
  \end{bmatrix}\\
  \rule{0pt}{4ex}
 \dot{\mathbf{v}}=\frac{1}{m}\mathbf{R}\mathbf{f}_{T}\!+\!\mathbf{g}\!-\!\mathbf{R}\mathbf{D}\mathbf{R}^{\top}\mathbf{v} & \dot{\boldsymbol{\omega}}=\mathbf{J}^{-1}\left(\boldsymbol{\tau}-\boldsymbol{\omega}\times\mathbf{J}\boldsymbol{\omega}\right)
  \end{array}
\end{align}  
where $\boldsymbol{\Lambda}(\mathbf{q})$ returns the quaternion's multiplication matrix, and $\mathbf{R}$ is the corresponding rotation matrix of $\mathbf{q}$. In addition, $\mathbf{g}$ represents the gravity vector, $m$ the quadrotor's mass, and $\mathbf{J}$ its inertia matrix. The input of the system is given as the rates of change of the thrusts produced by each rotor, $\mathbf{r}_T=\dot{\mathbf{f}}^{b}\in \mathbb{R}^4$.

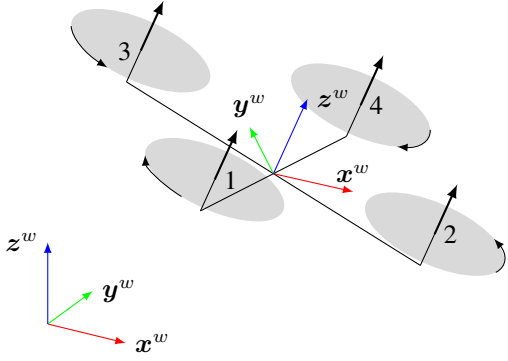
\begin{figure}[t!]
    \centering
    \tdplotsetmaincoords{65}{30}
\begin{tikzpicture}[tdplot_main_coords, >=latex]
\tikzset{RPY/.style={x={(\nxx,\nxy)},y={(\nyx,\nyy)},z={(\nzx,\nzy)}}}
\def\propz{0.5}
\def\propx{2.4}
\def\propy{2.2}
\def\r{1}
\def\l{1.2}
\def\f{0.8}
\def\c{0.8}
\def\x{0}
\rotateRPY{-10}{20}{0}
\begin{scope}[RPY]
\draw[] (-\propx-\x,0,0) -- (\propx-\x,0,0);
\draw[] (-\x,-\propy,0) -- (-\x,\propy,0);
\draw[draw=none, fill=black, opacity=0.15] (0-\x,\propy,\propz) circle (\r);
\draw[draw=none, fill=black, opacity=0.15] (\propx-\x,0,\propz) circle (\r) ;
\draw[draw=none, fill=black, opacity=0.15] (0-\x,-\propy,\propz) circle (\r) ;
\draw[draw=none, fill=black, opacity=0.15] (-\propx-\x,0,\propz) circle (\r) ;
\draw (-\x,\propy,0) --++ (0,0,\propz) node[right] {4};
\draw (\propx -\x,0,0) --++ (0,0,\propz) node[right] {2};
\draw (-\x,-\propy,0) --++ (0,0,\propz) node[right] {1};
\draw (-\propx -\x,0,0) --++ (0,0,\propz) node[left] {3};
\draw[thick,->] (-\x,\propy,\propz) -- ++(0,0,\f);
\draw[thick,->] (\propx-\x,0,\propz) -- ++(0,0,\f);
\draw[thick,->] (-\x,-\propy,\propz) -- ++(0,0,\f);
\draw[thick,->] (-\propx-\x,0,\propz) -- ++(0,0,\f);    

\def\psi{60}
\def\xa{30}
\def\xb{0}
\def\xc{-90}
\def\xd{210}
\path (\r-\x,\propy,\propz) arc (0:\xa:\r) coordinate (a);
\draw [->] (a) arc (\xa:\xa-\psi:\r) [draw=black] ;
\path (\propx-\x,\r,\propz) arc (90:\xb:\r) coordinate (b);
\draw [->] (b) arc (\xb:\xb+\psi:\r) [draw=black] ;
\path (\r-\x,-\propy,\propz) arc (0:\xc:\r) coordinate (c);
\draw [->] (c) arc (\xc:\xc-\psi:\r) [draw=black] ;
\path (-\propx-\x,\r,\propz)  arc (90:\xd:\r) coordinate (d);
\draw [->] (d) arc (\xd:\xd+\psi:\r) [draw=black] ;
\rotateRPY{0}{0}{45}
\begin{scope}[RPY]
\draw[->,color=red,text=black] (-\x,0,0) -- ++ (\l,0,0) node[above] {$\bm x^w$};
\draw[->,color=green,text=black] (-\x,0,0) -- ++ (0,\l,0) node[above] {$\bm y^w$};  
\draw[->,color=blue,text=black] (-\x,0,0) -- ++ (0,0,\l) node[right] {$\bm z^w$};     
\rotateRPY{0}{25}{0}

\end{scope}
\end{scope}

\begin{scope}[xshift=-3cm, yshift=-2cm]
\draw[->,color=red,text=black] (-\x,0,0) -- ++ (\l,0,0) node[right] {$\bm x^w$};
\draw[->,color=green,text=black] (-\x,0,0) -- ++ (0,\l,0) node[right] {$\bm y^w$};  
\draw[->,color=blue,text=black] (-\x,0,0) -- ++ (0,0,\l) node[left] {$\bm z^w$};
\end{scope}

\end{tikzpicture}
    \caption{Diagram of quadrotor model with the world and
body frames convention.}
    \label{fig:reference_frames}
    \vspace{-1.5em}
\end{figure}

\subsection{Constraints}

\subsubsection{Body-Rate Constraints}
The body rates are bounded by the maximum angular velocities of the operational limits of the low-level \ac{FCU} via
$|\boldsymbol{\omega}|\leq\boldsymbol{\omega}_{\text{max}}.$

\subsubsection{Rotor Constraints}
The thrust produced by each rotor is constrained by the maximum and minimum thrust limits, which are determined by the motor's specifications with
$\mathbf{f}_{\text{min}}\leq\mathbf{f}\leq\mathbf{f}_{\text{max}}.$
In addition, the rate of thrust change~$\mathbf{r}_{T}$ is bounded to account for the low-pass characteristic of the motors:
$\mathbf{r}_{T_{\text{min}}} \leq \mathbf{r}_{T} \leq \mathbf{r}_{T_{\text{max}}}.$

\subsubsection{Waypoint Constraints}
Some tasks explicitly specify waypoints to pass in a certain order, which can be modeled as a spherical region with a user-specified radius, $\delta_{i}$, that indicates the tolerance for waypoint constraint violation. For the $i$-th waypoint, the constraint can be expressed as:
\begin{equation}
\mathcal{W}_{i}(\delta_{i}) = \{\mathbf{p} \in \mathbb{R}^{3} \mid \|\mathbf{p} - \mathbf{p}_{w_i}\|_{2} \leq \delta_{i}\},
\end{equation}
Consider a problem with $N_{w_i}$ waypoints in total. The waypoint constraint $h_{\mathcal{G}}(\mathbf{p}) \leq 0$ suggests that the quadrotoy trajectory should fall into a functional space that satisfies:
\begin{equation}
\mathbf{p}(t_{i}) \in \mathcal{W}_{i}(\delta_{i}), \quad i = 1, 2, \ldots, N_{w_i},
\end{equation}
at some time instances $0 < t_{1} < t_{2} < \cdots < t_{N_{w_i}}$. Note that this formulation can also apply to circular gate modeling in drone racing, simply by setting the tolerance to the radius of the gate.

\subsubsection{Gate Constraints}

In many scenarios, the quadrotor is tasked with traversing a specific object or area, instead of a single point. Since the traversable region is significantly larger, this formulation provides greater flexibility for trajectory optimization. We collectively refer to these traversable areas of certain objects as gate constraints. Two general classes of gates are considered: polyhedral and spherical. Polyhedral gates are defined as
\begin{equation}
\mathcal{P}_{i}=\{\mathbf{p}\in\mathbb{R}^{3}|\;\mathbf{A}_{i}\mathbf{p}\leq\mathbf{b}_{i}\},
\end{equation}
where $\mathbf{A}_{i}$ and $\mathbf{b}_{i}$ represent the parameters of the polyhedron.

We denote the full gate constraint as $h_{\mathcal{G}}(\mathbf{p}) \leq 0$, which defines a feasible function space that meets:
\begin{equation}
\mathbf{p}(t_{i})\in\mathcal{P}_{i},\quad i = 1, 2, \ldots, N_{w_i}
\end{equation}
for some time instances $0 < t_{1} < t_{2} < \cdots < t_{N_{w_i}}$.

This formulation is also applicable to non-convex objects and environments that a sequence of convex polytopes can represent. A typical example is a winding tunnel, which can be modeled using a chain of rectangular prisms.

\subsubsection{Target Visibility Constraints}
In missions with inspection needs, keeping the targets within the camera's \ac{FOV} constantly is critical. The target here is treated as a 3D point, which typically represents the center of the object of interest. A straightforward implementation involves imposing \ac{FOV} constraints on the camera, thereby restricting the vehicle's pose. Assume that the target landmark's position in the world frame, $\mathbf{p}_{\ell_{i}}^{w}$, is given. The landmark's position in the camera frame can be computed as:
\begin{equation}
\mathbf{p}_{\ell_{i}}^{c}=\mathbf{R}_{bc}^{\top}(\mathbf{R}^{\top}(\mathbf{q})(\mathbf{p}_{\ell_{i}}^{w}-\mathbf{p})-\mathbf{t}_{cb}),
\end{equation}
where $\{\mathbf{R}_{bc},\mathbf{t}_{cb}\}$ represent the relative pose between the body frame and the camera frame. We parameterize the \ac{FOV} using the azimuth ($\alpha$) and elevation ($\beta$) angles. Let $\mathbf{p}_{\ell_{i}}^{c}=[X_i,Y_i,Z_i]^{\top}$ denote the coordinates of the $i$-th landmark expressed in the camera frame. The angular coordinates of the landmark are given by:
\begin{align}\label{eq:azimuth_and_elevation}
\alpha_i&=\arctan2(X_i,Z_i),\\ 
\beta_i&=\arctan2(Y_i,\sqrt{X_i^{2}+Z_i^{2}}). 
\end{align}
Consequently, the visibility constraint $h_{\mathcal{V}}$ for the $i$-th landmark is formulated as:
\begin{align}
h_{\mathcal{V}}(\mathbf{p},\mathbf{q},\mathbf{p}_{l_{i}}^{w})\leq0\Leftrightarrow\begin{cases}
\begin{array}{c}
|\alpha_{i}|\leq\alpha_{\max}\\
|\beta_{i}|\leq\beta_{\max}\\
Z_{i}>Z_{\min}
\end{array}\end{cases}
\end{align}
where $\alpha_{\max}$ and $\beta_{\max}$ represent the camera's horizontal and vertical half-angles, respectively, and $Z_i > Z_{\min}$ ensures the landmark is in front of the camera with a minimum distance.

\subsection{Perception Objectives}

We introduce three types of perception objectives used in this paper, \ac{PUM}, \ac{FOV}, and \ac{LA}. Fig. \ref{fig:perceptual_objectives} provides a visual demonstration of their motivations.

\subsubsection{Position Uncertainty Minimization}

Position uncertainty can be modeled in a variety of ways depending on the application. In GPS navigation, the standard metric is the \ac{PDOP} \cite{grewal2007global}, which characterizes how the relative geometry of satellites propagates measurement errors into position uncertainty. For general estimation tasks, the \ac{CRLB} \cite{cramer1999mathematical} is a useful tool that provides a theoretical lower bound on the covariance of any unbiased estimator and is defined as the inverse of the \ac{FIM}. Thus, maximizing the \ac{FIM} is equivalent to minimizing the theoretical lower bound of the estimation error. When analytical expressions are intractable, uncertainty can also be computed numerically by perturbing the measurements with noise and analyzing the variance of the resulting estimates \cite{foehn2022alphapilot}. 

However, these metrics cannot be directly utilized as perception objectives in gradient-based trajectory optimization. The camera's limited field of view introduces discontinuities as features abruptly enter or exit the image boundaries, rendering the objective function non-differentiable. Furthermore, the numerical complexity involved in evaluating these metrics at every optimization step is often prohibitive. We will discuss our solution in detail in Section \ref{sec:position_uncertainty_modeling}.

\subsubsection{FOV Limit Violation Penality (\ac{FOV})}

When tracking multiple or switching landmarks across mission phases, imposing hard \ac{FOV} constraints is often detrimental. Such strict constraints can render the optimization problem infeasible or force the vehicle into conservative behaviors. Therefore, we relax the constraint $h_{\mathcal{V}}(\mathbf{p},\mathbf{q},\mathbf{p}_{l_{i}}^{w})$ into a soft constraint by introducing a set of slack variables~$\mathbf{S}=[S_{0},S_{1},S_{2}]^{\top}$ for each state. At each discrete time instance $t_k$ within the specified interval or regions, the constraints of the target landmark~$i$ are relaxed to:
\begin{align}\label{eq:soft_fov_constraints}
\begin{cases}
\begin{array}{c}
|\alpha_{i}(t_k)|\leq\alpha_{\max}+S_{k,i,0}\\
|\beta_{i}(t_k)|\leq\beta_{\max}+S_{k,i,1}\\
Z_{\min}-Z_{i}(t_k)<S_{k,i,2}
\end{array}\end{cases},
\end{align}
We also define a penalty function on the magnitude of the slack variables multiplied by an adjustable weight $w_s$:
\begin{equation}
\mathcal{L}_{\text{\ac{FOV}}}=w_{\text{\ac{FOV}}}\cdot(\mathbf{1}^{\top}\mathbf{S}+\frac{1}{2}\|\mathbf{S}\|_{2}^{2}).
\end{equation}

When the weight $w_{\text{\ac{FOV}}}$ is sufficiently large, the formulation asymptotically approximates a hard constraint, and thus the drone is forced to sacrifice agility for visibility. Conversely, a lower weight tolerates minor violations, providing the necessary flexibility to maintain smooth flight, particularly during landmark transitions.

\subsubsection{Look-Ahead Gaze}

This is an empirical heuristic that constrains the camera's optical axis to align with a future trajectory waypoint at a constant time horizon $t_{\text{\ac{LA}}}$. It mimics the intuitive gaze adjustments of human pilots during freestyle flight and proves highly effective in enabling challenging maneuvers \cite{wang2025unlocking}. Let's consider a state at time $t_k$. The camera's optical direction expressed in the world frame is given by:
\begin{equation}
\mathbf{z}^{w}_{c_{k}} = \mathbf{R}_{k}\mathbf{R}_{bc}\mathbf{e}_{z}.
\end{equation}
The future point that we want the camera to looks towards is $\mathbf{p}(t_k+t_{\text{\ac{LA}}})$. Subsequently, the desired camera's optical direction can be calculated by $\mathbf{b}_{t_k}=\mathcal{N}(\mathbf{p}(t_k+t_{\text{\ac{LA}}})-\mathbf{p}(t_k))$, where $\mathcal{N}(\cdot)$ returns the normalized vector. Now we construct a cost function to align $\mathbf{z}^{w}_{c_{k}}$ and $\mathbf{b}_{t_k}$:
\begin{equation}
\mathcal{L}_{\text{\ac{LA}}}=-w_{\text{\ac{LA}}}\cdot\exp\left(-\lambda_{\text{\ac{LA}}}\cdot\left(\arccos^{4}\!\left\langle \mathbf{b},\mathbf{z}_{c_{k}}^{w}\right\rangle \right)\right),
\end{equation}
where $w_{\text{\ac{LA}}}$ is the weight and $\lambda_{\text{\ac{LA}}}$ is the decay parameter.

\subsection{Motion Regulation}

True time-optimal solutions typically produce approximate bang-bang control trajectories. However, the associated abrupt acceleration changes cause severe camera shake and image blur, which are detrimental to visual perception. Consequently, we augment the optimization problem with a motion regulation term. This term essentially functions as a jerk-minimization penalty to reduce high-frequency acceleration changes.
The jerk can be calculated from the existing state and input variables:
\begin{equation}
\mathbf{j}=(\frac{1}{m}\sum\mathbf{f})\cdot\mathbf{R}[\mathbf{e}_{z}]_{\times}^{\top}\boldsymbol{\omega}\cdot+(\frac{1}{m}\sum\mathbf{r}_{T})\cdot\mathbf{z}_{b}.
\end{equation}
It follows the jerk minimization term:
\begin{equation}
\mathcal{L}_{\text{MR}}:= w_{\text{MR}} \cdot \|\mathbf{j}\|_2^2,
\end{equation}
where $w_{\text{MR}}$ is typically chosen as a very small value. It is worth noting that this term is only used in perception-aware flights.

\subsection{\ac{TOGT} Problem Formulation}

Let us denote the state trajectory as $\mathbf{x}(\cdot)$ and the corresponding control input trajectory as $\mathbf{u}(\cdot)$, both of which are functions defined over the time interval $[0, T]$ where $T$ is the total duration of the trajectory. The objective is to find a dynamically feasible trajectory that minimizes the total flight time to pass through all waypoints/gates at the specified order. The optimization problem can be formulated as follows:
\begin{algbox}{\ac{TOGT} Trajectory Generation Problem}
\begin{mini!}[l] 
    {\substack{\mathbf{x}(\cdot), \mathbf{u}(\cdot), T}} 
    {T} 
    {\label{eq:time_optimal_trajectory_planning}} 
    {} 
    \addConstraint{\dot{\mathbf{x}}(t)}{= f(\mathbf{x}(t), \mathbf{u}(t)),\label{eq:opti_ode}}{\quad 0 \leq t \leq T}
    \addConstraint{\mathbf{x}(0)}{= \mathbf{x}_{\text{init}}}{}
    \addConstraint{\mathbf{x}(T)}{= \mathbf{x}_{\text{term}}}{}
    \addConstraint{\mathbf{x}_{\text{lb}}\leq}{ \mathbf{x}(t) \leq \mathbf{x}_{\text{ub}}}{}
    \addConstraint{\mathbf{u}_{\text{lb}}\leq}{ \mathbf{u}(t) \leq \mathbf{u}_{\text{ub}}}{}
    \addConstraint{h_{\mathcal{G}}}{(\mathbf{p}(t)) \leq 0,\label{eq:opti_ieq}}{}
\end{mini!}
\end{algbox}
where $\mathbf{x}_{\text{init}}$ and $\mathbf{x}_{\text{term}}$ denote the initial and terminal states, respectively, $\mathbf{x}_{\text{lb}}$ and $\mathbf{x}_{\text{ub}}$ represent the lower and upper bounds on the quadrotor's state, and $\mathbf{u}_{\text{lb}}$ and $\mathbf{u}_{\text{ub}}$ represent the lower and upper bounds on the control inputs, as defined in previous sections. The last constraint~\eqref{eq:opti_ieq} enforces the trajectory to pass through all gates at the specified order. 

\textbf{Perception-aware \ac{TOGT} problem.}
We augment Problem~\eqref{eq:time_optimal_trajectory_planning} with the perception objectives defined in the previous sections, yielding the final optimization problem shown below:
\begin{algbox}{Perception-Aware \ac{TOGT} Problem}
\begin{mini!}[l]
    {\substack{\mathbf{x}(\cdot), \mathbf{u}(\cdot), T}} 
    {T + \int_{0}^{T}\mathcal{L}_{\text{PA}}\;dt} 
    {\label{eq:perception_aware_time_optimal_trajectory_planning}} 
    {} 
    \addConstraint{\text{time-opt. constraints (\ref{eq:time_optimal_trajectory_planning}b-g)}}{}{}
    \addConstraint{\text{visibility constraints (\ref{eq:soft_fov_constraints})}}{}{}
\end{mini!}
\end{algbox}

where
\begin{equation}
\mathcal{L_{\text{PA}}}=\mathcal{L}_{\text{\ac{LA}}} + \mathcal{L}_{\text{\ac{FOV}}} + \mathcal{L}_{\text{\ac{PUM}}} + \mathcal{L}_{\text{MR}},
\end{equation}
expresses the overall perceptual requirement of the problem. The remaining question is how to solve these two problems stably and efficiently, which we will answer in Section~\ref{sec:numerical_optimization}.

\begin{figure}[!t]
    \centering
\tikzstyle{every node}=[font=\footnotesize]
\begin{tikzpicture}[>=stealth]
\node [inner sep=0pt, outer sep=0pt, anchor=north west] (img1) at (0,0)
{\includegraphics[width=4.2cm]{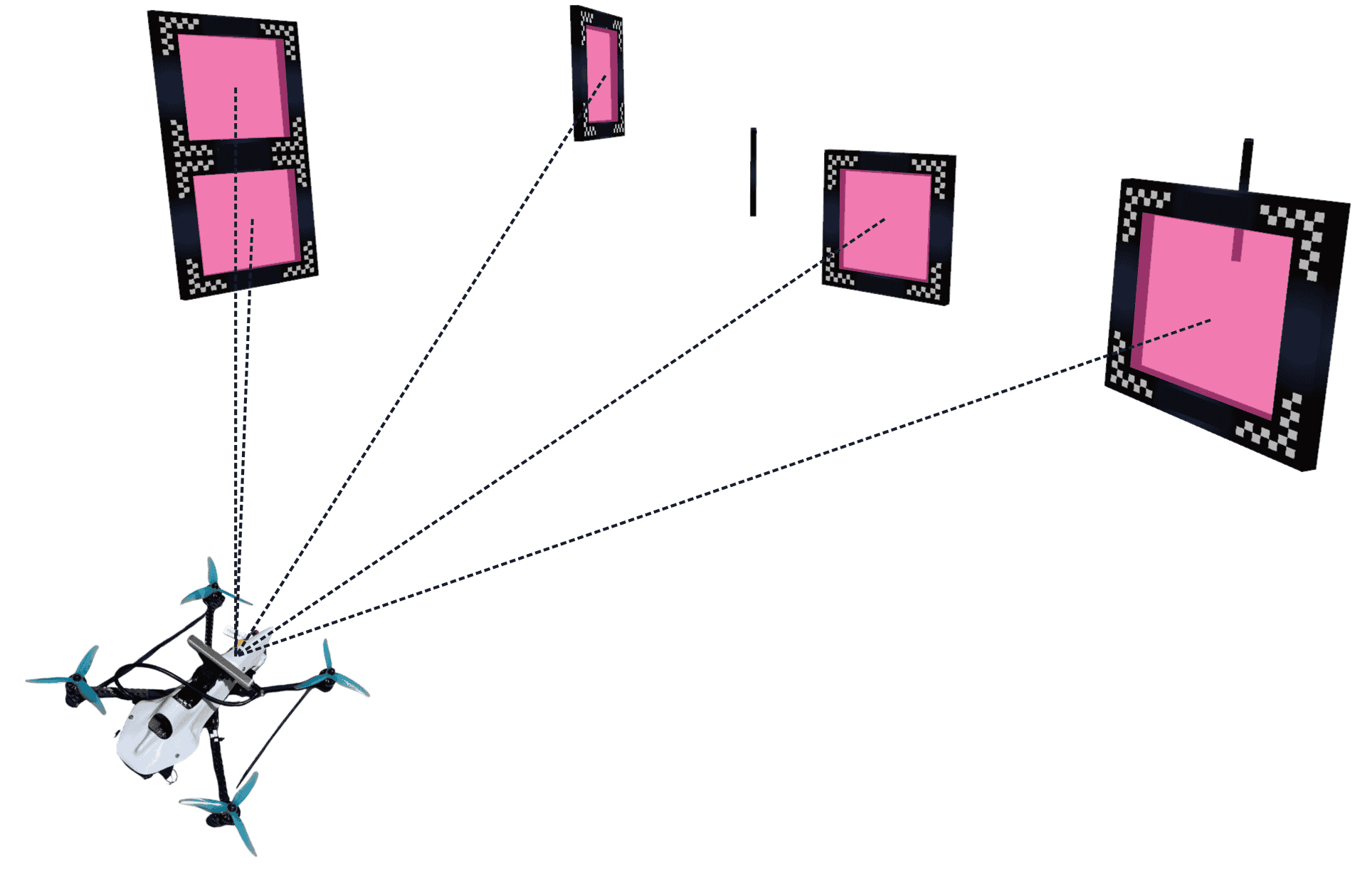}};
\node [inner sep=0pt, outer sep=0pt, right=1mm of img1] (img2)
{\includegraphics[width=4.2cm]{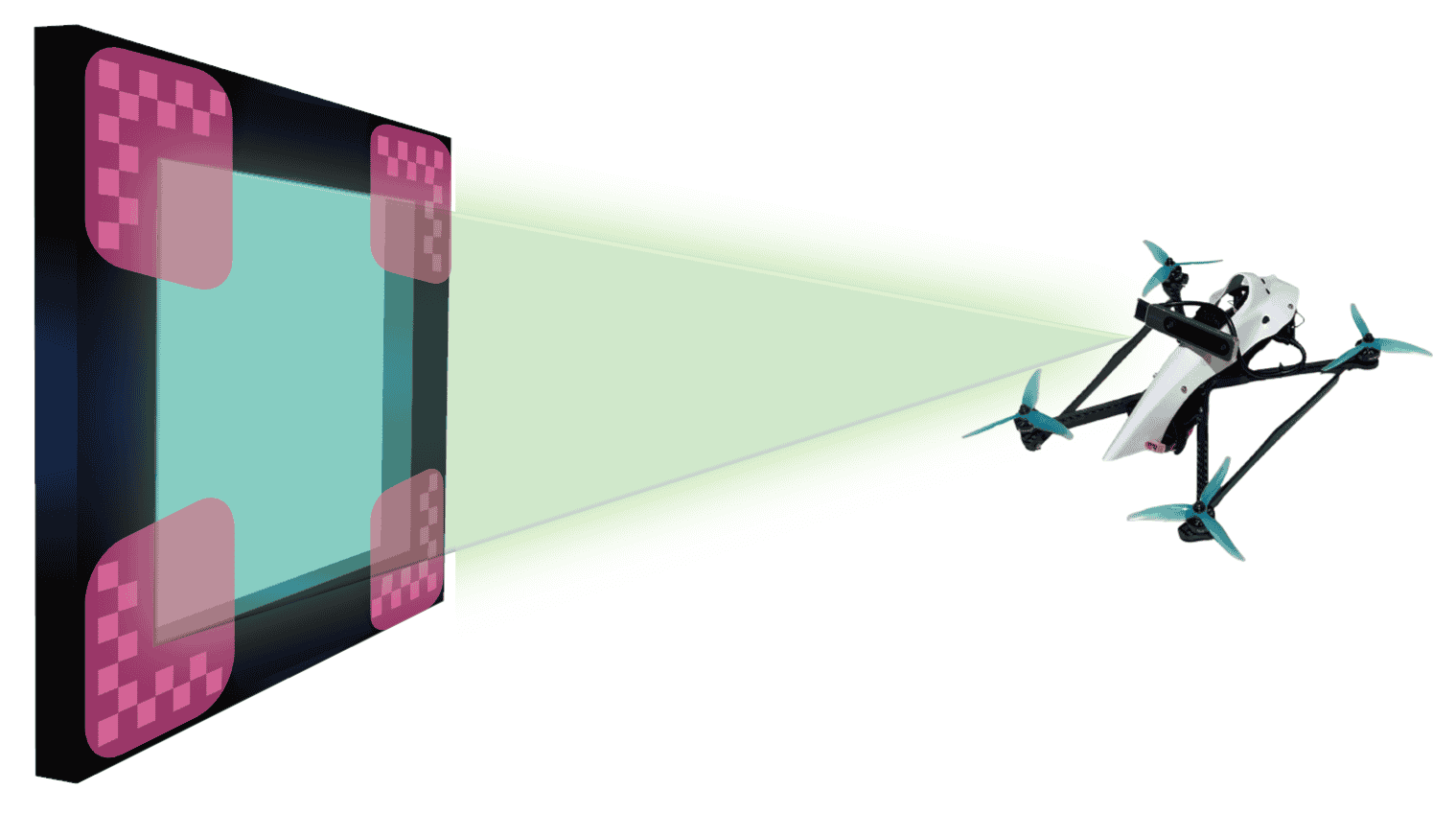}};
\node[below=1mm of img1](text1){(a) Multi-landmark geometry};
\node[below=1mm of img2](text2){(b) Single-landmark geometry};
\end{tikzpicture}
\caption{Two types of geometry information for localization: (a) the relative geometry of multiple landmarks, where each landmark is treated as a single 3D point; or (b) features observable from a single landmark.}
\label{fig:landmark_observation}
\vspace{-0.3cm}
\end{figure}

\section{Modeling of Position Uncertainty}\label{sec:position_uncertainty_modeling}

This section analyzes the theoretical formulation of position uncertainty derived from visual measurements of known landmarks. First, we outline the fisheye camera model and its feature representation. Next, we derive a differentiable approximation of the \ac{FIM}. We conclude by proposing an uncertainty metric that achieves a balance between computational tractability and estimation accuracy

We assume that the 3D geometries and global positions of a set of landmarks are known a priori, and that their feature coordinates (e.g., corners) are observable in the image plane. Given the camera intrinsics and extrinsics, we can estimate the \ac{UAV}’s 6-DoF pose in the world frame via the \ac{PnP} method using at least four pairs of matched features, provided these features are non-collinear. Note that a pose estimate can also be derived from a single landmark if it provides at least four non-collinear feature measurements. To facilitate analysis, we treat these as two distinct sources of visual measurements contributing to the position estimate: the relative geometry among landmarks and the 3D geometry of each individual landmark. Their distinction is illustrated in Fig. \ref{fig:landmark_observation}. This separation is reasonable. When landmarks are remote, the features on any single landmark provide highly correlated information; consequently, positioning accuracy is primarily dictated by the relative geometry among the landmarks. Conversely, when a landmark is proximal, its features are well-distributed across the image plane. This provides sufficient geometric constraints to yield accurate localization results, even without assistance from other visible landmarks.

\subsection{Fisheye Camera Model}

We employ an equidistant projection model \cite{kingslake1989history}, whose basic principle is illustrated in Fig.~\ref{fig:fisheye_model}. Its imaging process can be decomposed into two steps: first, linearly projecting the 3D points in space onto a virtual unit sphere; and then projecting the points on the unit sphere onto the image plane. To map a 3D point in the camera frame $\mathbf{p}^{c}$ to its 2D image coordinate $\mathbf{s}=[u,v]^T$, we first compute the so-called angle of incidence, $\theta$, as follows:
\begin{equation}
\theta = \arctan2(\sqrt{X^2 + Y^2}, Z).
\end{equation}
Then, the image coordinates can be obtained by using:
\begin{align}
u&=f_{x}\theta x_{n}+c_{x},\\
v&=f_{y}\theta y_{n}+c_{y},
\end{align}
with
\begin{equation}
x_{n}=\frac{X}{\sqrt{X^{2}+Y^{2}}}\text{ and } y_{n}=\frac{Y}{\sqrt{X^{2}+Y^{2}}},
\end{equation}
where $f_x$ and $f_y$ denote the focal lengths in the $x$ and $y$ directions, respectively, and $c_x$ and $c_y$ represent the principal point of the camera.

\begin{figure}[!t]
\centering
\vspace{-0.3cm}
\includegraphics[width=0.3\textwidth]{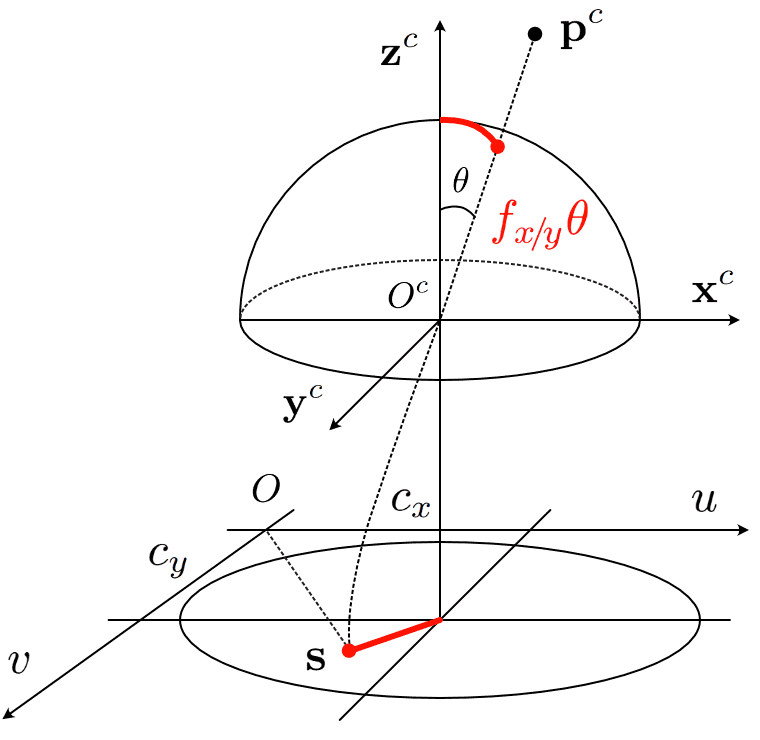}
\caption{Equidist projection model of fisheye cameras.}\label{fig:fisheye_model}
\vspace{-0.3cm}
\end{figure}

\subsection{Bearing Vector Representation}
Using bearing vectors to represent image features can greatly simplify derivation \cite{zhang2020fisher}. To obtain this representation, we calculate $(x_n, y_n)$ from their 2D image coordinates:
\begin{align}
x_n = (u - c_x)/f_x,\\
y_n = (v - c_y)/f_y.
\end{align}
It follows the computation of $\theta = \sqrt{x_n^2 + y_n^2}$. Finally, we obtain the bearing vector representation:
\begin{equation}
\boldsymbol{\rho}^c =
\begin{bmatrix}
\sin(\theta)/\theta \cdot x_n \\
\sin(\theta)/\theta \cdot y_n \\
\cos(\theta)
\end{bmatrix}.
\end{equation}

Now, let's study the noises of the bearing vector, $\boldsymbol{\rho}^c$, which can be propagated from the image noise, $\sigma_{u}$ and $\sigma_{v}$. The result is given below: 
\begin{equation}\label{eq:covariance_bearing_vector}
\boldsymbol{\Sigma}_{\rho} \approx \text{diag}\left(\frac{\sigma_{u}^{2}}{f_{x}^{2}}, \frac{\sigma_{v}^{2}}{f_{y}^{2}}, \left(\frac{\sigma_{u} c_{x}}{f_{x}^{2}} + \frac{\sigma_{v} c_{y}}{f_{y}^{2}}\right)^{2}\right).
\end{equation}
The derivation is detailed in Appendix A. To further simplify its expression, we utilize the fact that $c_x$, $c_y$, $f_x$, and $f_y$ are of the same magnitude: we assume $c_x = c_y$, $f_x = f_y$, and $\sigma_u = \sigma_v$, leading to an isotropic covariance matrix of $\boldsymbol{\rho}^c$ that takes the form of $\boldsymbol{\Sigma}_{\rho} \approx \text{diag}\left(\frac{\sigma_u^2}{f_x^2}, \frac{\sigma_u^2}{f_x^2}, \frac{\sigma_u^2}{f_x^2}\right)$. This simplification is vital for efficient \ac{FIM} computation.

\subsection{Fisher Information Matrix Derivation}\label{subsec:uncertainty_formulation}

A primary obstacle in computing the \ac{FIM} for camera-based observations is the \ac{FOV} constraints. We address this issue by introducing a continuous and differentiable approximation of the visibility function.

Consider $N$ landmarks in the environment, each producing $M$ feature points. Let $v_{i,j}$ denote the visibility indicator of the j-th feature of the i-th landmark, which returns 1 if the landmark is within the camera's \ac{FOV} and 0 otherwise. The \ac{FIM} at state $\mathbf{x}$ can be expressed as:
\begin{equation}\label{eq:fim_expression}
\mathbf{I}_{\text{FIM}}(\mathbf{x})=\sum_{i=1}^{N}\sum_{j=1}^{M}v_{i,j}\cdot\mathbf{J}_{i,j}^{\top}\boldsymbol{\Sigma}_{\rho}\mathbf{J}_{i,j}.
\end{equation}
To obtain a continuous and differentiable expression, we approximate the indicator function using the product of three separate smooth functions:
\begin{equation}\label{eq:visibility_indicator}
\hat{v}(\mathbf{x})=v_{\alpha}(\mathbf{x})\cdot v_{\beta}(\mathbf{x})\cdot v_{Z}(\mathbf{x}),
\end{equation}
where $v_{\alpha}$ represents the azimuth visibility, $v_{\beta}$ the elevation visibility, and $v_{Z}$ the depth posibility, which are defined as:
\begin{align}
v_{\alpha}&=\frac{1}{2}+\frac{1}{2}\tanh\left(\lambda_{v}\cdot(\alpha_{\max}-|\alpha|)\right),\\ 
v_{\beta}&=\frac{1}{2}+\frac{1}{2}\tanh(\lambda_{v}\cdot(\beta_{\max}-|\beta|)),\\ 
v_{Z}&=\frac{1}{2}+\frac{1}{2}\tanh(\lambda_{v}\cdot(Z_{\min}-Z)),
\end{align}  
where $\lambda_{v}$ is the sharpness parameter. The Jacobian of the bearing vector $\mathbf{p}_{i,j}^{c}$ with respect to the current camera position $\mathbf{p}_{c}^{w}$ can be derived by using the chain rule:
\begin{equation}\label{eq:jacobian_bearing_vector}
\mathbf{J}_{i,j}=\frac{\partial\boldsymbol{\rho}_{i,j}^{c}}{\partial\mathbf{p}_{i,j}^{c}}\frac{\partial\mathbf{p}_{i,j}^{c}}{\partial\mathbf{p}_{c}^{w}},
\end{equation}
where $\mathbf{p}_{i,j}^{c}$ is the corresponding feature position in the camera frame, and 
\begin{align}
\frac{\partial\boldsymbol{\rho}_{i,j}^{c}}{\partial\mathbf{p}_{i,j}^{c}}&=\frac{1}{d_{i,j}}(\mathbf{I}_{3}-\boldsymbol{\rho}_{i,j}^{c}(\boldsymbol{\rho}_{i,j}^{c})^{\top}),\\
\frac{\partial\mathbf{p}_{i,j}^{c}}{\partial\mathbf{p}_{c}^{w}}&=-\mathbf{R}_{cw},
\end{align}
with $d_{i,j}=\|\mathbf{p}_{i,j}^{c}\|_{2}$. Since $\boldsymbol{\Sigma}_{\rho}$ is already given in \eqref{eq:covariance_bearing_vector}, the \ac{FIM} can be computed by iterating through all visible landmarks. 

From an optimization perspective, the expression of \ac{FIM} is still too complicated as it involves too many state variables, such as the rotation matrix. The following derivation shows that \ac{FIM} can take a much simpler form that only depends on the vehicle's position.

\begin{proposition}\label{prop:fisher_information_matrix}
If the measurement noise of the bearing vector is isotropic, then the \ac{FIM} in \eqref{eq:fim_expression} can be reduced exactly to:
\begin{equation}\label{eq:fisher_information_matrix}
 \mathbf{I}_{\text{FIM}}(\mathbf{p})=\sum_{i=1}^{N}\sum_{j=1}^{M}v_{i,j}\cdot\mathbf{A}_{i,j}^{\top}\boldsymbol{\Sigma}_{\rho}\mathbf{A}_{i,j},
\end{equation}
where
\begin{align}
\mathbf{A}_{i,j}&=\frac{1}{d_{i,j}}(\mathbf{I}_{3}-\boldsymbol{\rho}_{i,j}^{w}(\boldsymbol{\rho}_{i,j}^{w})^{\top}),\label{eq:jacobian}\\ 
\boldsymbol{\rho}_{i,j}^{w}&=\frac{\mathbf{p}_{\ell_{i,j}}^{w}-\mathbf{p}_{c}^{w}}{\|\mathbf{p}_{\ell_{i,j}}^{w}-\mathbf{p}_{c}^{w}\|_{2}}.\label{eq:global_bearing_vector}
\end{align}
\end{proposition}
\begin{proof}
See Appendix B.
\end{proof}  

Proposition \ref{prop:fisher_information_matrix} eliminates the dependency on the rotation matrix $\mathbf{R}_{cw}$ in $\mathbf{J}_{i,j}$, decoupling $\mathbf{R}_{cw}$ from the position uncertainty computation. This greatly mitigates the computation burden.

Let's recall the relationship between the \ac{FIM} and \ac{CRLB}. The minimum achievable covariance for any unbiased estimator can be expressed as $\boldsymbol{\Sigma}_{\text{lb}}=\mathbf{I}_{\text{FIM}}^{-1}$. Therefore, it is valid to obtain the following scalar measure of the uncertainty: 
\begin{equation}
\mathcal{L}_{\text{\ac{PUM}}}=\log\det(\boldsymbol{\Sigma}_{\text{lb}})=-\log\det(\mathbf{I}_{\text{FIM}}).\label{eq:fim_position_uncertainty}
\end{equation}
A key advantage of this approach is that it avoids the explicit inversion of $\mathbf{I}_{\text{FIM}}$. However, computing $\mathbf{I}_{\text{FIM}}$ itself remains computationally demanding, as it requires iterating through all features in each landmark, which adds up to $NM$ features in total. For instance, if we have 100 landmarks with 10 features each, we must perform 1,000 individual Jacobian updates to assemble the \ac{FIM} at just one sampled camera location, which is unacceptable. In the following section, we present a more efficient way to evaluate the position uncertainty.

\subsection{Fast Position Uncertainty Evaluation}\label{subsec:uncertainty_approximation}

As previously noted, features belonging to the same landmark often yield highly correlated information. To reduce computational overhead, we can simplify the \ac{FIM} by representing each landmark as a single point, typically its geometric centroid, thereby avoiding the processing of redundant features. However, it is essential to explicitly model the position uncertainty of single-landmark observations that contain multiple features. For instance, a proximal gate situated near the optical axis with all corners visible can provide more informative positioning data than the collective contribution of all other visible gates. As a result, we need to model two types of uncertainties that arise from two sources of observation (See Fig. \ref{fig:landmark_observation}), and then combine them properly to yield a comprehensive uncertainty metric. One feasible approach is using the covariance update principle \cite{barfoot2024state} for multiple observations of the same state:
\begin{equation}
\hat{\mathbf{\boldsymbol{\Sigma}}}^{-1}=\sum_{k=1}^{K}\boldsymbol{\mathbf{\Sigma}}_{k}^{-1},
\end{equation}
where $K$ is the total number of observations, $\mathbf{\Sigma}_k$ is the covariance associated with each observation, and $\hat{\mathbf{\Sigma}}$ is the resulting covariance of the fused state estimate.

\begin{figure}[!htbp]
\centering
\includegraphics[width=0.4\textwidth]{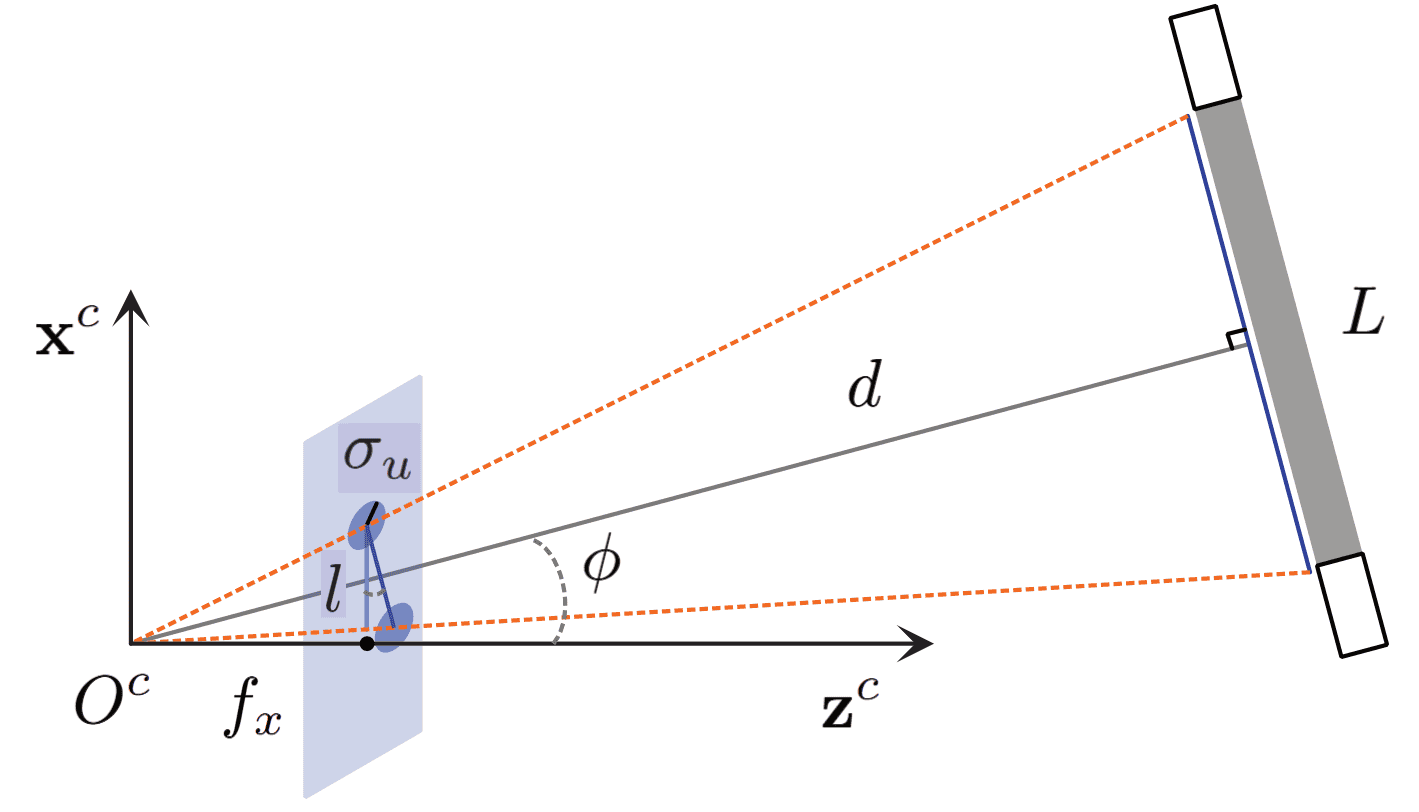}
\caption{The distance to a planar landmark can be approximately estimated by computing the ratio between the actual landmark size~
$L$ and its observed size~$l$ in pixels. Consequently, we can derive the uncertainty of the distance measurement and use it as an estimate of the corresponding position uncertainty.}\label{fig:triangulation}
\vspace{-0.1cm}
\end{figure}

\textbf{Single-landmark geometry:}
We approximate each 3D landmark as a simple 2D line segment. Let $L$ denote its physical characteristic length (actual size) and $l$ denote the length of its projection onto the image plane. Fig. \ref{fig:triangulation} illustrates this modeling process applied to a rectangular gate. Based on these parameters, the distance to the landmark can be estimated as
\begin{equation}
d=f_{x\!/\!y}(L/l)\cos\phi,\label{eq:distance_estiamte}
\end{equation}
where $\phi$ denotes the angle between the object's facing direction $\mathbf{z}_{\ell}$ and the camera's optical axis $\mathbf{z}_{c}$. It is easy to calculate $\phi$ in the world frame:
\begin{equation}
\cos\phi = \frac{\mathbf{z}_{\ell}^{w} \cdot \mathbf{z}_{c}^{w}}{\|\mathbf{z}_{\ell}^{w}\| \cdot \|\mathbf{z}_{c}^{w}\|}.
\end{equation}
There is an analytical expression of the uncertainty of the distance estimate obtained from \eqref{eq:distance_estiamte}.
\begin{proposition}\label{prop:distance_uncertainty}
Given an observable object with a known size $L$, the standard deviation of the noise of the distance estimate can be approximated as:
\begin{equation}
\sigma_{d} = \left|\frac{\sqrt{2} d^{2}}{L \cos(\phi) \cdot f_{x}}\right| \cdot \sigma_{u}.
\end{equation}
\end{proposition}
\begin{proof}
  See Appendix C.
\end{proof}  
Proposition~\ref{prop:distance_uncertainty} tells how accurate the distance measurement can be given a single landmark with a nontrivial dimension. However, it may fail to reflect the true uncertainty when a landmark is sufficiently close to the camera that it is only partially visible. To address this, we introduce a correction term that grows exponentially as the distance decreases to obtain a reasonable uncertainty measure $\sigma_{d^{\prime}}$:
\begin{equation}\label{eq:augmented_depth_uncertainty}
\sigma_{d^{\prime}}^{2} = \sigma_{d}^{2} + \sigma_{u}/(d^{4}+\epsilon),
\end{equation}
where $\epsilon$ is a small value to avoid a singularity. 

We extend this distance uncertainty metric to model the full positional uncertainty. While this approach is an approximation, it yields a conservative upper bound on the error that is easy to compute. The covariance for a single-landmark observation is thus defined as:
\begin{equation}
\boldsymbol{\Sigma}_{\ell}=\text{diag}(\sigma_{d^{\prime}}^{2},\sigma_{d^{\prime}}^{2},\sigma_{d^{\prime}}^{2}).\label{eq:covariance_single_landmark}
\end{equation}
Due to its isotropic structure, $\boldsymbol{\Sigma}_{\ell}$ is invariant under reference frame transformations. Consequently, it serves as a valid uncertainty measure for both $\mathbf{p}^{w}$ and $\mathbf{p}^{c}$, which significantly simplifies the derivation.

\begin{algorithm}[!t]
	\caption{Fast Position Uncertainty Evaluation}
	\label{alg:fast_position_uncertaity_evaluation}
	\begin{algorithmic}[1]
		\STATE \textbf{Inputs:} Quadrotor's state $\mathbf{x}$, landmark centroid positions $\{\mathbf{p}^{w}_{\ell_1},\mathbf{p}^{w}_{\ell_2},...,\mathbf{p}^{w}_{\ell_N}\}$, forward directions $\{\mathbf{c}^{w}_{\ell_1},\mathbf{c}^{w}_{\ell_2},...,\mathbf{c}^{w}_{\ell_N}\}$, and size $L$, camera intrinsics $\{f_x,f_y,c_x,c_y\}$ and extrinsics $\{\mathbf{R}_{bc},\mathbf{t}_{cb}\}$, camera's \ac{FOV} $\{\alpha_{\max},\beta_{\max}, Z_{\min}\}$, image noise $\{\sigma_u, \sigma_v\}$
		\STATE Propagate the measurement noise to the bearing vector covariance $\Sigma_{\rho}$ by \eqref{eq:covariance_bearing_vector}
		\STATE Get the camera's position $\mathbf{p}^{w}_{c}$ from $\mathbf{x}$ and $\{\mathbf{R}_{bc},\mathbf{t}_{cb}\}$
		\STATE \textbf{Repeat at each landmark $\ell_{i}$:}
		\STATE \hspace{1em} \textbf{1. Visibility evaluation:}
		\STATE \hspace{2em} Compute $\boldsymbol{\rho}^{w}_{\ell_i}$,  $\boldsymbol{\rho}^{c}_{\ell_i}$, and $Z_i$ via \eqref{eq:global_bearing_vector}
		\STATE \hspace{2em} Extract $\alpha_i$ and $\beta_i$ from $\boldsymbol{\rho}^{c}_{\ell_i}$ using \eqref{eq:azimuth_and_elevation}
		\STATE \hspace{2em} Get visibility indicators $\hat{v}_{i}$ by \eqref{eq:visibility_indicator}
		\STATE \hspace{1em} \textbf{2. Uncertainty of single-landmark observation:}
		\STATE \hspace{2em} Evaluate augmented depth uncertainty $\sigma_{d^{\prime}}$ with \eqref{eq:augmented_depth_uncertainty}
		\STATE \hspace{2em} Construct inversed covariance  $\Sigma_{\ell_i}^{-1}$ as per \eqref{eq:covariance_single_landmark}  
		\STATE \hspace{1em} \textbf{3. Construct Information Matrix:} 
        \STATE \hspace{2em} Derive $\bar{\mathbf{A}}_{i}$ by using $\boldsymbol{\rho}^{w}_{\ell_i}$ according to \eqref{eq:jacobian}
        
		\STATE \textbf{4. Approximate \ac{FIM} Calculation:}
		\STATE Sum individual information matrices weighted by visibility $\hat{v}_{i}$ as in \eqref{eq:fim_simple}
        
		\STATE \textbf{5. Uncertainty Metric Computation:}
        \STATE Acquire the final metric $\mathcal{L}_{\text{PMU}}^{\prime}$ by following \eqref{eq:approximate_position_uncertainty}
	\end{algorithmic}
\end{algorithm}

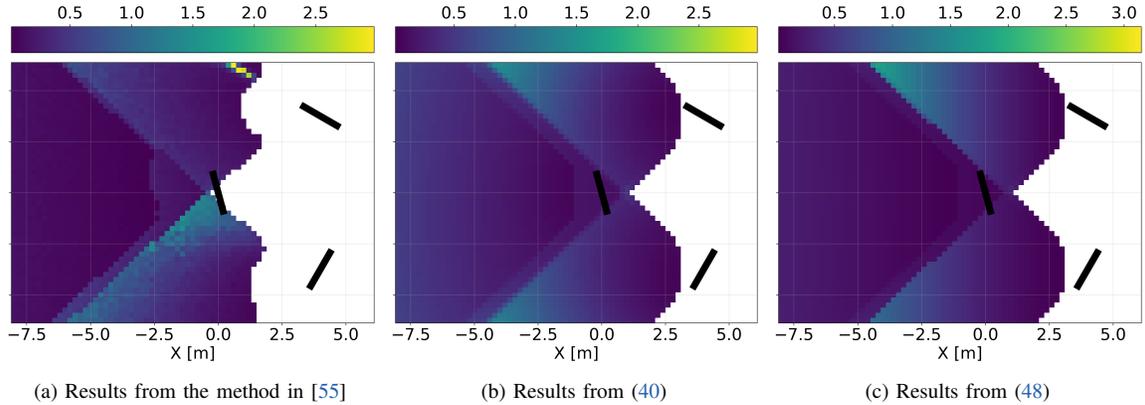
\begin{figure*}[!t]
    \centering
\tikzstyle{every node}=[font=\footnotesize]
\begin{tikzpicture}[>=stealth]
\node [inner sep=0pt, outer sep=0pt, anchor=north west] (img1) at (0,0)
  {\includegraphics[width=5.0cm]{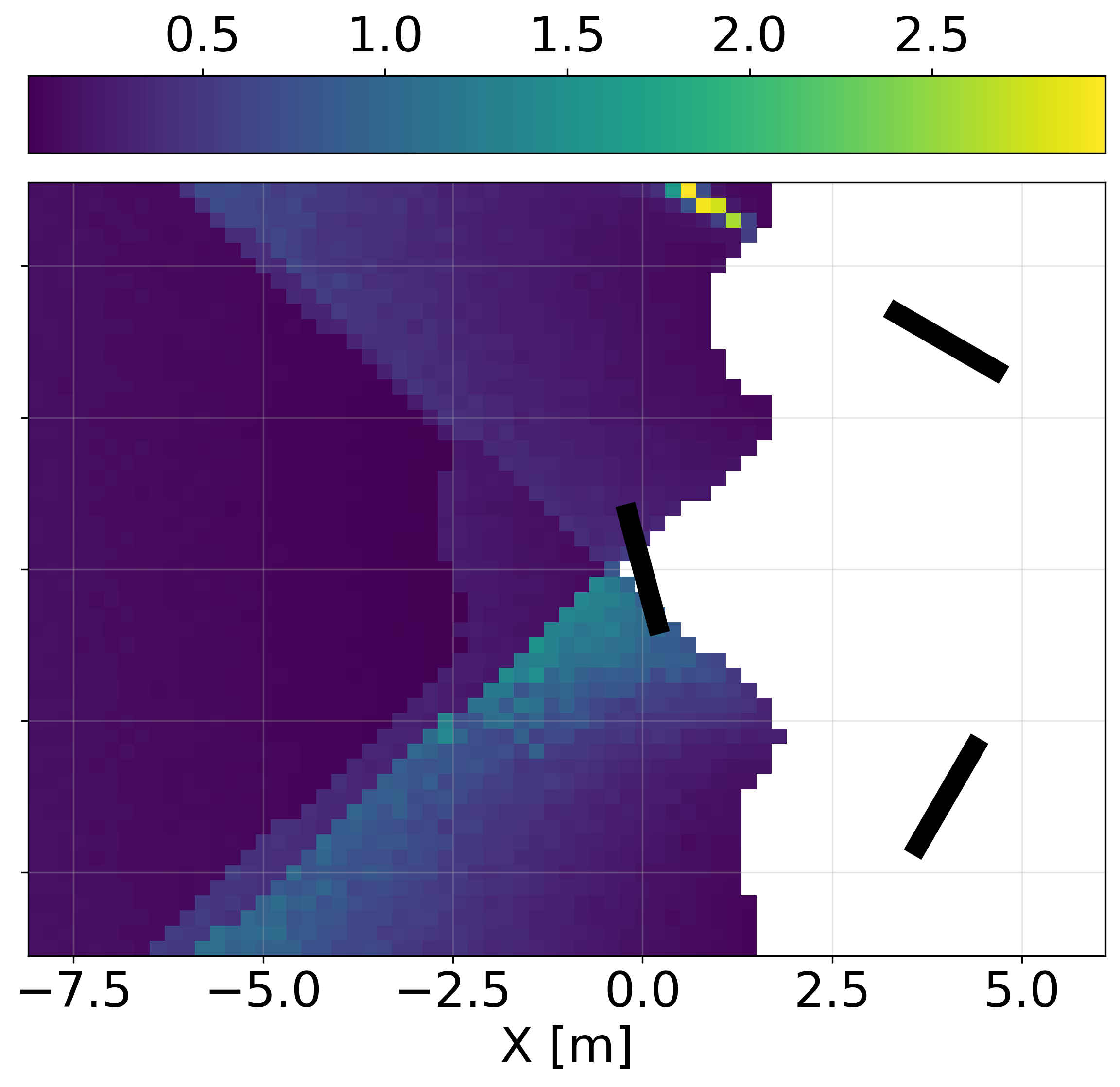}};
\node [inner sep=0pt, outer sep=0pt, right=1mm of img1] (img2)
  {\includegraphics[width=5.0cm]{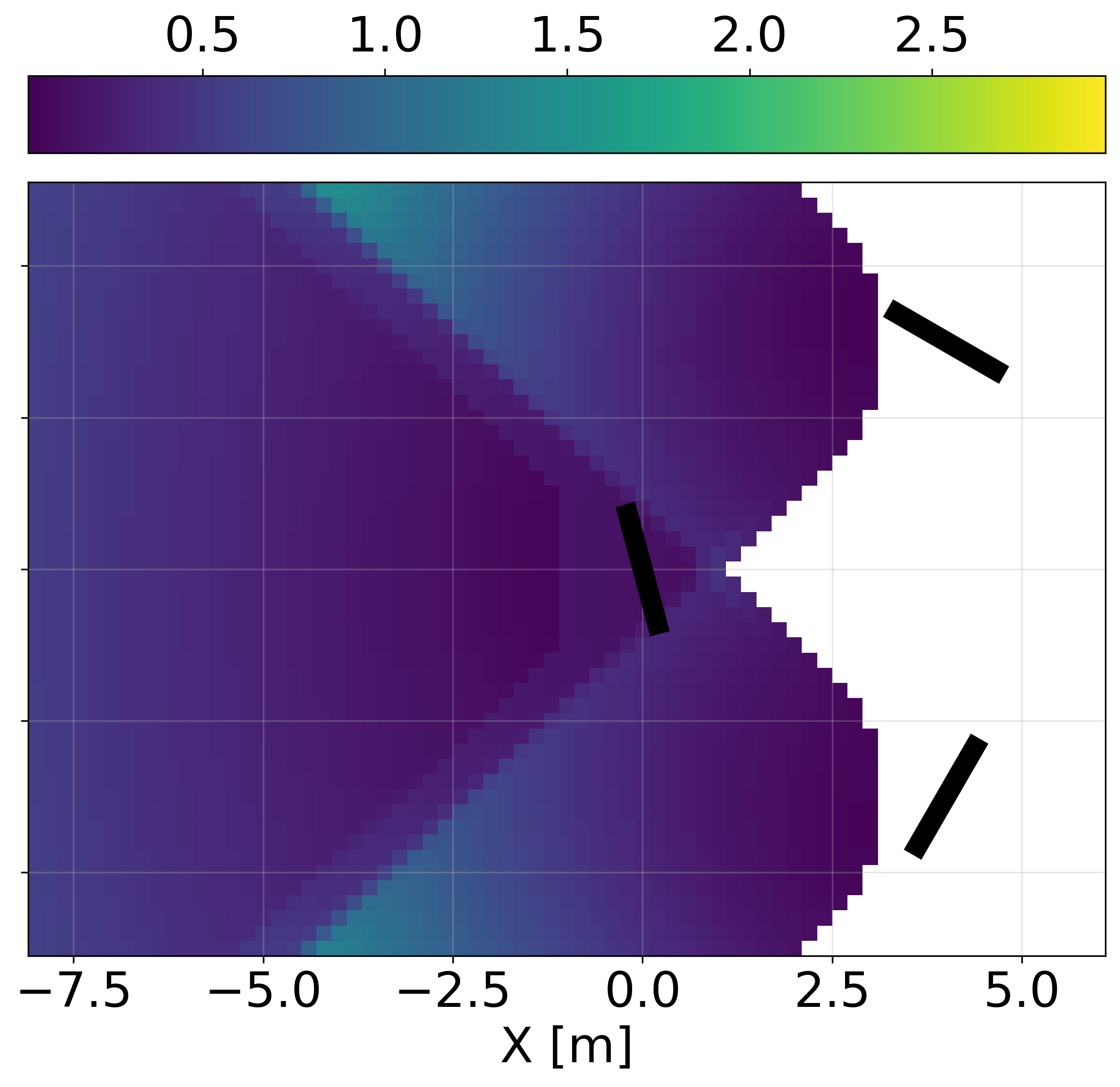}};
\node [inner sep=0pt, outer sep=0pt, right=1mm of img2] (img3)
  {\includegraphics[width=5.0cm]{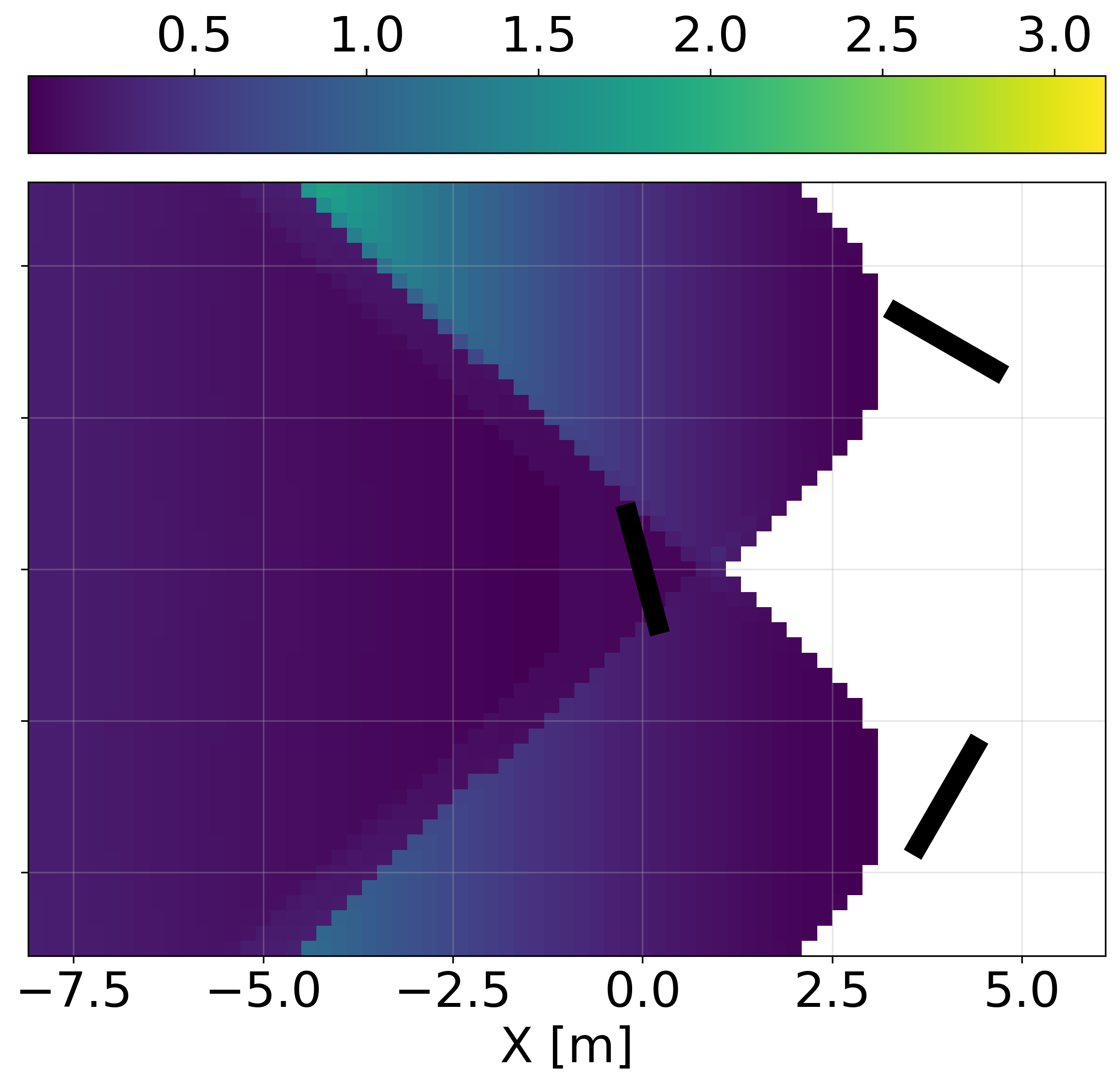}};

\node[below=1mm of img1](text1){(a) Results from the method in \cite{foehn2022alphapilot}};
\node[below=1mm of img2](text2){(b) Results from \eqref{eq:fim_position_uncertainty}};
\node[below=1mm of img3](text3){(c) Results from \eqref{eq:approximate_position_uncertainty}};
  
\end{tikzpicture}
\caption{Comparison of position uncertainty measures in an environment containing three square gates (black), each featuring four corner points. For each cell in the $xy$-plane, the camera's optical axis is assumed to be aligned with the world $x$-axis. Position uncertainty is calculated using three methods: (a) the sampling-based approach introduced in \cite{foehn2022alphapilot}, (b) the original \ac{FIM}-based method \eqref{eq:fim_position_uncertainty}, and (c) our proposed approximation \eqref{eq:approximate_position_uncertainty}. The results demonstrate that our derived measure closely approximates the true uncertainty.}
\label{fig:uncertainty_comp}
\end{figure*}

\textbf{Multi-landmark spatial geometry:} Bahnam et al.~\cite{bahnam2026monorace} confirmed in their experiments that the variation in spatial configuration of gates greatly improves solution accuracy of \ac{PnP}. To model this phenomenon mathematically, we represent visible landmarks by their 3D centroid points and compute the corresponding \ac{FIM} via \eqref{eq:fisher_information_matrix}. This yields:
\begin{equation}\label{eq:fim_simple}  
\mathbf{I}^{\prime}_{\text{FIM}}=\sum_{i=1}^{N}v_{i}\cdot\bar{\mathbf{A}}_{i}^{\top}\boldsymbol{\Sigma}_{\rho}\mathbf{\bar{A}}_{i},
\end{equation}
where $\mathbf{\bar{A}}_{i}$ is constructed from the bearing vector of the centroid of the i-th landmark, $\mathbf{p}_{\ell_{i}}^{w}=\sum_{j=1}^{M}\mathbf{p}_{\ell_{i,j}}^{w}$. 

Combine \eqref{eq:fim_simple} with \eqref{eq:covariance_single_landmark} of each visible landmark, we get the final uncertainty metric:
\begin{equation}
\mathcal{L}_{\text{PMU}}^{\prime}=-\log\det\left(\mathbf{I}^{\prime}_{\text{FIM}}\!+\!\sum_{i=1}^{N}v_{i}\cdot\boldsymbol{\Sigma}_{\ell_{i}}^{-1}\right).\label{eq:approximate_position_uncertainty}
\end{equation}
The rationale behind \eqref{eq:approximate_position_uncertainty} is intuitive: when only a single landmark is visible, or when one is exceptionally proximal and informative, $\mathcal{L}_{\text{PMU}}^{\prime}$ is dominated by that individual observation and its associated covariance \eqref{eq:covariance_single_landmark}. Conversely, when multiple landmarks are visible across a wide geometric distribution, the \ac{FIM} term \eqref{eq:fim_simple} becomes the primary determinant of $\mathcal{L}_{\text{PMU}}^{\prime}$.

Readers can refer to Algorithm \ref{alg:fast_position_uncertaity_evaluation} for the detailed computation pipeline. Additionally, Fig. \ref{fig:uncertainty_comp} validates this metric by comparing our derived estimates against the actual position uncertainty.


%
\section{Optimization Framework}\label{sec:numerical_optimization}

This section details the numerical solver and optimization strategy. Given the highly non-convex problem structure of \eqref{eq:perception_aware_time_optimal_trajectory_planning} and the potentially large number of variables and constraints caused by long flight distances, naively solving \eqref{eq:time_optimal_trajectory_planning} and \eqref{eq:perception_aware_time_optimal_trajectory_planning} in a one-stop manner is computationally intractable. 

The primary challenge stems from the lack of good initial guesses for time allocation and trajectory. The former makes it difficult to determine the optimal discretization resolution: excessive sampling points introduce unnecessary computational overhead and impede convergence, whereas an insufficient number degrades numerical integration accuracy. Meanwhile, poor initial trajectories often lead to entrapment in local minima. To address these challenges, we propose a three-step optimization framework designed to ensure consistent and efficient convergence.

\subsection{Three-Step Optimization Framework}

Fig. \ref{fig:optimization_diagram} illustrates the proposed three-step framework. Specifically, step one utilizes a polynomial-based \ac{TOGT} planner~\cite{qin2023time}. Step two enhances time-optimality by assigning a sufficient number of polynomial segments capable of representing the true time-optimal trajectory. Finally, the third step employs a direct multiple-shooting method similar to~\cite{zhou2023efficient} to refine the trajectory, incorporating the optimal gate-traversal points obtained from the second step and all necessary perception objectives. 

\textbf{Step 1: polynomial-based TOGT planner.}
Polynomial parameterization, combined with differential flatness, has proven extremely efficient for formulating dynamically feasible quadrotor trajectories \cite{mellinger2011minimum}. Furthermore, recent research demonstrates that convex geometrical constraints can be explicitly eliminated \cite{wang2022geometrically}, paving the way for the solution to the \ac{TOGT} problem presented in \cite{qin2023time}. 

Let's introduce a differential-flatness transformation that maps flat outputs $\mathbf{y}$ and its higher-order derivatives, denoted as $\mathbf{y}^{[s]}=[\mathbf{y}^{T},\mathbf{\dot{y}}^{T},\dots,\mathbf{y}^{(s)^{T}}]^{T}$, to actual quadrotor states and inputs:
\begin{equation}\label{eq:flatness_map}
\begin{aligned}
  \mathbf{x} & =\Psi_{\mathbf{x}}\left(\mathbf{y},...,\mathbf{y}^{(s-1)}\right),\\
  \mathbf{u} & =\Psi_{\mathbf{u}}\left(\mathbf{y},...,\mathbf{y}^{(s)}\right).
  \end{aligned}
\end{equation}
This allows optimization of $\mathbf{x}$ and $\mathbf{u}$ to be performed through the optimization of the polynomial-parameterized $\mathbf{y}$. The state and input constraints can also be equivalently expressed as:
\begin{equation}\label{eq:flatness_constraint}
h_{\Psi}(\mathbf{y}):=h(\Psi_{\mathbf{x}}(\mathbf{y}^{[s-1]}),\Psi_{\mathbf{u}}(\mathbf{y}^{[s]}))\leq0.
\end{equation}

\begin{figure}[!t]
\centering
\includegraphics[width=0.48\textwidth]{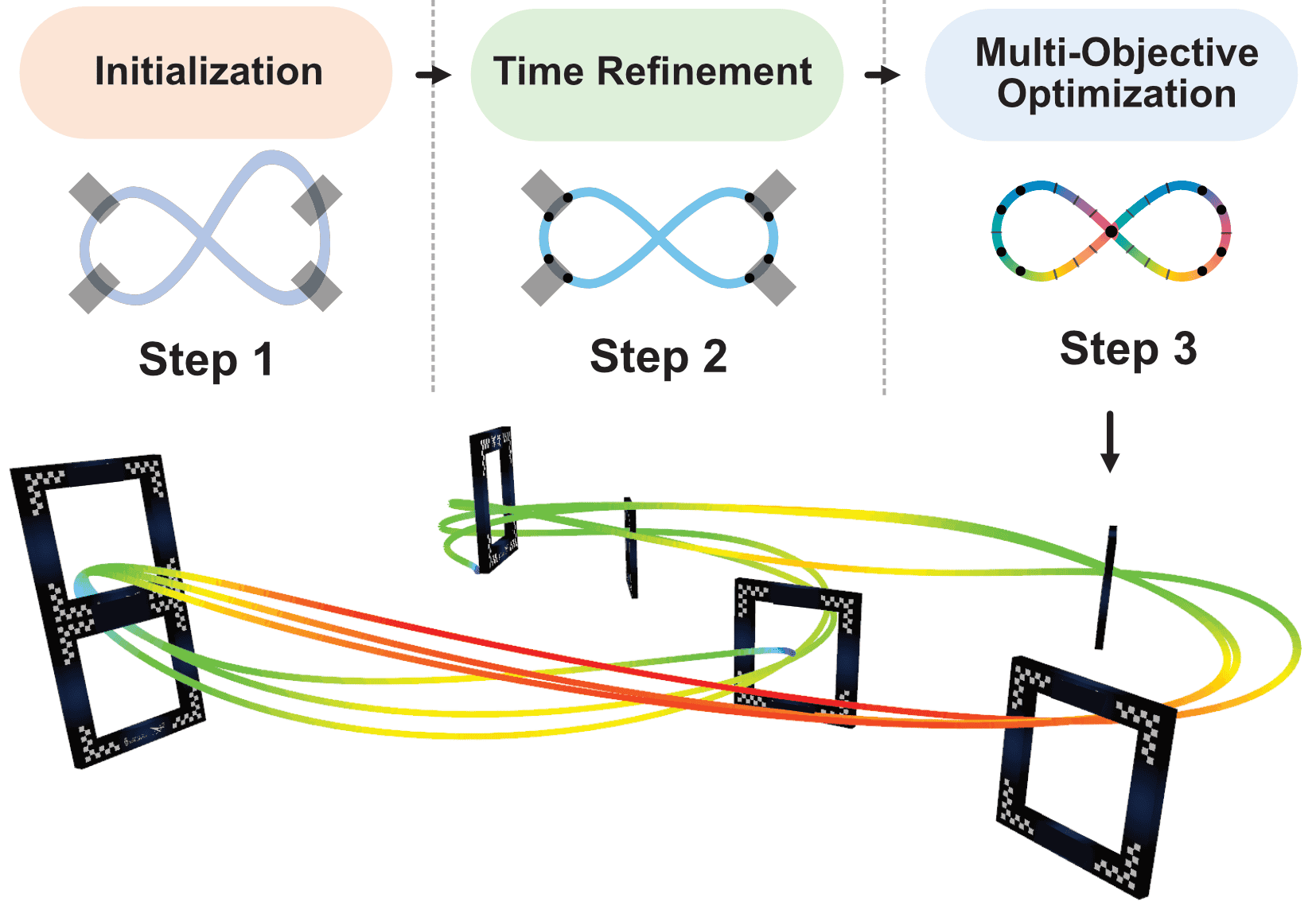}
\caption{The optimization strategy proceeds through three stages. Step 1: A polynomial-based approach rapidly generates a sub-time-optimal trajectory. Step 2: increasing segment density approximates the time-optimal solution to establish initial guesses and optimal traversal points. Step 3: A multi-stage multiple-shooting framework refines the trajectory by constraining these points as waypoints and integrating perception objectives.}\label{fig:optimization_diagram}
\vspace{-0.5em}
\end{figure}

However, since $\mathbf{y}$ is a piecewise polynomial, a significant number of continuity constraints must be formulated between segments to ensure a feasible trajectory. Wang et al. \cite{wang2022geometrically} present an analytical method to calculate the coefficients of the piecewise polynomial $\mathbf{y}$ using a sequence of waypoints, $\mathbf{p}_{w_i}$, and the time allocation for each segment, $T_{i}$. They prove that this solution is unique, given an effort-minimization cost and a specific polynomial order. Consequently, the optimization can be performed over $\mathbf{p}_{w_i}$ and $T_{i}$ instead of the polynomial coefficients, which greatly reduces the dimensionality of the variables. In the context of \ac{TOGT} trajectory planning, it leads the following approximation of (\ref{eq:time_optimal_trajectory_planning}):
\begin{subequations}\label{eq:polynomial_time_optimal_trajectory_planning}
\begin{align}
  \min_{\mathbf{p}_{w_{i}},T_{i}} & \sum_{i=1}^{N_{w}+1}T_{i},\\
  \text{s.t.} \quad & \mathbf{y}_{1}^{[s-1]}(0)=\mathbf{y}_{\text{init}},\quad \mathbf{y}_{N_{w}\!+\!1}^{[s-1]}(T_{N_{w}\!+\!1})=\mathbf{y}_{\text{term}},\\
  & \mathbf{y}_{i}^{[s-1]}(0)=\mathbf{y}_{i-1}^{[s-1]}(T_{i}),\\
  & \text{constraint in (\ref{eq:flatness_constraint})},\\ 
  & h_{\mathcal{G}_{i}}(\mathbf{p}_{w_{i}})\leq0,\\ 
  & -T_{i}<0,
  \end{align}
\end{subequations}
where $\mathbf{y}_{\text{init}}$ and $\mathbf{y}_{\text{term}}$ are the flat output states corresponding to the given initial and terminal quadrotor states, respectively. Finally, solving ~\eqref{eq:polynomial_time_optimal_trajectory_planning} following the procedure in \cite{qin2023time} yields a good initial guess of the \ac{TOGT} trajectory.

\textbf{Step 2: Pushing time optimality.}
We further enhance time optimality by subdividing each trajectory segment into $M_{p}$ polynomial pieces and solving the problem once again based on the previous solution. The rationale for this strategy, discussed in \cite{qin2024time}, is that there exists a minimum viable number of polynomial pieces capable of representing time-optimal trajectories with sufficiently high resolution, and therefore integrating more polynomials can effectively shrink the gap to the optimal trajectory. The enhanced problem formulation is given as
\begin{subequations}\label{eq:spline_time_optimal_trajectory_planning}
\begin{align}
\min_{\mathbf{P},\mathbf{T}}\  & \sum_{i=1}^{N_{w}+1}\sum_{j=1}^{M_{p}}T_{i,j},\\
\text{s.t.}\quad & \mathbf{y}_{1,1}^{[s-1]}(0)=\mathbf{y}_{\text{init}},\;\mathbf{y}_{N_{w}\!+\!1,M_{p}}^{[s-1]}(T_{N_{w}\!+\!1,M_{p}})=\mathbf{y}_{\text{term}},\\
  & \mathbf{y}_{i,j}^{[s\!-\!1]}(0)=\mathbf{y}_{i,j\!-\!1}^{[s\!-\!1]}(T_{i,j}),\\
  & \text{constraint in (\ref{eq:polynomial_time_optimal_trajectory_planning}d-e)},\\
  & T_{i,j}>0.
\end{align}
\end{subequations}
where the vector $\mathbf{P}\in\mathbb{R}^{3\times ((N_{w}+1)M-1)}$ contains all intermediate waypoints (both free and gate-related), and $\mathbf{T}\in \mathbb{R}^{(N_{w}+1)M}$ denotes the time allocation. The optimized trajectory typically deviates by less than 5\% from the theoretical time-optimal duration. Post-optimization, we recover the polynomial coefficients from $\mathbf{P}$ and $\mathbf{T}$, sample the trajectory, and compute the quadrotor states and inputs via (\ref{eq:flatness_map})

\textbf{Step 3: Multi-stage multiple-shooting.}
Based on the discreteized near-time-optimal trajectory, this step formulates a multiple-shooting problem with waypoint constraints. We denote $\hat{T}_{i}$ as the optimized time duration of the $i$-th segment from the second step. The node number of this segment is selected as $N_{i}=\hat{T}_{i}/\Delta \bar{t}$, where $\Delta \bar{t}$ is the desired sampling time. Then, the total node number becomes $N=\sum_{i=1}^{N_{w}+1} N_i$. The control input here is defined as $\mathbf{u}:=\mathbf{r}_{T}$. As a result, we obtain the following multi-stage multiple-shooting problem
\begin{subequations}\label{eq:discrete_time_perception_aware_time_optimal_trajectory_planning}
\begin{align}
\min_{\mathbf{x}_{k},\mathbf{u}_{k},\Delta t_{i}}\  & \sum_{i=1}^{N_{w}+1}T_{i}+\sum_{k=1}^{N}\mathcal{L}(\mathbf{x}_{k}, \mathbf{u}_{k}),\\
\text{s.t.}\quad & \mathbf{x}_{k_{i}+1}=\mathbf{F}(\mathbf{x}_{k_{i}},\mathbf{u}_{k_{i}},\Delta t_{i}),\\
  & \mathbf{x}_{0}=\mathbf{x}_{\text{init}},\quad\mathbf{x}_{N}=\mathbf{x}_{\text{term}},\\
  & \mathbf{x}_{\min}\leq\mathbf{x}_{k}\leq\mathbf{x}_{\max},\\
  & \mathbf{u}_{\min}\leq\mathbf{u}_{k}\leq\mathbf{u}_{\max},\\
  & \|\mathbf{p}_{i}-\mathbf{p}_{w_{i}}\|_{2}^{2}\leq0,\\
  & \Delta t_{i} > 0.\\
  & \text{constraints in (\ref{eq:soft_fov_constraints})},
\end{align}
\end{subequations}
where $T_{i}=N_{i}\Delta t_{i}$ and $\mathbf{x}_{k_{i}+1}=\mathbf{F}(\mathbf{x}_{k_{i}},\mathbf{u}_{k_{i}})$ is the discrete-time dynamics of the quadrotor derived from Section \ref{subsec:dynamics}. The position $\mathbf{p}_{i}$ is extracted from the state of the vehicle passing through the $i$-th gate, while $\mathbf{p}_{w_{i}}$ is obtained from the optimal gate-traversal points obtained in the last stage.

At each stage, the Runge-Kutta 4 (RK4) integrator uses a dynamically adjusted sampling time, $\Delta t_{i}$, based on the decision variable. With a proper initialization, we utilize IPOPT~\cite{wachter2002interior} to solve the optimization problem, employing CasADi~\cite{andersson2019casadi} as the automatic differentiation toolbox. In most cases, the optimized $\Delta t_{i}$ values merely fluctuate around the initial $\Delta  \bar{t}$, as the initial time allocation is already near-optimal thanks to the first two steps; this is the key reason for our remarkable convergence performance in handling free-end-time problems.

\subsection{Summary}

In essense, this pipeline adopts a divide-and-conquer methodology and effectively leverages the strengths of both polynomial and direct methods. Specifically, the polynomial approach in the first two steps quickly generates a near-time-optimal trajectory with an optimality gap of less than 5\%. Subsequently, the direct method utilizes its superior formulation flexibility to incorporate complex objectives and enhance trajectory representability, producing a near bang-bang control policy that is featured in time-optimal flights.

\section{Model Predictive Contouring Tracking Controller}\label{sec:mpctc}

In this section, we introduce a \ac{MPCTC} method for accurate trajectory following, applicable to any dynamically feasible trajectory, including those generated by the aforementioned methods.

\begin{figure}[!htbp]
\centering
\includegraphics[width=0.4\textwidth]{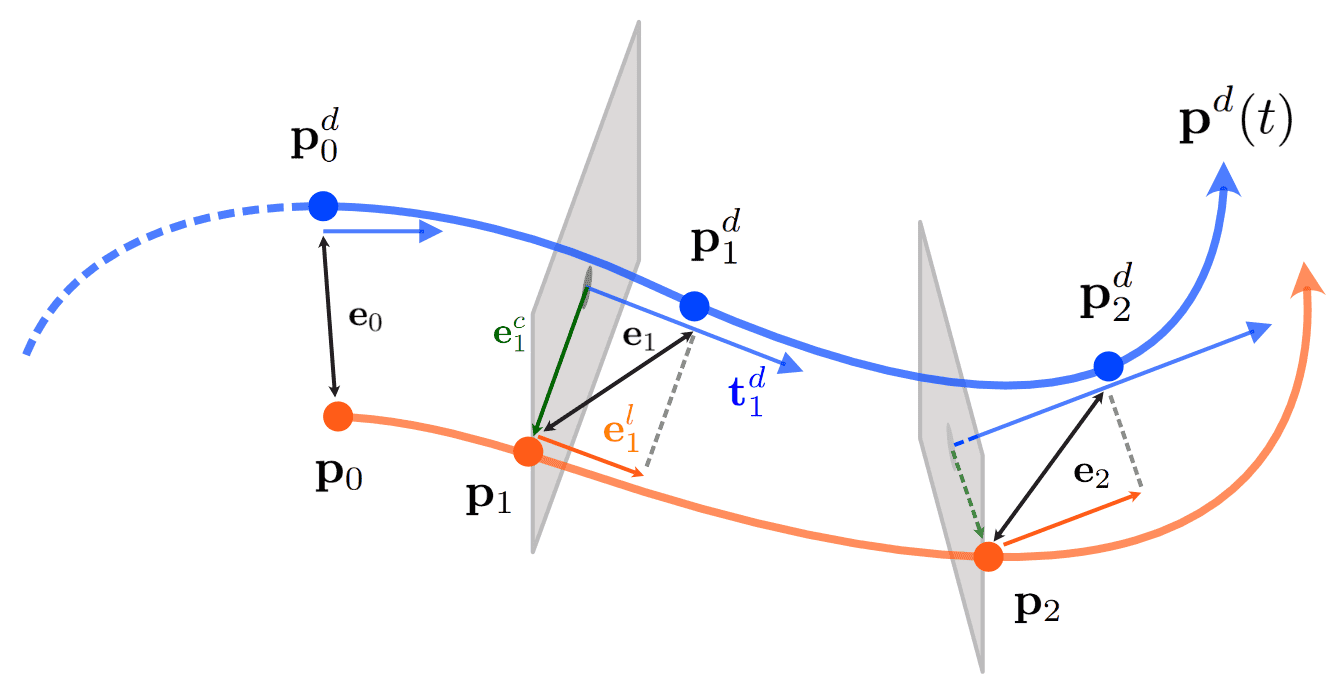}
\caption{Contouring error $\mathbf{e}^c$ and progress error $\mathbf{e}^l$.}\label{fig:mpctc}
\vspace{-0.3cm}
\end{figure}

\subsection{Overview}

Consider a task of tracking a desired time-parameterized trajectory $\mathbf{x}^d$ as shown in Fig. \ref{fig:mpctc}, and the current quadrotor state is $\mathbf{x}_{0}$, which is slightly off the reference path. Let's denote their corresponding positions as $\mathbf{p}^d$ and $\mathbf{p}_{0}$, respectively. If we continue to adhere to the reference time sequence of $\mathbf{x}^d$, the vehicle will inevitably cut corners due to both the initial position error and model mismatch. This occurs because the reference state, which is temporally ahead of the current and predicted states, continuously pulls the system forward, even when the vehicle has already deviated significantly from the desired path; in some senses, this could be a primary factor contributing to the significantly higher failure rate of the \ac{MPC} in racing scenarios \cite{song2023reaching}.

To address this issue, we propose two key enhancements. First, prior to reference state selection, the closest point on the reference path to the current position $\mathbf{p}_{0}$, denoted as $\mathbf{p}_{0}^{d}$, is identified. Second, the position error is decomposed into contouring and progress components, as in the \ac{MPCC} approach \cite{romero2022model}, with higher weighting assigned to the contouring error to enhance path-following accuracy.

\subsection{Determination of the Reference States}

We first search for $\mathbf{p}_{0}^{d}$ on the reference trajectory. The underlying optimization problem is described as:
\begin{equation}
\begin{aligned}
\min_{t} \quad & \|\mathbf{p} - \mathbf{p}^d(t)\|_2^2, \\
\text{s.t.} \quad & t \in [0, T^{d}], 
\end{aligned}
\end{equation}
where $T^{d}$ is the end time of the $\mathbf{p}^d$. We adopt a sampling-based method for simplicity and efficiency. Specifically, a fixed number of points are incrementally sampled along the reference path, beginning from the closest point identified in the previous iteration, and the process terminates when the distance between the sampled point and the current position exceeds that of the preceding point.

Given the initial reference point $\mathbf{p}_{0}^{d}$ and the corresponding reference time $t_0^d$, the first reference state can be defined as $\mathbf{x}^d_0=\mathbf{x}^d(t_0^d)$. The subsequent reference states are then obtained as
\begin{equation}
\begin{aligned}
\mathbf{x}^d_1 &= \mathbf{x}^d(t_0^d + \Delta t_{\text{mpc}}), \\
\mathbf{x}^d_2 &= \mathbf{x}^d(t_0^d + 2\cdot \Delta t_{\text{mpc}}), \\
&\ldots \\
\mathbf{x}^d_{N_{\text{mpc}}} &= \mathbf{x}^d(t_0^d + N_{\text{mpc}}\cdot \Delta t_{\text{mpc}}),
\end{aligned}
\end{equation}
where 
$\Delta t_{\text{mpc}}$ denotes the \ac{MPC} sampling interval and  $N_{\text{mpc}}$ represents the number of prediction nodes. 

\subsection{Contouring and Progress Errors}

Since the reference velocity is available, the tangential direction at each sampled point along the reference path can be directly obtained and is denoted by $\mathbf{t}_{k}^{d}$:
\begin{equation}
\mathbf{t}_{k}^{d}=\mathbf{v}_{k}^{d}/\|\mathbf{v}_{k}^{d}\|_{2}.
\end{equation}
This makes the computation of the contouring and progress errors very convenient. For an arbitrary node $k$, we can readily obtain the position tracking vector as:
\begin{equation}
\mathbf{e}_{k}=\mathbf{p}_{k}-\mathbf{p}_{k}^{d},
\end{equation}  
Provided the tangential vector, the above error can be decomposed into two components: contouring error, $\mathbf{e}_{k}^{c}$, and progress error, $\mathbf{e}_{k}^{l}$:
\begin{align}
\mathbf{e}_{k}^{l}&=(\mathbf{e}_{k}^{T}\mathbf{t}_{k}^{d})\cdot\mathbf{t}_{k}^{d},\\
\mathbf{e}_{k}^{c}&=\mathbf{e}_{k}-\mathbf{e}_{k}^{l}=(\mathbf{I}_{3}-\mathbf{t}_{k}^{d}(\mathbf{t}_{k}^{d})^{T})\mathbf{e}_{k}.
\end{align}
Note that when the magnitude of the reference velocity falls below a certain threshold, the last valid tangential vector is retained, or alternatively, a constant vector is used. The key distinction between \ac{MPCTC} and \ac{MPCC} lies in how the progress term is treated: in \ac{MPCTC}, it is predetermined by the input trajectory rather than treated as an optimization variable.

\subsection{Optimal Control Problem Formulation}
The goal is to, given the current state $\mathbf{x}_0$, find the optimal control strategy~$\pi(\mathbf{x})$ that best tracks the target trajectory. The optimal control problem is constructed as follows:
\begin{algbox}{Model Predictive Contour-Tracking Control}
\begin{argmini!}[l]
    {\mathbf{u}_{k}} 
    { \begin{aligned}[t]
        \sum_{k=0}^{N_{\text{mpc}}} & \|\mathbf{e}_{k}^{l}\|_{\mathbf{Q}_{l}}^{2} + \|\mathbf{e}_{k}^{c}\|_{\mathbf{Q}_{c}}^{2} + \|\mathbf{q}_{k}\ominus\mathbf{q}_{k}^{d}\|_{\mathbf{Q}_{q}}^{2} \\
        & + \|\mathbf{v}_{k}-\mathbf{v}_{k}^{d}\|_{\mathbf{Q}_{v}}^{2} + \|\boldsymbol{\omega}_{k}-\boldsymbol{\omega}_{k}^{d}\|_{\mathbf{Q}_{\omega}}^{2} \\
        & + \|\mathbf{f}_{k}-\mathbf{f}_{k}^{d}\|_{\mathbf{Q}_{f}}^{2} + \|\mathbf{u}_{k}\|_{\mathbf{Q}_{u}}^{2}
      \end{aligned} 
    }
    {\label{eq:mpc_contour_tracking}} 
    {\pi(\mathbf{x}_0) =} 
    \addConstraint{\mathbf{x}_{k+1}}{= F(\mathbf{x}_{k},\mathbf{u}_{k},\bar{dt})}{}
    \addConstraint{\text{constraints in (\ref{eq:discrete_time_perception_aware_time_optimal_trajectory_planning}d-e)}}{}{}
\end{argmini!}
\end{algbox}
where $\mathbf{Q}_{l}$, $\mathbf{Q}_{c}$, $\mathbf{Q}_{q}$, $\mathbf{Q}_{v}$, $\mathbf{Q}_{\omega}$, $\mathbf{Q}_{f}$, and $\mathbf{Q}_{u}$ are weights for tracking errors in progress, contouring, orientation, velocity, angular velocity, thrust, and thrust rates, in the respective dimmensions.

The controller operates in a receding horizon manner. We implement it using \texttt{acados}~\cite{verschueren2021} with the quadratic programming solver HPIPM~\cite{frison2020} to solve the optimal control problem with the real-time iteration scheme described in \cite{diehl2005real}.

\section{Results}\label{sec:results}

In this section, we evaluate the proposed time-optimal trajectory generation framework in terms of solution quality, time optimality, and enhancements to the position estimate of a vision-based localization system. All the quadrotor parameters used in this section are provided in Tab. \ref{tab:quad_config}, where the RPG configuration is for simulation, and the EXP configuration is for real-world drones. The same solver setup given in Tab. \ref{tab:optimization_params} is utilized for all tests.

\begin{table}[htbp!]
\centering
\caption{Quadrotor configurations}\label{tab:quad_config}
\tabcolsep=0.08cm	
\begin{threeparttable}
\begin{tabular}{@{}ccccccc@{}}
    \toprule
    & $m$ (kg) & $l$ (m) & $\mathbf{J}_{\text{diag}}$ (g$\cdot$m$^2$) & $(\underline{f},\overline{f})$ (N) & $c_{\tau}$  & $\overline{\boldsymbol{\omega}}$ (rad$\,$s$^{-\!1}$) \\ \midrule
    RPG & 0.7  & 0.125  & [2.4,1.8,3.7]     & [0.0, 8.5]       & 0.033       & [10, 10, 6]            \\

    
    EXP & 1.0  & 0.125  & [2.5,2.1,4.3]     & [0.0, 8.5]       & 0.013       & [10, 10, 6]            \\

    \bottomrule
\end{tabular}
\begin{tablenotes}[para,flushleft]
\item $\mathbf{J}_{\text{diag}}$ denotes the diagonal elements of the inertia matrix.
\end{tablenotes}
\end{threeparttable}
\end{table}

\subsection{Time-Optimal Trajectory Generation}

\begin{figure}[!htbp]
    \centering
    \tikzstyle{every node}=[font=\footnotesize]
    \begin{tikzpicture}[>=stealth]
        \node [inner sep=0pt, outer sep=0pt] (img1) at (0,0)
        {\includegraphics[width=8.1cm]{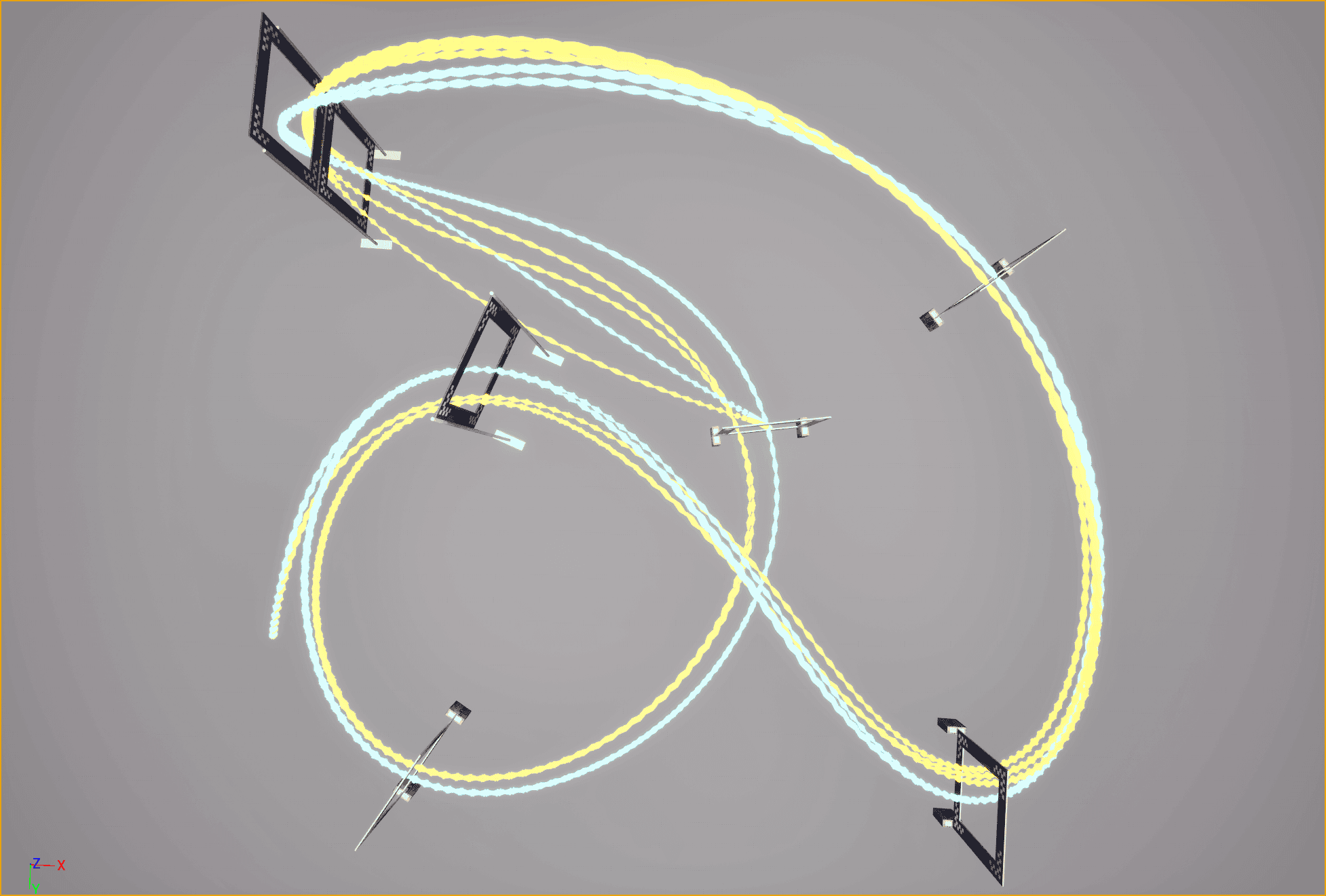}};

        \node [inner sep=0pt, outer sep=0pt, below=0mm of img1.south west, anchor=north west] (img2) 
        {\includegraphics[width=2.7cm]{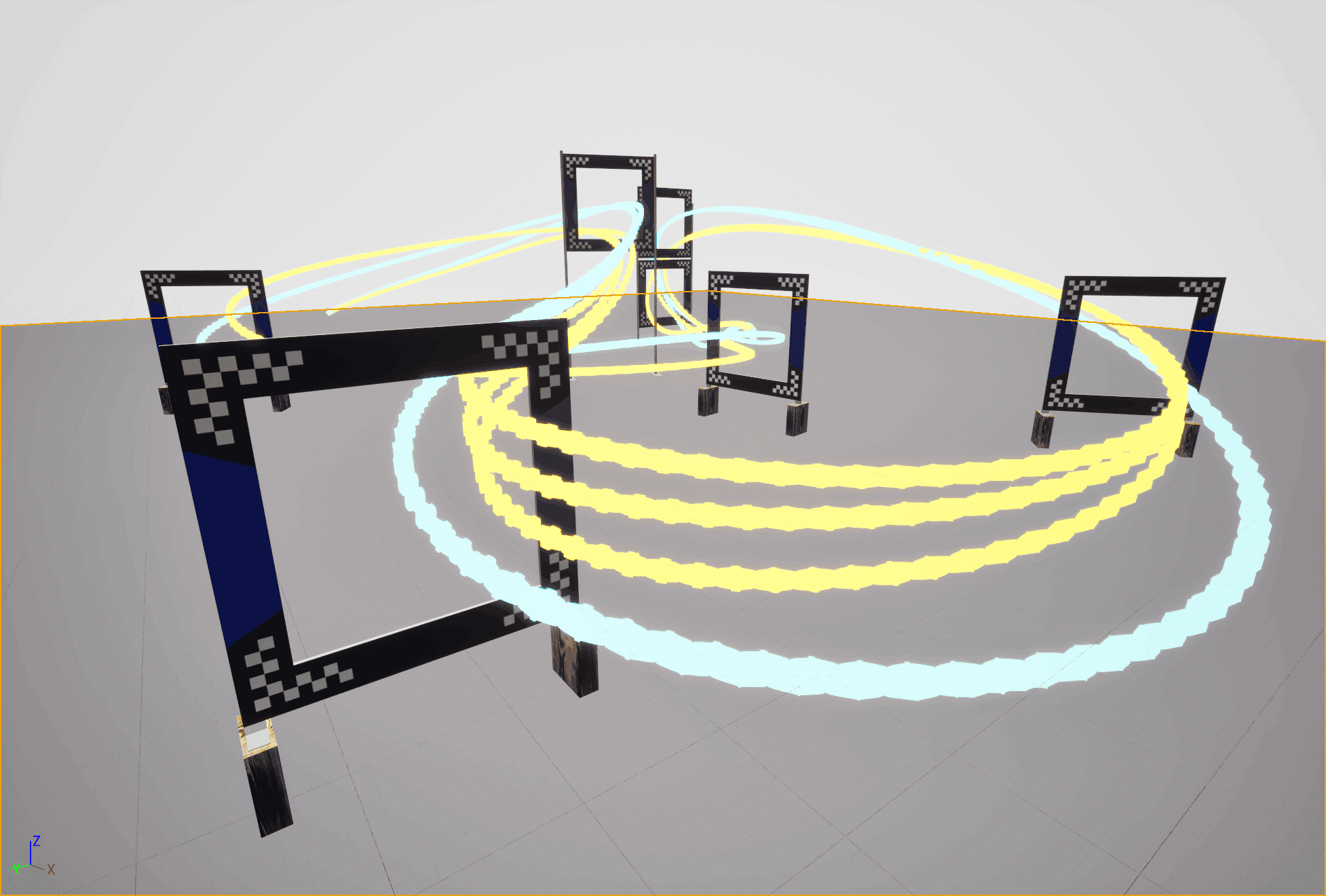}};
        
        \node [inner sep=0pt, outer sep=0pt, right=0mm of img2] (img3) 
        {\includegraphics[width=2.7cm]{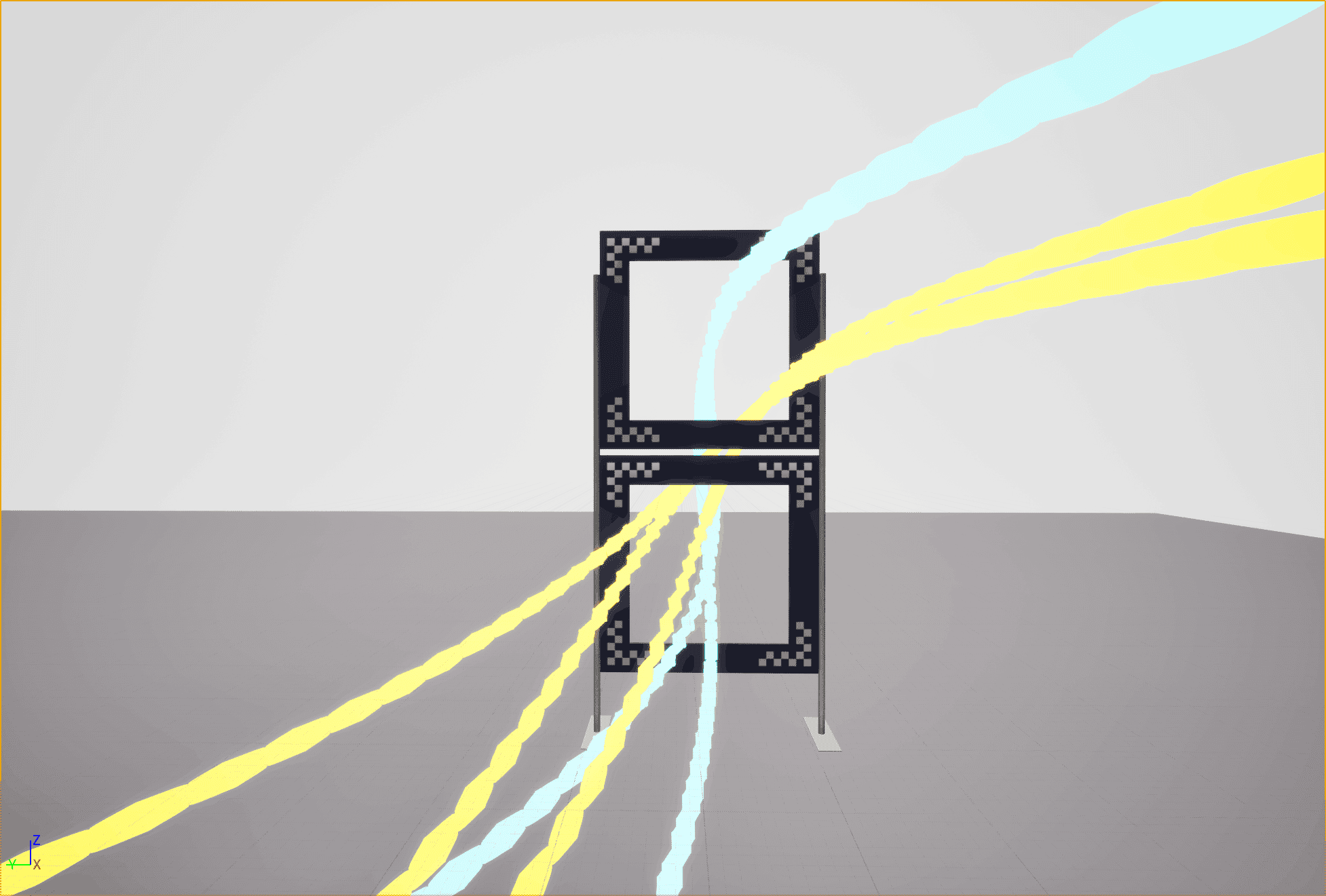}};
        
        \node [inner sep=0pt, outer sep=0pt, right=0mm of img3] (img4) 
        {\includegraphics[width=2.7cm]{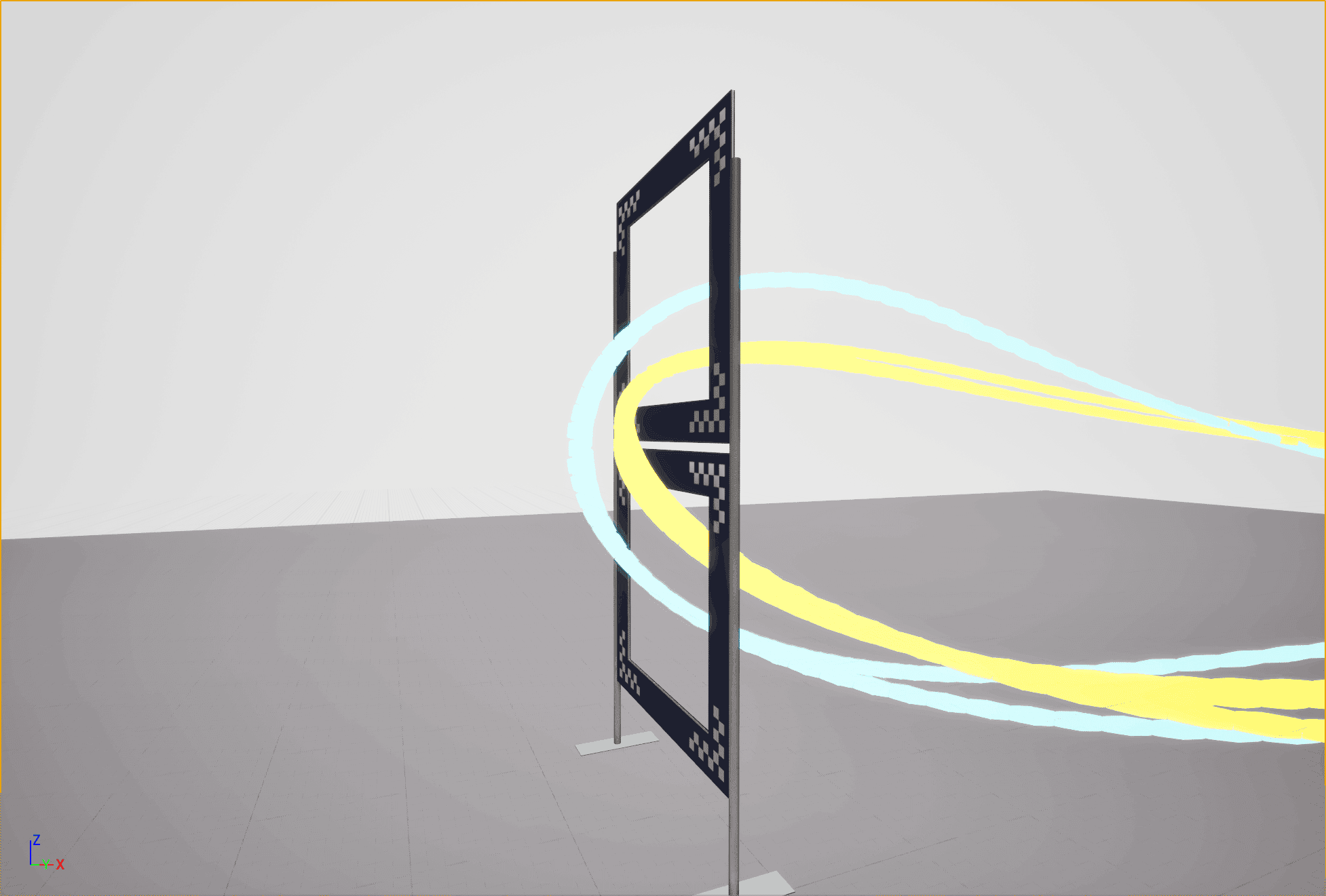}};

        \node [inner sep=0pt, outer sep=0pt, below=2mm of img2.south west, anchor=north west] (img5) 
        {\includegraphics[width=8.2cm]{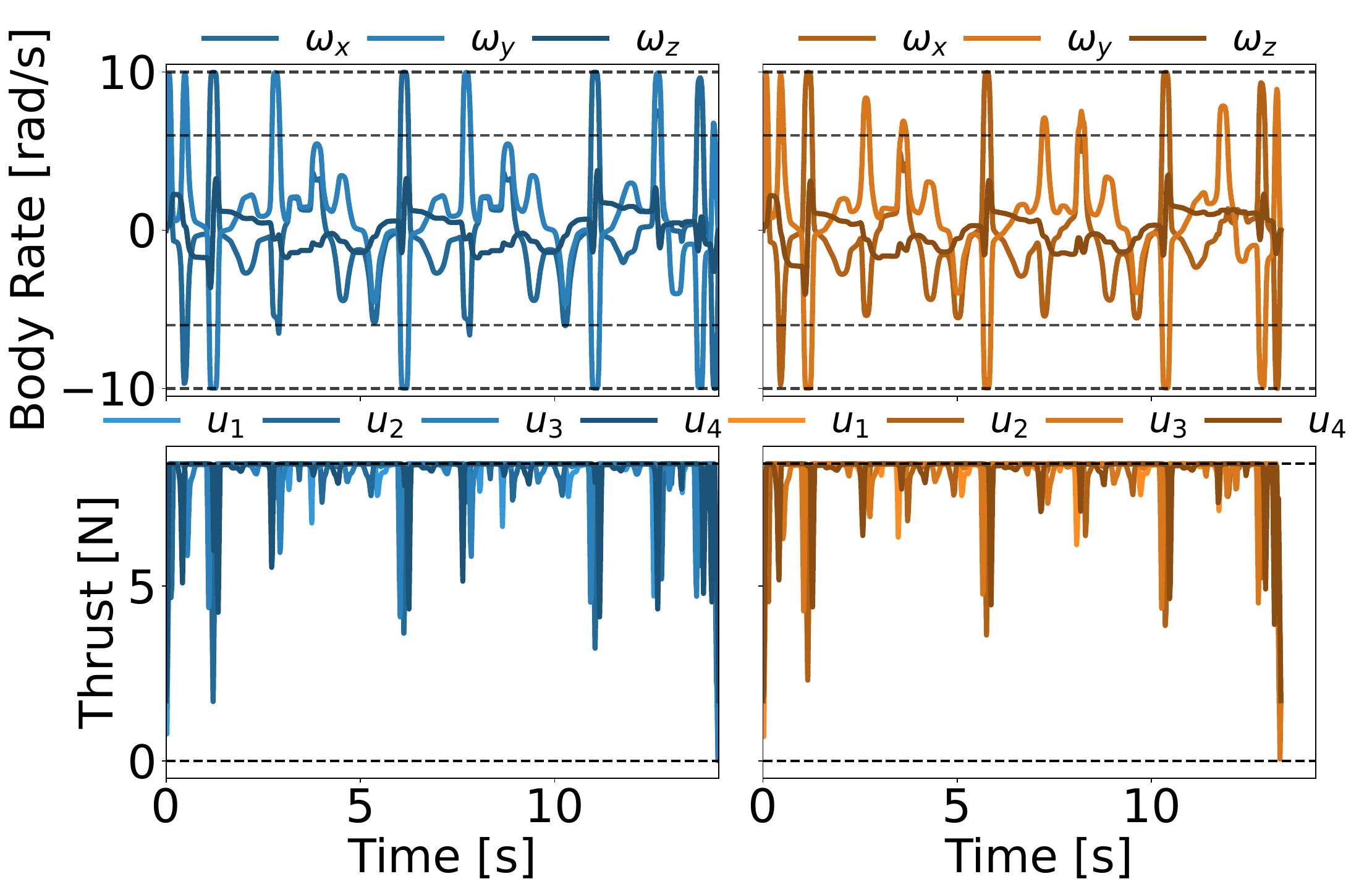}};
    \end{tikzpicture}
    \caption{\ac{TOWP} trajectory (blue) versus \ac{TOGT} trajectory (yellow) on the Split-S track. Note that while both trajectories satisfy the single-rotor thrust and body-rate constraints, the TOGT trajectory has more timing advantage as it makes better use of the navigable space of each gate.}
    \label{fig:togt_vs_towp}
\end{figure}



We first consider a basic racing scenario with a Split-S track shown in ig. \ref{fig:togt_vs_towp} where time is the only meirc and no other objective, such as perception awareness, is incorporated. We benchmark different approaches in terms of their resulting durations finishing two laps. The start and en positions are $(-5.0, 4.75, 1.1)\,\text{m}$ and $(4.5, -0.45, 1.55)\,\text{m}$, respectively, both with stationary flight status. Note that one lap is from the first gate at $(-0.6, -0.86, 3.78)\,\text{m}$ to the last gate at $(-1.95, 6.81, 1.55)\,\text{m}$. All gates are sqaure with the same inner side length of 1.45 m. The quadrotor's collision volumn is modeled as a ball with a diameter of $0.4\,\text{m}$.

\begin{table}[htbp!]
\centering
\caption{Parameters of optimization}\label{tab:optimization_params}
\begin{tabular}{@{}cccccc@{}}
\toprule
                  & Optimization Step         & Paramter Name                         & Symbol                  & Value     &  \\ \midrule
                  & Step II                   & Num. of polynomial pieces & $M_{p}$                 & 5         &  \\
\multirow{3}{*}{} & \multirow{3}{*}{Step III} & Initial sampling time                 & $\Delta \bar{t}$        & 2 ms     &  \\
                  &                           & Convergence tolerance                 & $\epsilon_{\text{tol}}$ & $1e^{-5}$ &  \\
                  &                           & Max. iteration                        & $N_{\text{max}}$               & 5000      &  \\ \cmidrule(l){2-6} 
\end{tabular}
\end{table}

\begin{table}[htbp]
\centering
\caption{Trajectory benchmarking in the Split-S track.}\label{tab:benchmarking}
\resizebox{\linewidth}{!}{
    \begin{tabular}{lccc}
        \toprule
        Methods & Duration (s) & Comp. Time (s)& Length (m) \\
        \midrule
        Pure Time Optimality & & & \\
        \backgroundcolor
        \tikz[baseline=-0.6ex] \node[circle, draw={rgb,1:red,1.0; green,0.647; blue,0.0}, fill={rgb,1:red,1.0; green,0.647; blue,0.0}, line width=0.8pt, inner sep=0pt, minimum size=6pt] {}; CPC\cite{foehn2021time} & 14.25 & 2631.0 & 210.34 \\
        \tikz[baseline=-0.6ex] \node[circle, draw={rgb,1:red,0.365; green,0.733; blue,0.388}, fill={rgb,1:red,0.365; green,0.733; blue,0.388}, line width=0.8pt, inner sep=0pt, minimum size=6pt] {}; Poly \cite{qin2023time} & 16.75 & 6.13 & 221.91 \\
        \backgroundcolor
        \tikz[baseline=-0.6ex] \node[circle, draw={rgb,1:red,0.0; green,0.45; blue,0.0}, fill={rgb,1:red,0.0; green,0.45; blue,0.0}, line width=0.8pt, inner sep=0pt, minimum size=6pt] {}; Fast-Fly\cite{zhou2023efficient} & 14.25 & 256.48 & 210.13 \\
        \tikz[baseline=-0.6ex] \node[circle, draw={rgb,1:red,0.247; green,0.369; blue,0.710}, fill={rgb,1:red,0.247; green,0.369; blue,0.710}, line width=0.8pt, inner sep=0pt, minimum size=6pt] {}; \ac{TOWP}-Ours & 14.22 & 79.37 & 209.27 \\
        \backgroundcolor
        \tikz[baseline=-0.6ex] \node[circle, draw={rgb,1:red,0.867; green,0.09; blue,0.09}, fill={rgb,1:red,0.867; green,0.09; blue,0.09}, line width=0.8pt, inner sep=0pt, minimum size=6pt] {}; \ac{TOGT}-Ours & \textbf{13.33} & \textbf{68.63} & \textbf{192.82} \\
        
        Perception-Aware \ac{TOWP}-Ours & & & \\
        \tikz[baseline=-0.6ex] \node[regular polygon, regular polygon sides=3, rotate=0, draw={rgb,1:red,0.247; green,0.369; blue,0.710}, fill={rgb,1:red,0.247; green,0.369; blue,0.710}, line width=0.8pt, inner sep=0pt, minimum size=6pt] {}; LA & 15.36 & \textbf{144.21} & \textbf{205.59} \\
        \backgroundcolor
        \tikz[baseline=-0.6ex] \node[regular polygon, regular polygon sides=4, rotate=0, draw={rgb,1:red,0.247; green,0.369; blue,0.710}, fill={rgb,1:red,0.247; green,0.369; blue,0.710}, line width=0.8pt, inner sep=0pt, minimum size=6pt] {}; FOV & 15.25 & 210.39 & 216.80 \\
        \tikz[baseline=-0.6ex] \node[regular polygon, regular polygon sides=5, rotate=0, draw={rgb,1:red,0.247; green,0.369; blue,0.710}, fill={rgb,1:red,0.247; green,0.369; blue,0.710}, line width=0.8pt, inner sep=0pt, minimum size=6pt] {}; PUM & \textbf{14.94} & 267.67 & 210.77 \\
        \backgroundcolor
        \tikz[baseline=-0.6ex] \node[regular polygon, regular polygon sides=6, rotate=0, draw={rgb,1:red,0.247; green,0.369; blue,0.710}, fill={rgb,1:red,0.247; green,0.369; blue,0.710}, line width=0.8pt, inner sep=0pt, minimum size=6pt] {}; FOV-PUM & 15.52 & 455.41 & 210.48 \\
        \tikz[baseline=-0.6ex] \node[regular polygon, regular polygon sides=6, rotate=90, draw={rgb,1:red,0.247; green,0.369; blue,0.710}, fill={rgb,1:red,0.247; green,0.369; blue,0.710}, line width=0.8pt, inner sep=0pt, minimum size=6pt] {}; LA-FOV & 15.37 & 245.75 & 211.31 \\
        \backgroundcolor
        \tikz[baseline=-0.6ex] \node[regular polygon, regular polygon sides=8, rotate=0, draw={rgb,1:red,0.247; green,0.369; blue,0.710}, fill={rgb,1:red,0.247; green,0.369; blue,0.710}, line width=0.8pt, inner sep=0pt, minimum size=6pt] {}; LA-PUM & 15.14 & 459.79 & 209.05 \\
        \tikz[baseline=-0.6ex] \node[
            star, 
            star points=5, 
            star point ratio=2.25, 
            rotate=0, 
            draw={rgb,1:red,0.247; green,0.369; blue,0.710}, 
            fill={rgb,1:red,0.247; green,0.369; blue,0.710}, 
            line width=0.8pt, 
            inner sep=0pt, 
            minimum size=6pt
        ] {}; LA-FOV-PUM & 15.62 & 467.74 & 210.70 \\
        
        Perception-Aware \ac{TOGT}-Ours & & & \\

        \tikz[baseline=-0.6ex] \node[regular polygon, regular polygon sides=3, rotate=0, draw={rgb,1:red,0.867; green,0.09; blue,0.09}, fill={rgb,1:red,0.867; green,0.09; blue,0.09}, line width=0.8pt, inner sep=0pt, minimum size=6pt] {}; LA & 14.62 & \textbf{169.29} & \textbf{187.83} \\
        \backgroundcolor
        \tikz[baseline=-0.6ex] \node[regular polygon, regular polygon sides=4, rotate=0, draw={rgb,1:red,0.867; green,0.09; blue,0.09}, fill={rgb,1:red,0.867; green,0.09; blue,0.09}, line width=0.8pt, inner sep=0pt, minimum size=6pt] {}; FOV & 14.41 & 263.51 & 195.55 \\
        \tikz[baseline=-0.6ex] \node[regular polygon, regular polygon sides=5, rotate=0, draw={rgb,1:red,0.867; green,0.09; blue,0.09}, fill={rgb,1:red,0.867; green,0.09; blue,0.09}, line width=0.8pt, inner sep=0pt, minimum size=6pt] {}; PUM & \textbf{14.20} & 331.11 & 196.05 \\
        \backgroundcolor
        \tikz[baseline=-0.6ex] \node[regular polygon, regular polygon sides=6, rotate=0, draw={rgb,1:red,0.867; green,0.09; blue,0.09}, fill={rgb,1:red,0.867; green,0.09; blue,0.09}, line width=0.8pt, inner sep=0pt, minimum size=6pt] {}; FOV-PUM & 14.87 & 351.16 & 199.87 \\
        \tikz[baseline=-0.6ex] \node[regular polygon, regular polygon sides=6, rotate=90, draw={rgb,1:red,0.867; green,0.09; blue,0.09}, fill={rgb,1:red,0.867; green,0.09; blue,0.09}, line width=0.8pt, inner sep=0pt, minimum size=6pt] {}; LA-FOV & 14.56 & 177.64 & 196.78 \\
        \backgroundcolor
        \tikz[baseline=-0.6ex] \node[regular polygon, regular polygon sides=8, rotate=0, draw={rgb,1:red,0.867; green,0.09; blue,0.09}, fill={rgb,1:red,0.867; green,0.09; blue,0.09}, line width=0.8pt, inner sep=0pt, minimum size=6pt] {}; LA-PUM & 14.24 & 328.08 & 194.86 \\
        \tikz[baseline=-0.6ex] \node[star, 
            star points=5, 
            star point ratio=2.25, 
            rotate=0, draw={rgb,1:red,0.867; green,0.09; blue,0.09}, fill={rgb,1:red,0.867; green,0.09; blue,0.09}, line width=0.8pt, inner sep=0pt, minimum size=6pt] {}; LA-FOV-PUM & 15.65 & 459.02 & 196.92 \\ 
        \bottomrule
    \end{tabular}
}
\begin{minipage}{\linewidth}
    \begin{tablenotes}[para,flushleft]
        \small
        \item \textit{Note:} Our approaches are marked in blue or red. For example, \tikz[baseline=-0.6ex] \node[regular polygon, regular polygon sides=5, rotate=0, draw={rgb,1:red,0.867; green,0.09; blue,0.09}, fill={rgb,1:red,0.867; green,0.09; blue,0.09}, line width=0.8pt, inner sep=0pt, minimum size=6pt] {}; represents optimizing a \ac{TOGT} trajectory with \ac{PUM} objective.
    \end{tablenotes}
\end{minipage}
\vspace{-0.2cm}
\end{table}

We qualitatively evaluate the generated flight trajectories shown in Fig.~\ref{fig:togt_vs_towp}. We see that the \ac{TOGT} trajectory (yellow) clearly has an advantage over the \ac{TOWP} trajectory (blue) in terms of timing as it traverses as close as possible to the inner edges of the gates. This is attributed to the explicit modeling of the gate geometry. As a result, while the \ac{TOWP} trajectory produces a 4.14 s of lap time, the \ac{TOGT} can push this limit to 3.84 s, which is 7.2\% shorter. Moreover, the total duration of the \ac{TOGT} trajectory is 6.3\% shorter than that of the \ac{TOWP}, with a 12.4 m reduction in the total required travel distance. It is worth mentioning that the shorter lap time of the \ac{TOGT} trajectory is not due to violating body rate or single-rotor thrust constraints. As displayed in the bottom panels of Fig.~\ref{fig:togt_vs_towp}, both trajectories meet the same actuation constraints. This result verifies that incorporating gate constraints into the planning problem is essential to achieve the true minimum lap time.

\begin{figure}[!htbp]
    \centering
\tikzstyle{every node}=[font=\footnotesize]
\begin{tikzpicture}[>=stealth]
\node [inner sep=0pt, outer sep=0pt] (img1) at (0,0)
{\includegraphics[width=8.0cm]{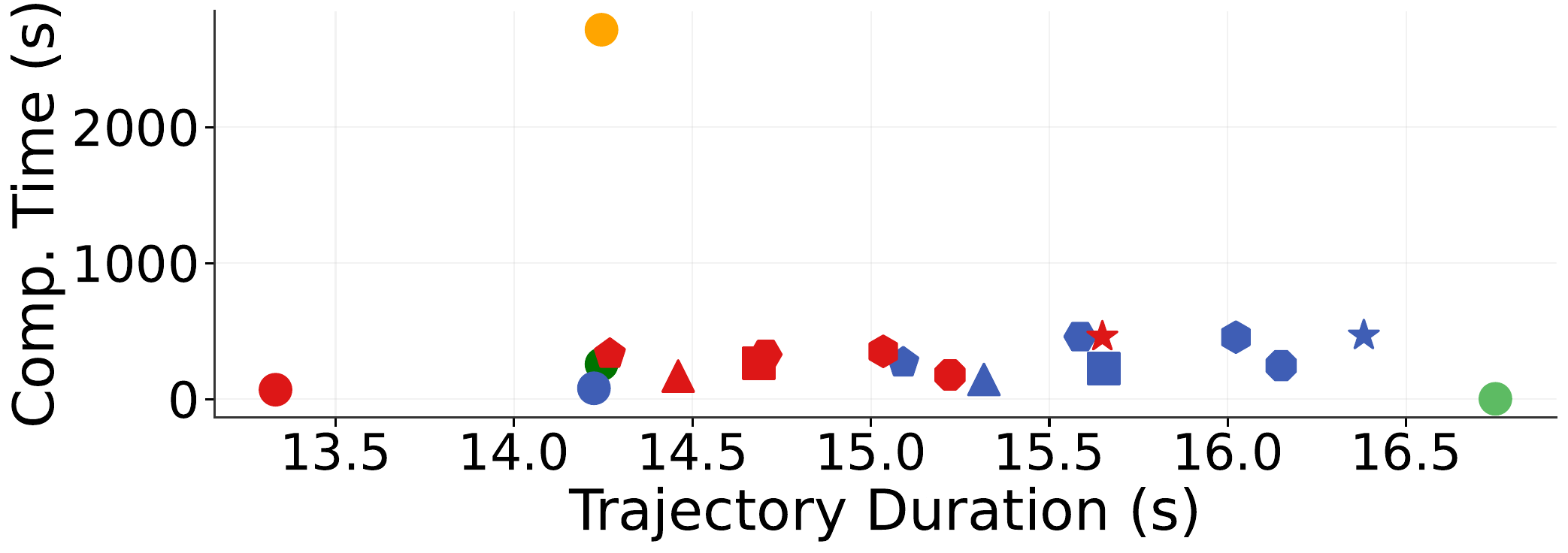}};
\end{tikzpicture}
\caption{Quantitative comparison of computational cost and optimized trajectory duration. Our \ac{TOGT} {\protect\tikz[baseline=-0.6ex] \protect\node[circle, draw={rgb,1:red,0.867; green,0.09; blue,0.09}, fill={rgb,1:red,0.867; green,0.09; blue,0.09}, line width=0.8pt, inner sep=0pt, minimum size=6pt] {};} formulation yields the best performance in both metrics.}
\label{fig:methods_comparison}
\vspace{-0.2cm}
\end{figure}

Then, we compare our methods with state-of-the-art approaches. In Tab.~\ref{tab:benchmarking}, we present trajectory durations, computation times, and total trajectory lengths from different methods. While \cite{zhou2023efficient} and \cite{foehn2021time} converge to trajectories with durations comparable to our waypoint-constrained approach, our method achieves these results with significantly lower computational overhead. It is almost four times faster than Fast-Fly \cite{zhou2023efficient} and more than thirty times faster than CPC \cite{foehn2021time}. Although Reference \cite{qin2023time} is more computationally efficient, the resulting timing is sub-optimal, which is 2.52 s longer than our approach. After switching the racing mode from waypoint passing to gate traversing, our approach triumphs in all metrics. The total duration is 0.89 s shorter than all the waypoint-flight methods. Moreover, even the computation time drops due to the relaxation of the geometric constraints.

The benefits of modeling gates extend beyond shorter lap times; it empowers the planner to handle collisions explicitly. For instance, we found that \ac{TOWP} often results in mission failure by generating trajectories that collide with certain gates. This is especially prevalent in complex layouts, such as the two tracks shown in Fig. \ref{fig:figure8_tii_dive}(b) and Fig. \ref{fig:figure8_tii_dive}(c). By accounting for gate geometry and orientation, \ac{TOGT} overcomes these issues. As illustrated in these figures, \ac{TOGT} trajectories not only identify shortcuts through square, dive, and circle gates but also ensure collision avoidance by specifying entry/exit sides.

\begin{table}[!htbp]
\centering
\caption{\ac{FPV} camera configurations}\label{tab:cam_config}
\begin{threeparttable}
\begin{tabular}{@{}cccccccc@{}}
\toprule
&  & $\alpha_{\text{max}}\;(^{\circ})$ & $\beta_{\text{max}}\;(^{\circ})$ & $Z_{\text{min}}\;(\text{m})$ & $\theta_{\text{tilt}}\;(^{\circ})$ & $\sigma_{u\!/\!v}\;(\text{px})$ &  \\ \midrule
& RPG & 128.1                              & 72.2                              & 0.3                  & 30.0                                                                               & 10.0                        &  \\
& EXP & 128.1                              & 72.2                              & 0.3                   & 0.0                                                                               & 10.0                        &  \\
\bottomrule
\end{tabular}
\begin{tablenotes}[para,flushleft]
Note that $\theta_{\text{tilt}}$ denotes the tilted angle of the camera.
\end{tablenotes}
\end{threeparttable}
\end{table}

\subsection{Perception-Aware Time-Optimal Flights}\label{subsec:patof}

We incorporate perception awareness into the trajectory optimization and observe the differences in the flight behavior. We assume the quadrotor is equipped with a forward-looking camera (See Tab.~\ref{tab:cam_config} for the detailed parameters), and features of landmarks within the FOV can always be extracted and identified. In this experiment, each landmark (i.e., the racing gate) has four features in total at its inner corners. The optimization parameters for each perception objective are offered in Tab.~\ref{tab:perception_params}.

\begin{table}[!htbp]
\centering
\caption{Parameters of perceptual costs}\label{tab:perception_params}
\begin{threeparttable}
\begin{tabular}{@{}ccccc@{}}
\toprule
\multicolumn{1}{l}{} & \multicolumn{1}{l}{Perception Type} & \multicolumn{1}{l}{Symbol} & \multicolumn{1}{l}{Value} & \multicolumn{1}{l}{} \\ \midrule
& Motion Regularization                & $w_{\text{jerk}}$          & $1e^{-7}$               &                      \\
& \multirow{2}{*}{LA}                 & $w_{\text{LA}}$            & 0.005                      &                      \\
&                                     & $\lambda_{\text{LA}}$      & 3.0                       &                      \\
& FOV                                 & $w_{\text{FOV}}$           & 0.05                       &                      \\
& \multirow{2}{*}{PUM}                 & $w_{\text{PUM}}$            & 0.0003                     &                      \\
&                                     & $\lambda_{v}$      & 10                      &                      \\ \bottomrule                    
\end{tabular}
\begin{tablenotes}[para,flushleft]
Note that the motion stabilization is applied to all perception-optimized planning.
\end{tablenotes}
\end{threeparttable}
\end{table}

In terms of the evaluation of the resulting visual quality, we sample the planned trajectory at 0.05 s intervals and compute the position uncertainty using our derived metric. If the uncertainty exceeds 2 m, a threshold beyond which the estimate becomes unreliable due to significant potential bias, we model the error accumulation by incrementing the last valid uncertainty value by 0.1 m for each subsequent sample. Finally, we visualize this uncertainty by plotting a circle whose radius corresponds to the accumulated uncertainty value at that timestamp. The results are given in Fig.~\ref{fig:splits_with_perception_awareness},

We see that when time is the sole optimization variable (see Fig.~\ref{fig:splits_with_perception_awareness}(a)), the quadrotor ignores perception requirements and often orients the camera toward regions where no gates are visible. This behavior leads to a rapid expansion of the uncertainty radius, signaling a significant degradation in localization performance. The \ac{LA} and \ac{FOV} effectively mitigate the explosion of uncertainty in varying degree, but in different manners. As the names imply, the \ac{LA} objective encourages the quadrotor to orient toward the future path. When the upcoming gate happens to be aligned with the future reference direction, the uncertainty can be reduced by the corresponding gate observations. However, if this is not the case, as seen in the rightmost trajectory segment of Fig.~\ref{fig:splits_with_perception_awareness}(b), the uncertainty remains at a high level due to the absence of visual landmarks. In contrast, the \ac{FOV} objective (see Fig.~\ref{fig:splits_with_perception_awareness}(c)) produces a more consistent uncertainty pattern by actively maintaining the gate within the camera's view. Notably, however, \ac{FOV} exhibits a marked increase in uncertainty when the quadrotor just passes the last gate. This occurs because the quadrotor requires a certain transition period to reorient its camera toward the next gate, leading to a temporary loss of visual tracking. Furthermore, we observe that the overall uncertainty area is shrunk dramatically by incorporating the \ac{PUM} objective. As illustrated in Fig.~\ref{fig:splits_with_perception_awareness}(d-f), the quadrotor automatically learns a strategy of orienting toward regions where the highest number of gates are visible. This behavior results in a significant improvement in the richness and diversity of visible landmarks, leading to consistently small uncertainty circles throughout the trajectory.





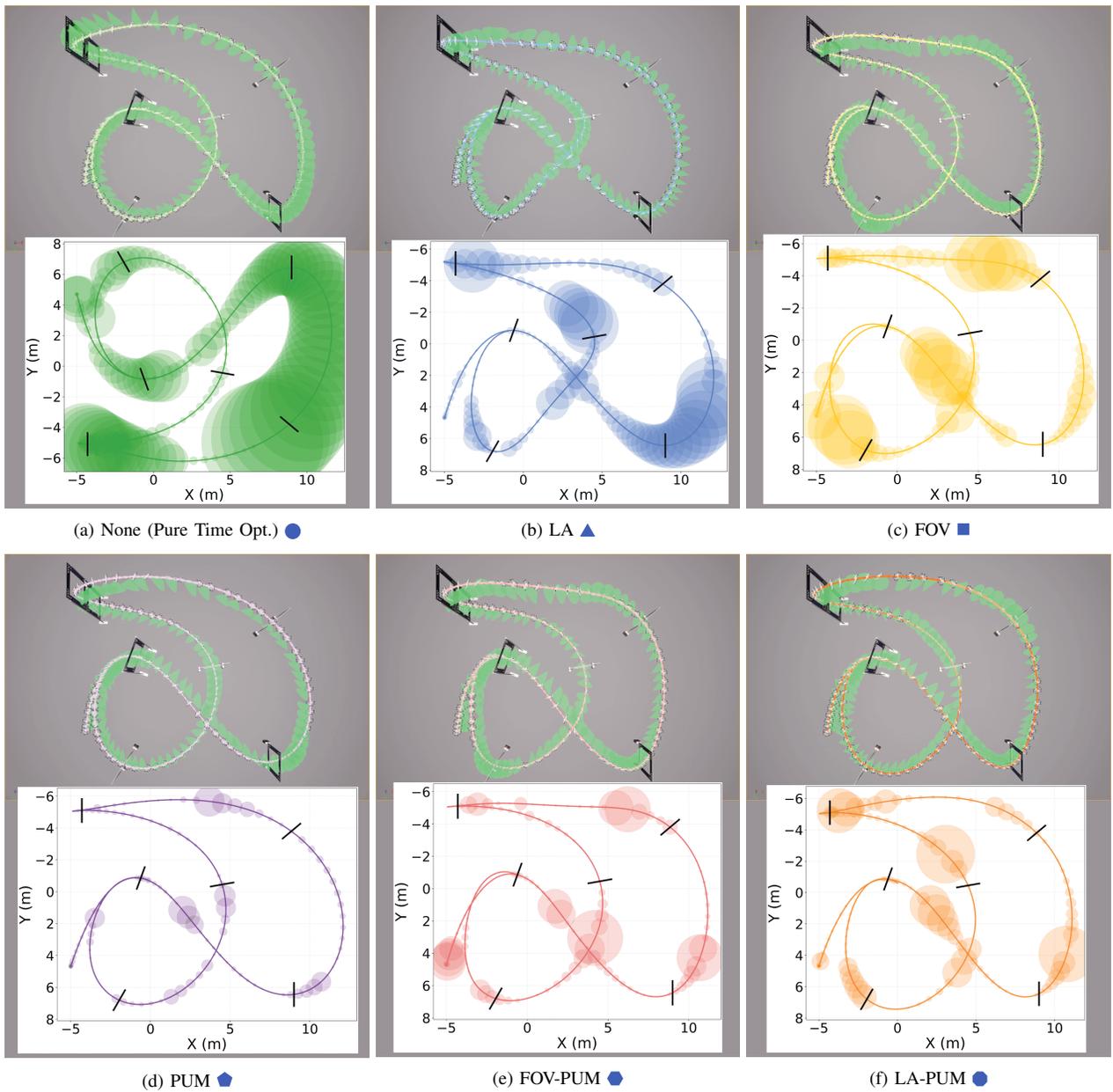
\begin{figure*}[t!]
    \centering
\tikzstyle{every node}=[font=\footnotesize]
\begin{tikzpicture}[>=stealth]
\node [inner sep=0pt, outer sep=0pt] (img1) at (0,0)
{\includegraphics[width=5.5cm]{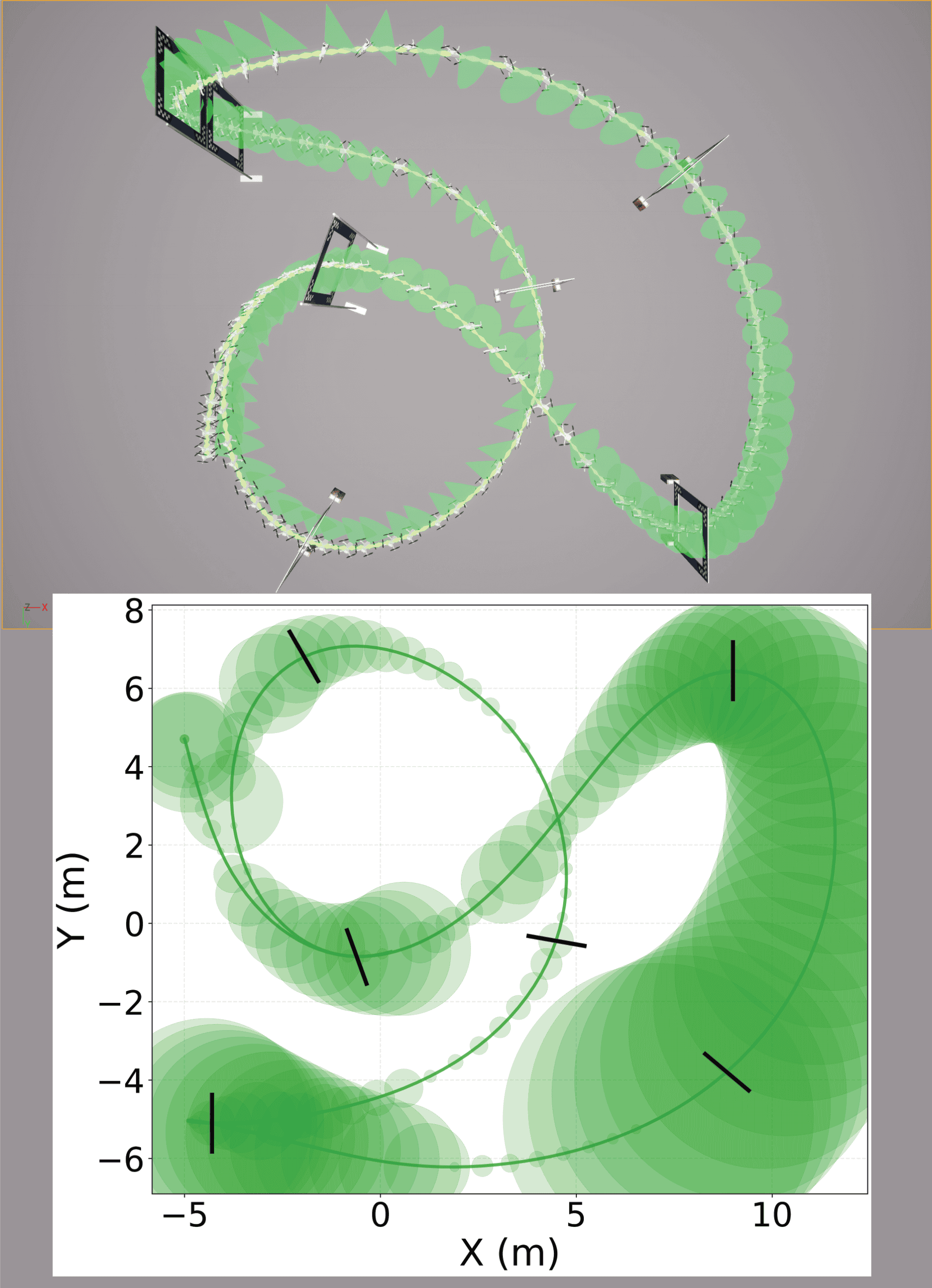}};
\node [inner sep=0pt, outer sep=0pt, right=1mm of img1] (img2) {\includegraphics[width=5.5cm]{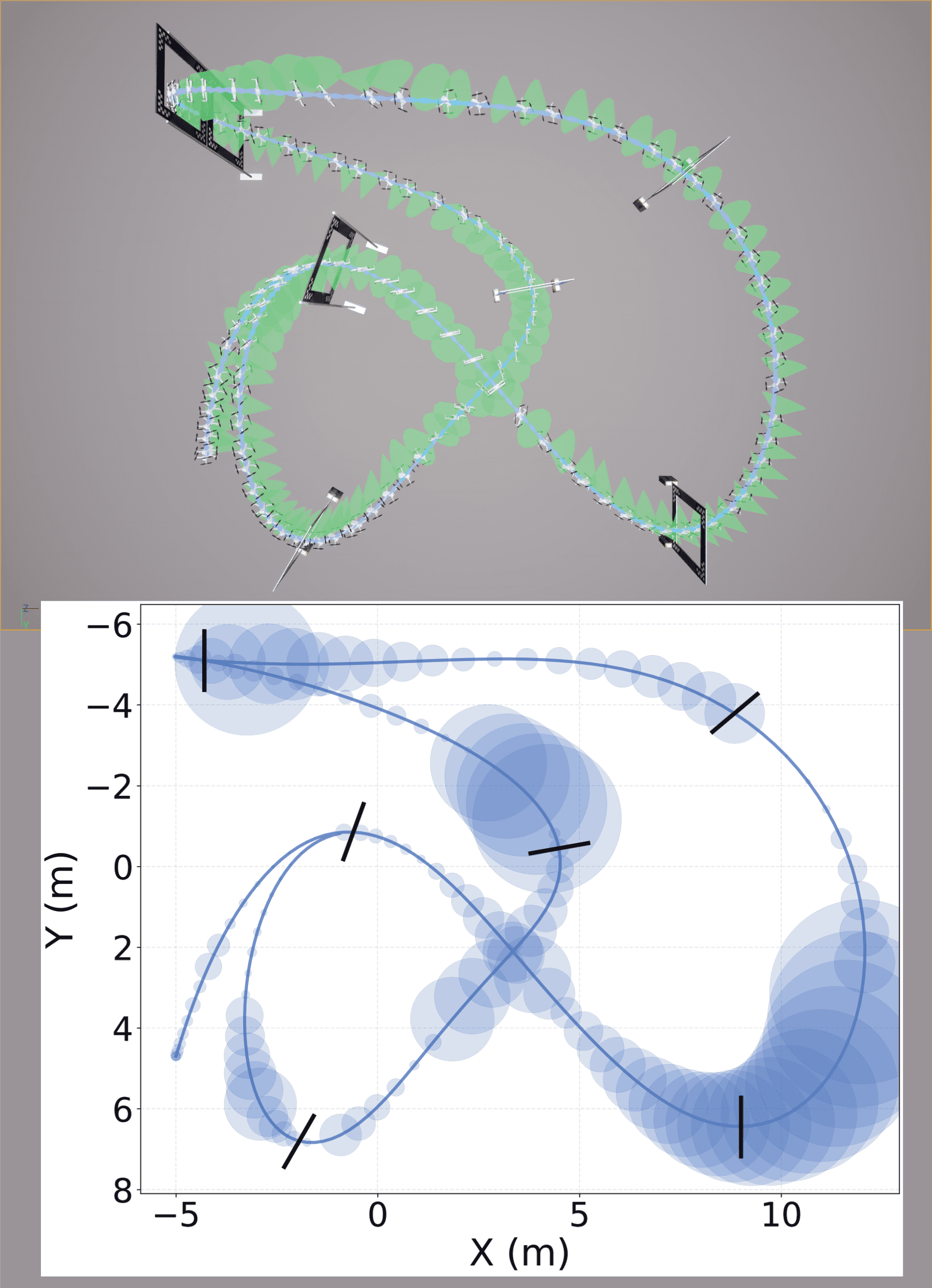}};
\node [inner sep=0pt, outer sep=0pt, right=1mm of img2] (img3) {\includegraphics[width=5.5cm]{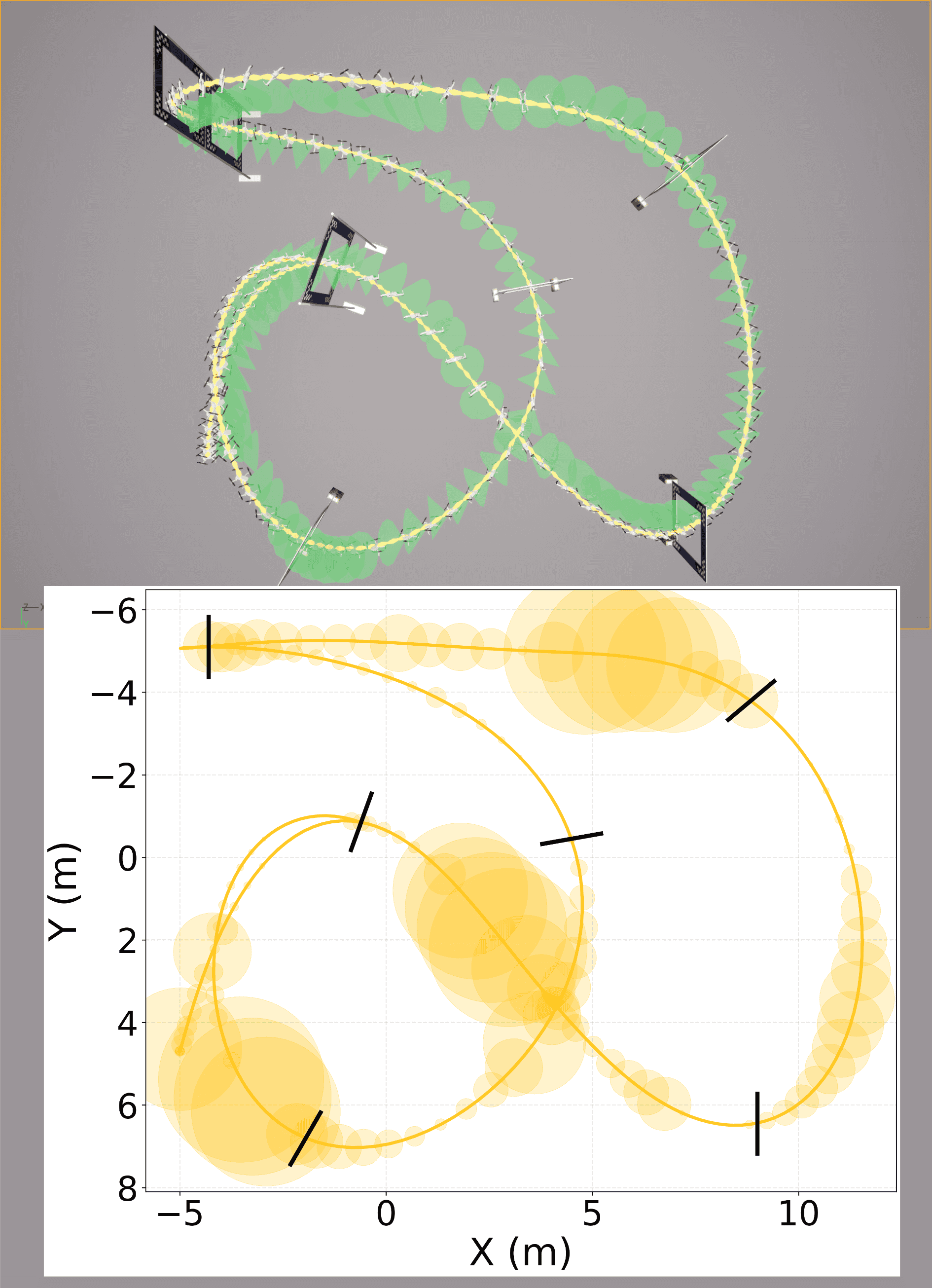}};

\node[below=1mm of img1](text1){(a) None (Pure Time Opt.) \tikz[baseline=-1.5ex] \node[circle, draw={rgb,1:red,0.247; green,0.369; blue,0.710}, fill={rgb,1:red,0.247; green,0.369; blue,0.710}, line width=0.8pt, inner sep=0pt, minimum size=6pt] {};};
\node[below=1mm of img2](text2){(b) \ac{LA} \tikz[baseline=-1.5ex] \node[regular polygon, regular polygon sides=3, rotate=0, draw={rgb,1:red,0.247; green,0.369; blue,0.710}, fill={rgb,1:red,0.247; green,0.369; blue,0.710}, line width=0.8pt, inner sep=0pt, minimum size=6pt] {};};
\node[below=1mm of img3](text3){(c) \ac{FOV} \tikz[baseline=-1.5ex] \node[regular polygon, regular polygon sides=4, rotate=0, draw={rgb,1:red,0.247; green,0.369; blue,0.710}, fill={rgb,1:red,0.247; green,0.369; blue,0.710}, line width=0.8pt, inner sep=0pt, minimum size=6pt] {};};

\node [inner sep=0pt, outer sep=0pt, below=1mm of text1] (img4) {\includegraphics[width=5.5cm]{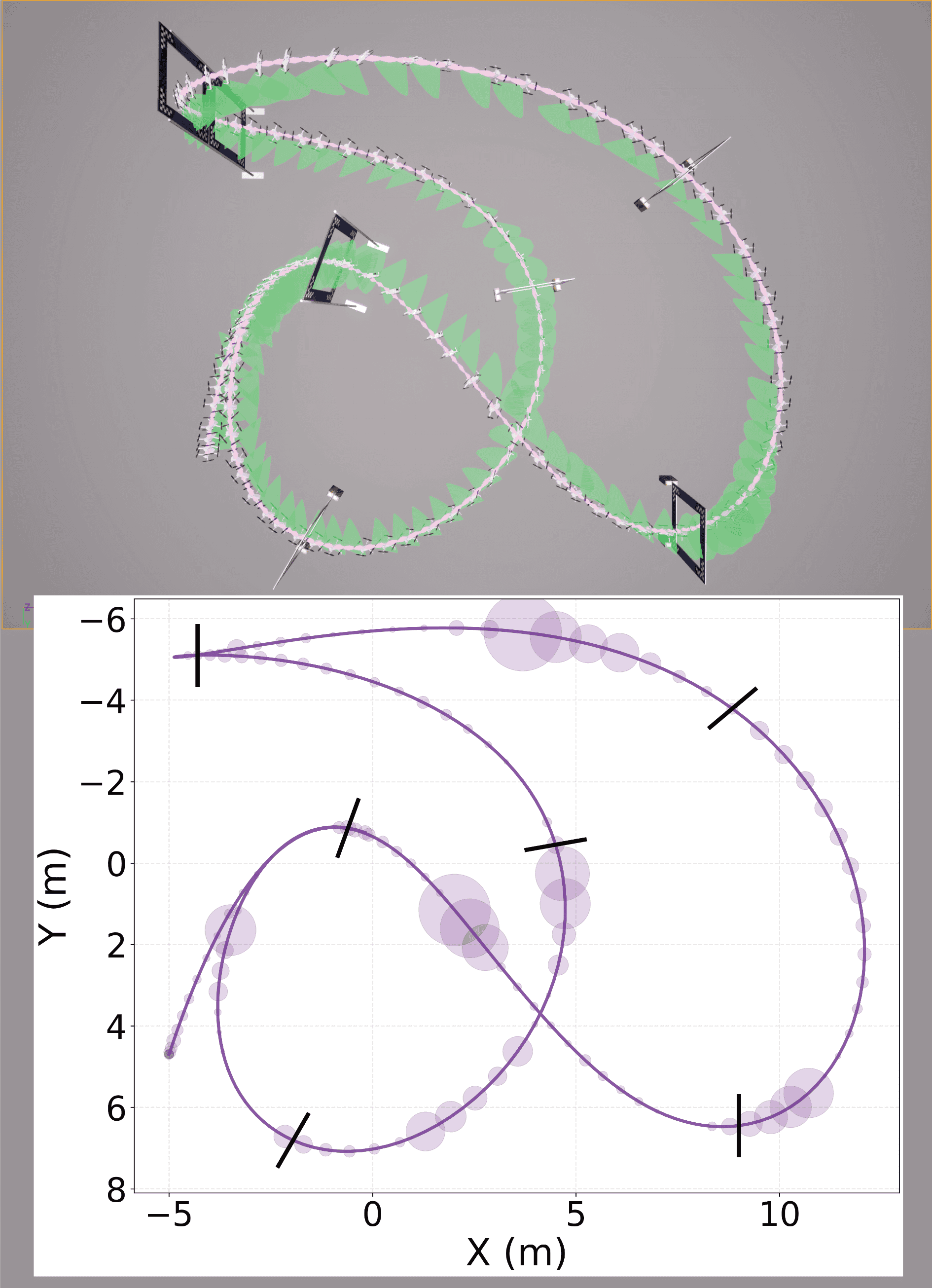}};
\node [inner sep=0pt, outer sep=0pt, below=1mm of text2] (img5) {\includegraphics[width=5.5cm]{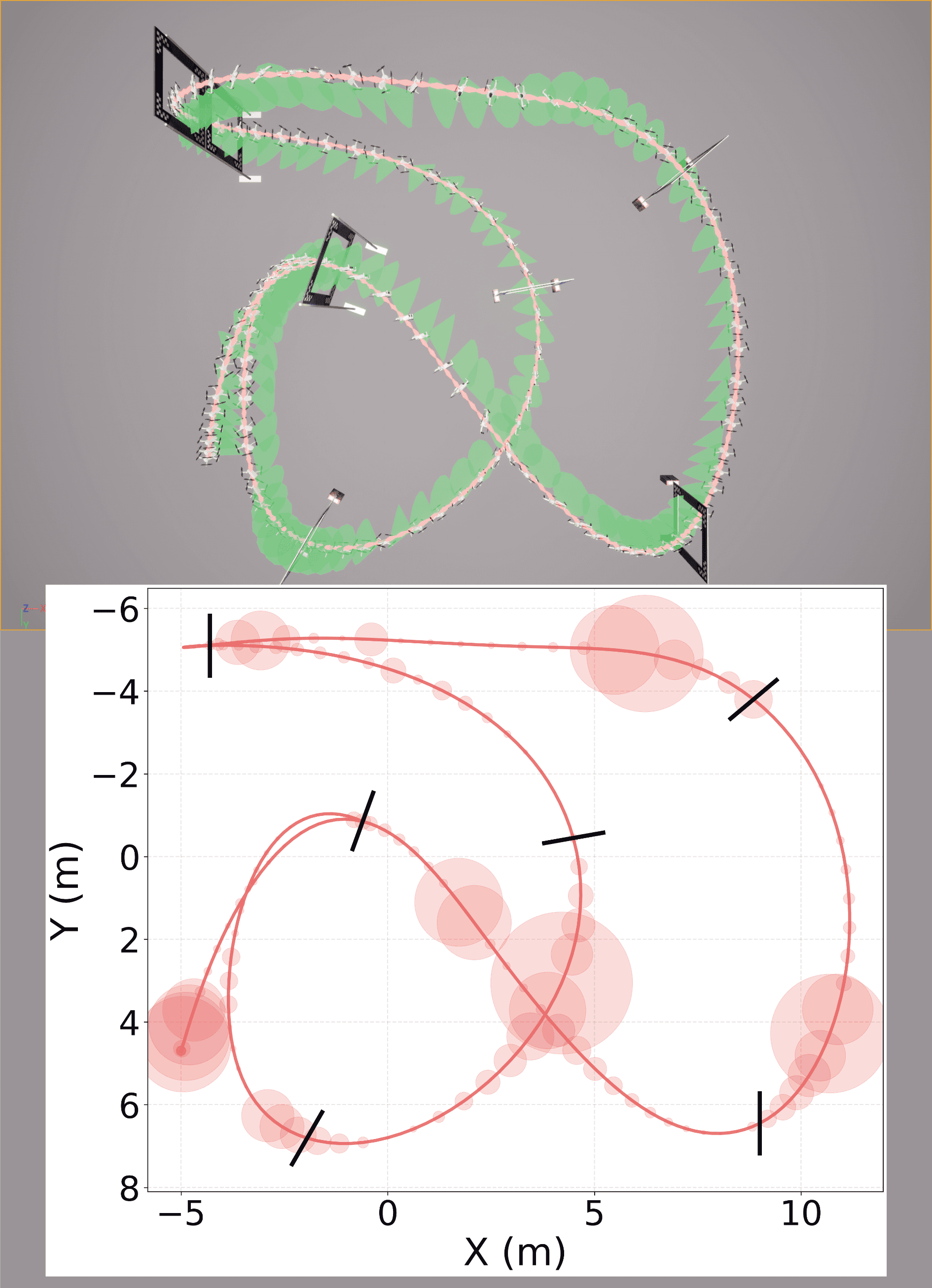}};
\node [inner sep=0pt, outer sep=0pt, below=1mm of text3] (img6) {\includegraphics[width=5.5cm]{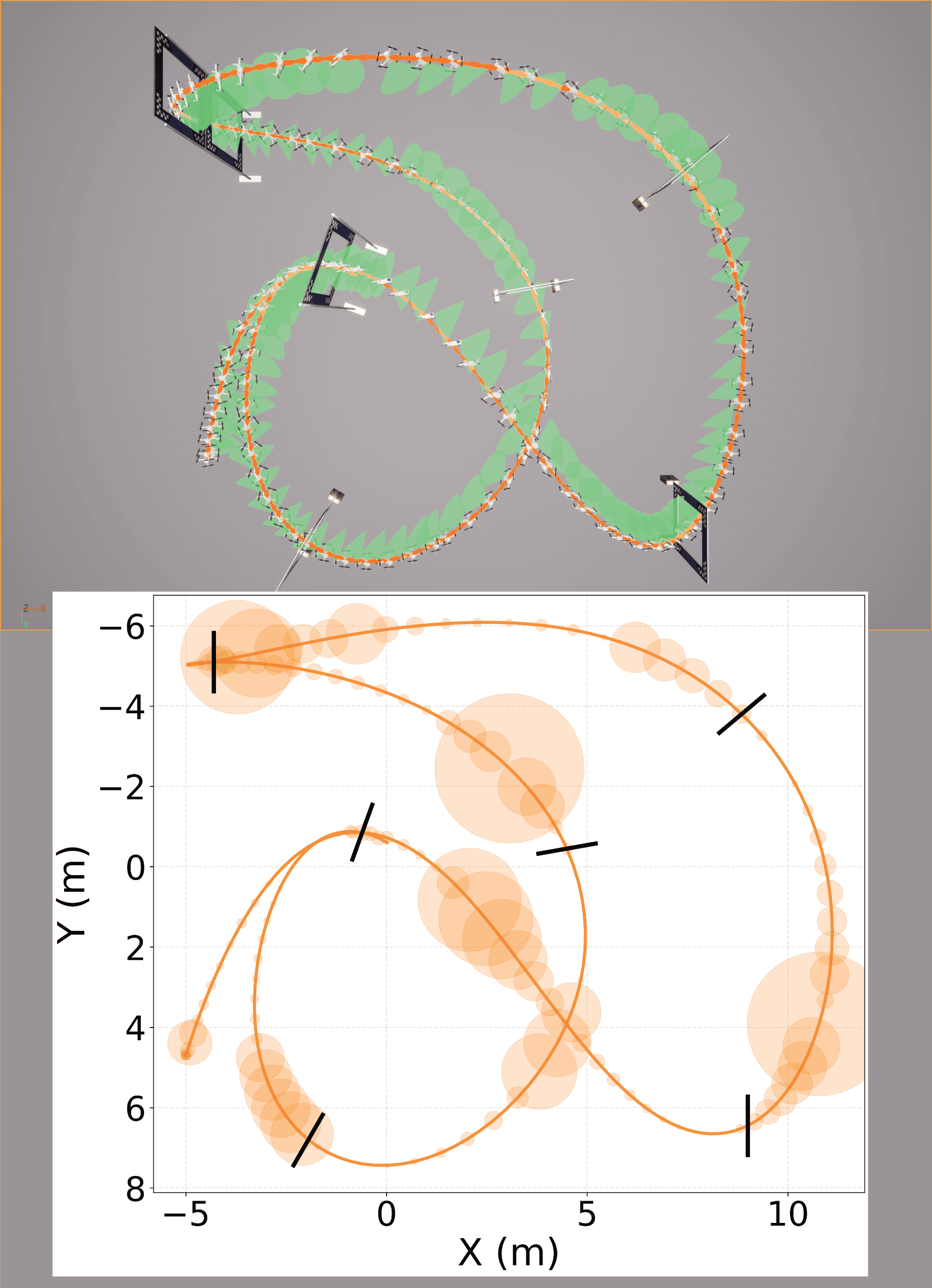}};

\node[below=1mm of img4](text4){(d) \ac{PUM} \tikz[baseline=-1.8ex] \node[regular polygon, regular polygon sides=5, rotate=0, draw={rgb,1:red,0.247; green,0.369; blue,0.710}, fill={rgb,1:red,0.247; green,0.369; blue,0.710}, line width=0.8pt, inner sep=0pt, minimum size=6pt] {};};
\node[below=1mm of img5](text5){(e) FOV-PUM   \tikz[baseline=-1.5ex] \node[regular polygon, regular polygon sides=6, rotate=0, draw={rgb,1:red,0.247; green,0.369; blue,0.710}, fill={rgb,1:red,0.247; green,0.369; blue,0.710}, line width=0.8pt, inner sep=0pt, minimum size=6pt] {};};
\node[below=1mm of img6](text6){(f) LA-PUM   \tikz[baseline=-1.5ex] \node[regular polygon, regular polygon sides=8, rotate=0, draw={rgb,1:red,0.247; green,0.369; blue,0.710}, fill={rgb,1:red,0.247; green,0.369; blue,0.710}, line width=0.8pt, inner sep=0pt, minimum size=6pt] {};};

\end{tikzpicture}
\caption{Impact of various perception objectives on time-optimal trajectories and the evolution of position uncertainty. The circles represent the uncertainty area at discrete timestamps. Results demonstrate that trajectories optimized with \ac{PUM} exhibit substantially lower position uncertainty than other methods.}
\label{fig:splits_with_perception_awareness}
\end{figure*}


\begin{figure}[!htbp]
    \centering
\tikzstyle{every node}=[font=\footnotesize]
\begin{tikzpicture}[>=stealth]
\node [inner sep=0pt, outer sep=0pt] (img1) at (0,0)
{\includegraphics[width=8.0cm]{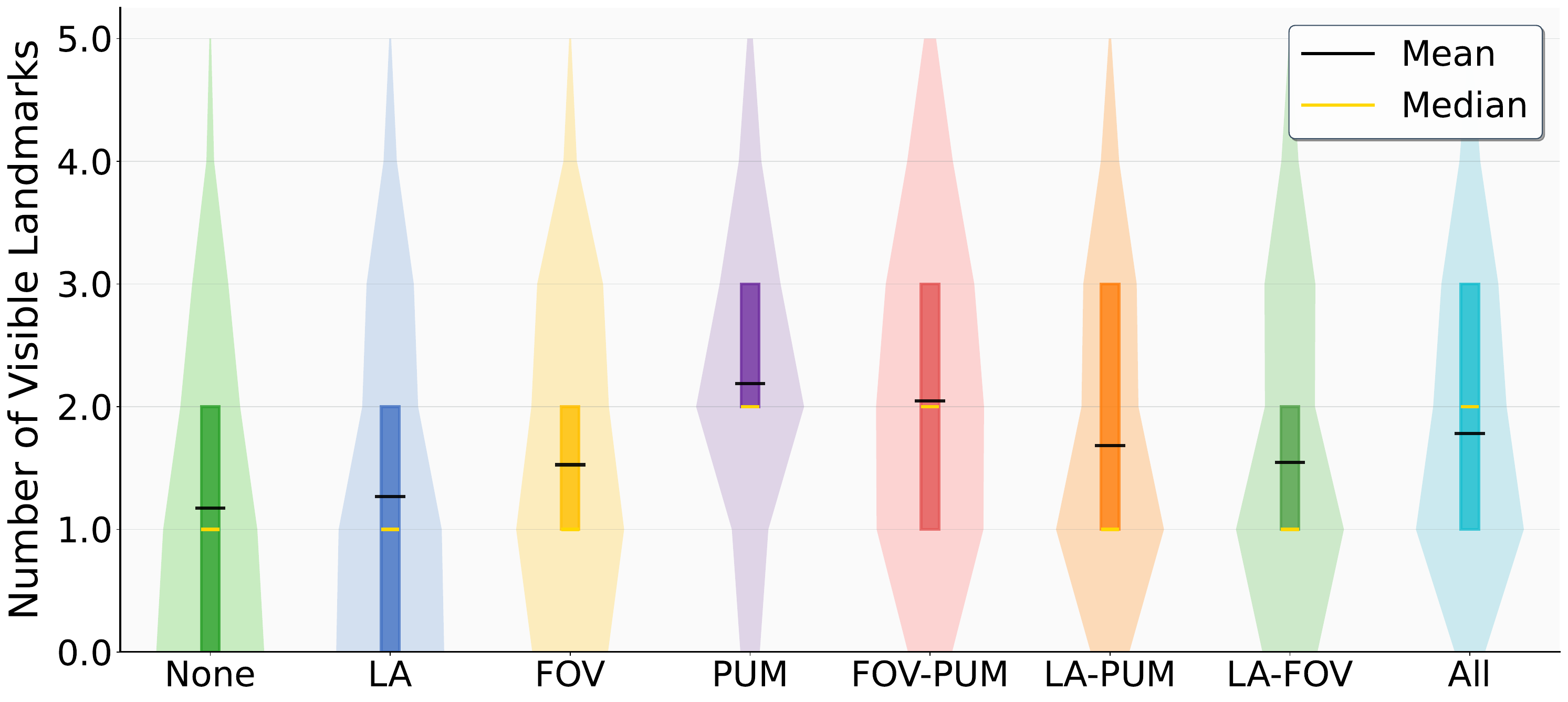}};
\end{tikzpicture}
\caption{Distributions of visible landmarks/gates on Split-S track, which reveals that perception-aware flights boost the visibility substantially.}
\label{fig:fig_visible_landmarks_distribution}
\end{figure}



\begin{figure}[!htbp]
    \centering
\tikzstyle{every node}=[font=\footnotesize]
\begin{tikzpicture}[>=stealth]
\node [inner sep=0pt, outer sep=0pt] (img1) at (0,0)
{\includegraphics[width=8.0cm]{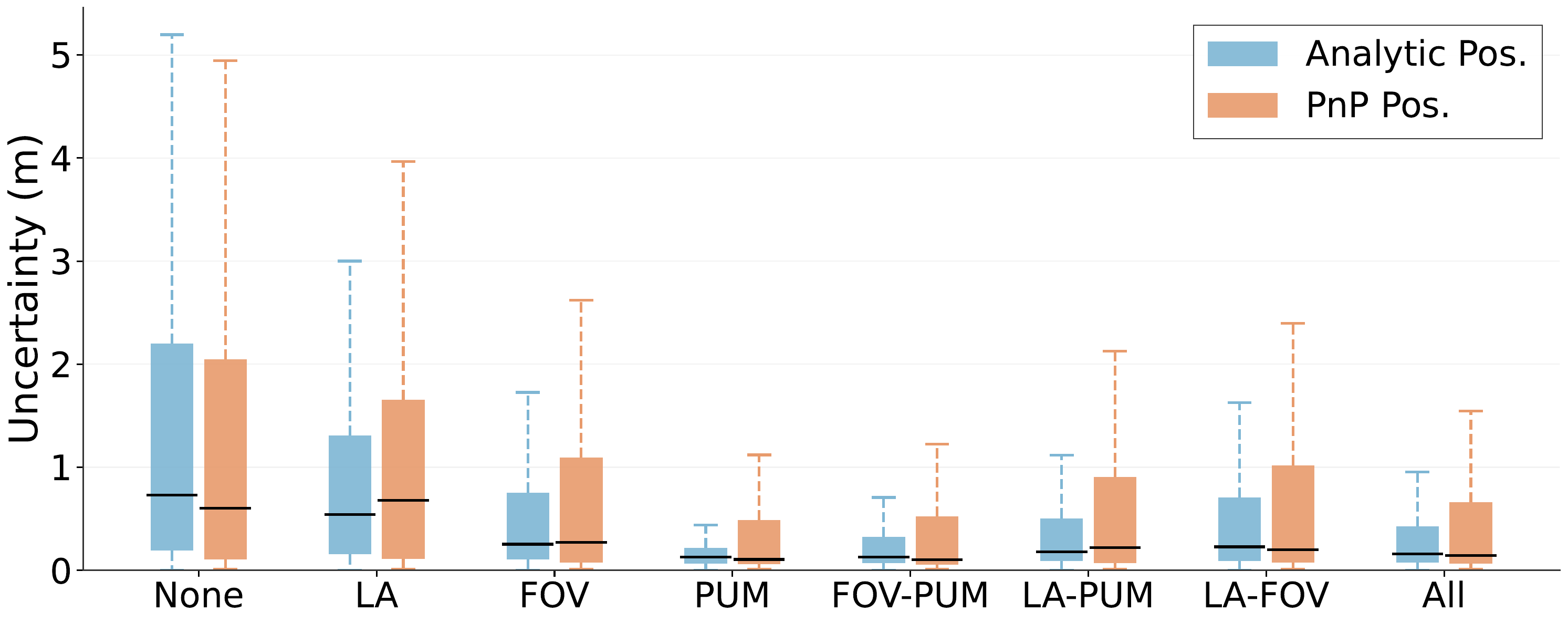}};
\end{tikzpicture}
\caption{Distributions of the derived analytical position uncertainty and the actual position uncertainty of the \ac{PnP} method. To get the \ac{PnP} uncertainty, we perturb the observed feature coordinates, re-evaluate the \ac{PnP} problem 60 times, and compute the mean and covariance of the resulting pose estimates.}
\label{fig:fig_analytic_numerical_uncertainty_distribution}
\end{figure}


Quantitative visual quality comparisons are presented in Fig.~\ref{fig:fig_visible_landmarks_distribution} and Fig.~\ref{fig:fig_analytic_numerical_uncertainty_distribution}. The results suggest that the \ac{PUM} objective expands the distribution of trajectory segments where more than two gates are visible. Specifically, it achieves a median of two visible gates, whereas both \ac{LA} and \ac{FOV} have only one. Due to this enhanced landmark visibility, \ac{PUM} achieves the lowest uncertainty. As shown in Fig.~\ref{fig:fig_analytic_numerical_uncertainty_distribution}, the pure \ac{PUM} case reaches an average analytical position uncertainty of 0.2418 m, while the FOV-PUM case achieves 0.2758 m. The corresponding \ac{PnP} position uncertainty decreases accordingly, reaching a minimum average value of 0.4271 m and the smallest standard deviation of 0.6520 m. These results confirm that optimizing the derived position uncertainty is effective in reducing the actual numerical instability of the position estimates solved by the \ac{PnP} method. This justifies the proposed metric as a valid measure to the true position uncertainty.

Remarkably, using the \ac{PUM} objective also proves to be the most effective strategy for preserving optimal timing. As shown in Tab.~\ref{tab:benchmarking}, employing \ac{PUM} yields the smallest increase in trajectory duration compared to the pure time-optimal baselines. Specifically, in the \ac{TOWP} case, it results in a time increase of only 0.72 s, while for \ac{TOGT}, the increase is a mere 0.87 s. This indicates that \ac{PUM} avoids excessive deviations from the time-optimal maneuver, executing only the necessary orientation changes to enhance visual quality while preserving overall agility. In detail, one distinct feature of the \ac{PUM} strategy compared to both \ac{LA} and \ac{FOV} is that it allows the quadrotor to fly backward, matching the behavior of the pure time-optimal trajectory, provided it can observe more gates by looking behind itself. While this strategy is counterintuitive for humans, it is logically sound from a robotic vision perspective; the performance of state estimation is governed by the number and geometry of all observed landmarks rather than a specific one enforced by the task itself.

\subsection{Effectiveness of FOV Constraints}

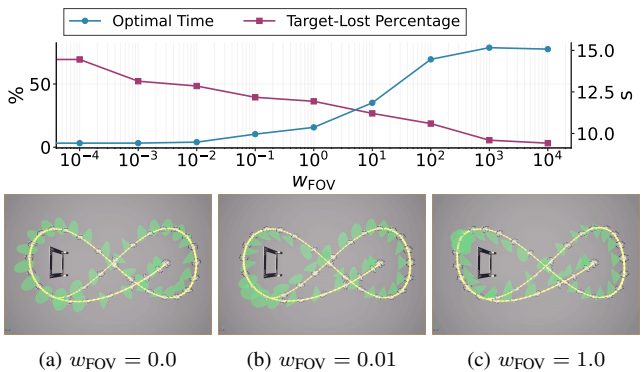
\begin{figure}[!htbp]
    \centering
\tikzstyle{every node}=[font=\footnotesize]
\begin{tikzpicture}[>=stealth]
\node [inner sep=0pt, outer sep=0pt, xshift=4.2cm] (img1) at (0,0)
{\includegraphics[width=8.45cm]{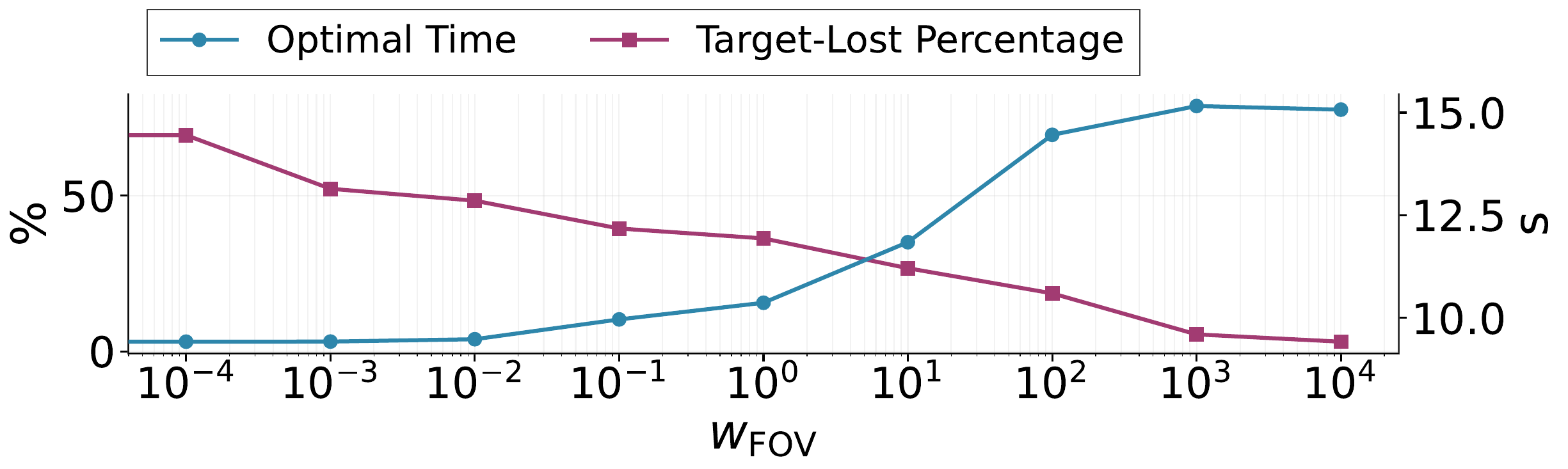}};

\node [inner sep=0pt, outer sep=0pt, below=0mm of img1.south west, anchor=north west] (img2)
{\includegraphics[width=2.75cm]{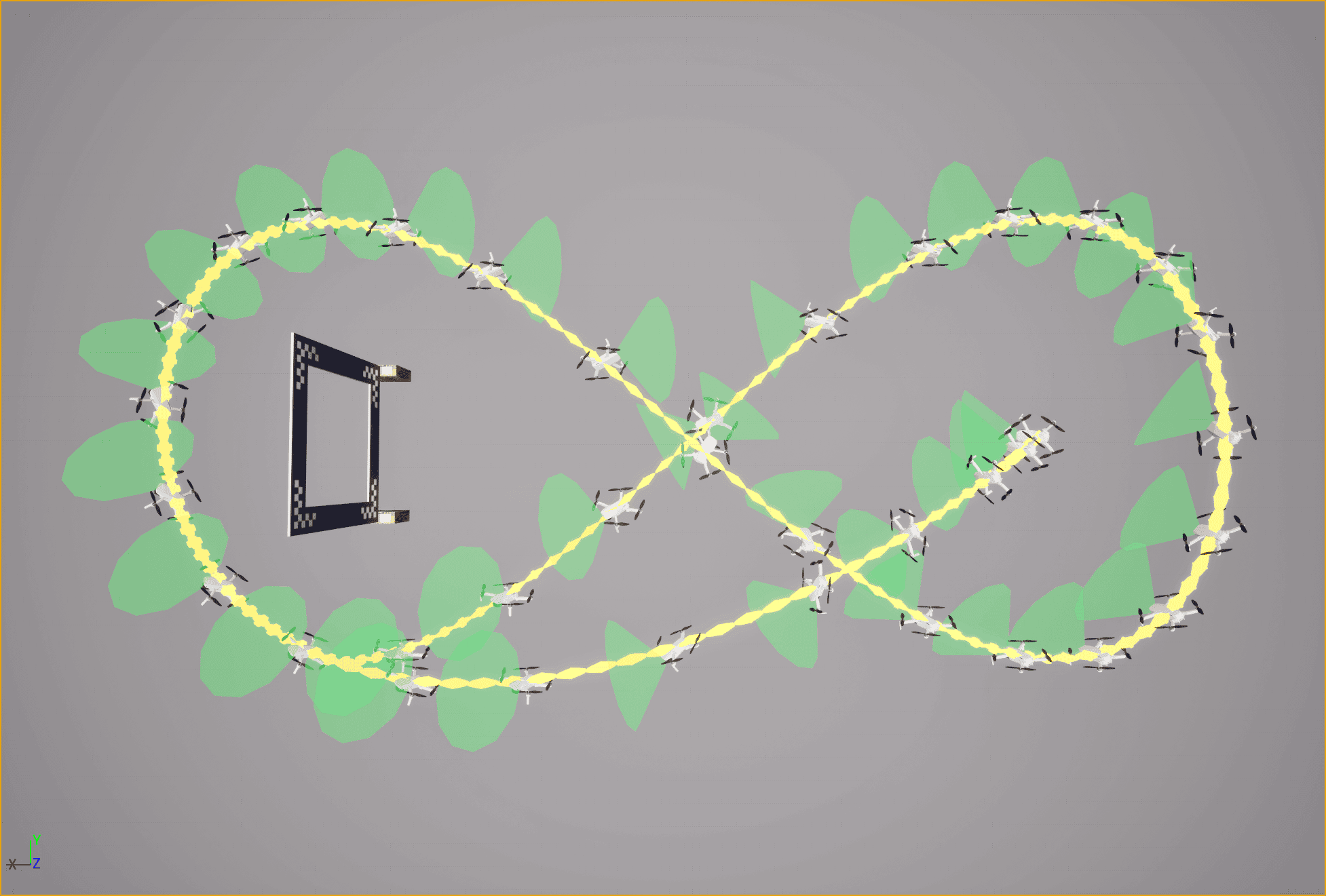}};
\node [inner sep=0pt, outer sep=0pt, right=1mm of img2] (img3) {\includegraphics[width=2.75cm]{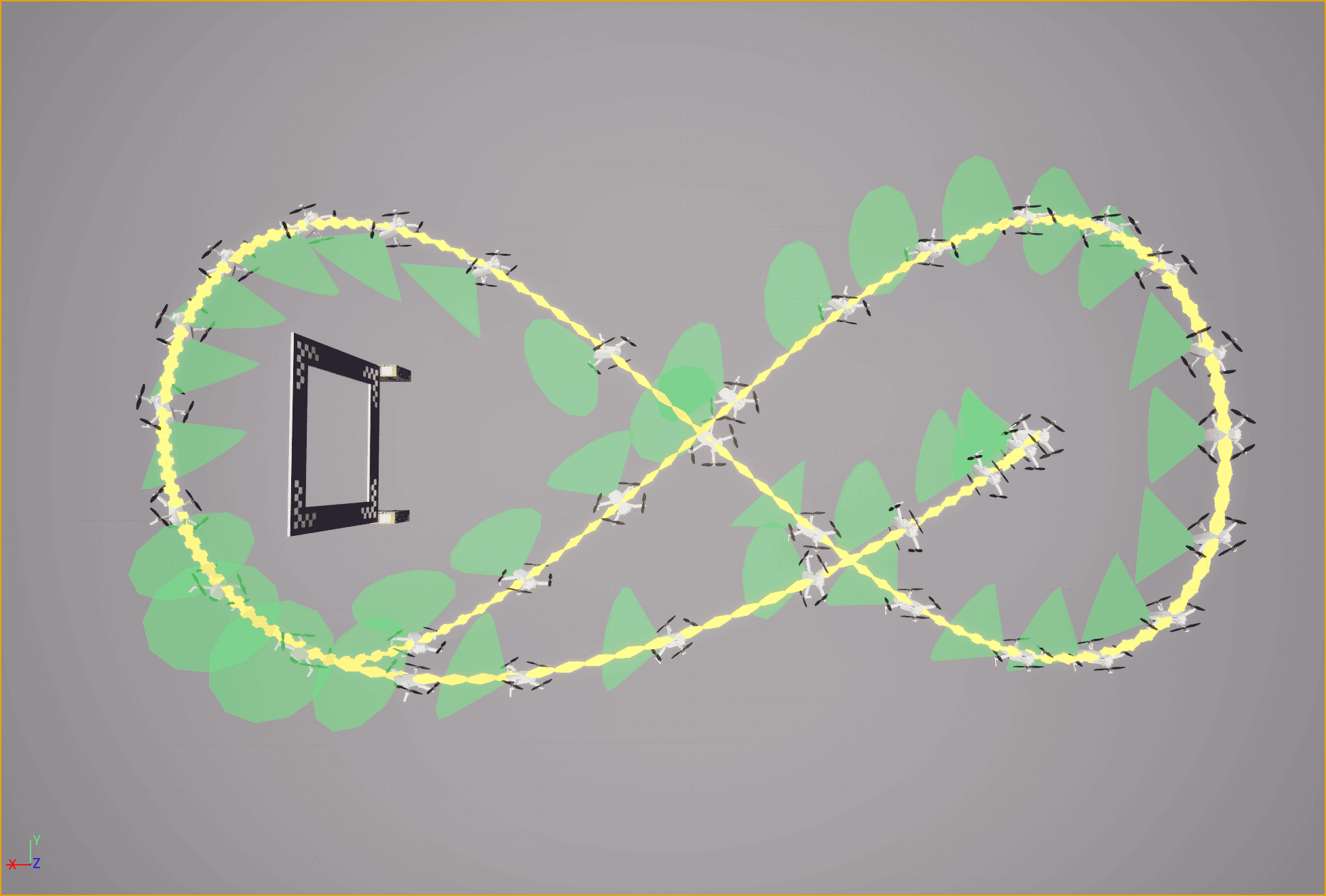}};
\node [inner sep=0pt, outer sep=0pt, right=1mm of img3] (img4) {\includegraphics[width=2.75cm]{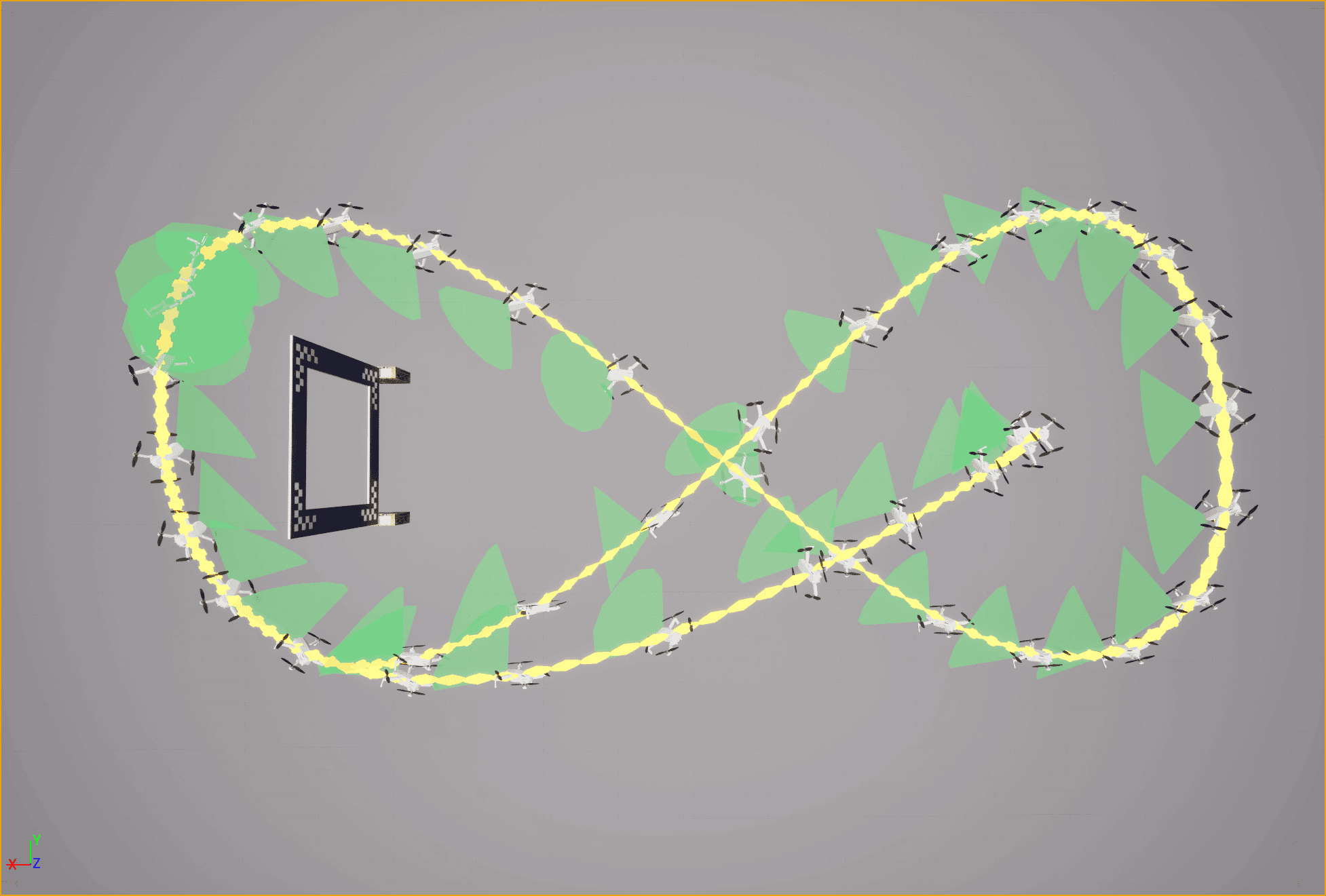}};

\node[below=1mm of img2](text1){(a) $w_{\text{FOV}}=0.0$};
\node[below=1mm of img3](text2){(b) $w_{\text{FOV}}=0.01$};
\node[below=1mm of img4](text3){(c) $w_{\text{FOV}}=1.0$};

\end{tikzpicture}
\caption{Impact of the weight for the FOV constraints on the resulting flight behavior. As the penalty weight increases, the drone prioritizes visibility to the target gate more and eventually reaches a near-zero target-loss percentage. But it causes a longer time to complete the task.}
\label{fig:figure8_fov}
\end{figure}

We are interested in how the quadrotor behavior changes as the strictness of the \ac{FOV} constraints increases. In this experiment, the quadrotor is tasked with navigating a figure-8 track defined by six waypoints (see Fig.~\ref{fig:figure8_tii_dive}(a) but with all gates removed). One gate is positioned at $(8, 2)$ m to serve as the target object, and the quadrotor will attempt to maintain visibility of it throughout the flight.

Fig. \ref{fig:figure8_fov} demonstrates that increasing the weight $w_{\text{FOV}}$ leads to a monotonic decrease in the percentage of time the target is outside the \ac{FOV}. However, this comes with a consequence of a growing trajectory duration. The duration stops growing when $w_{\text{FOV}}$ reaching a point at which the \ac{FOV} constraints are always satisfied. We find that a small value of $w_{\text{FOV}}$, such as 0.01, can cause huge differences in the quadrotor's orienting behavior. We see the quadrotor proactively force its camera toward the gate to keep it visible. As the weighting factor increases to 1, the trajectory undergoes a salient structural change. This shift indicates that modulating the quadrotor's orientation alone is no longer sufficient to maintain target visibility; instead, a coordinated adjustment of both translation and rotation is required to satisfy the visibility constraint.

\subsection{Closed-Loop Evaluation}

\begin{figure}[!htbp]
    \centering
\tikzstyle{every node}=[font=\footnotesize]
\begin{tikzpicture}[>=stealth]
\node [inner sep=0pt, outer sep=0pt] (img1) at (0,0)
{\includegraphics[width=6.0cm]{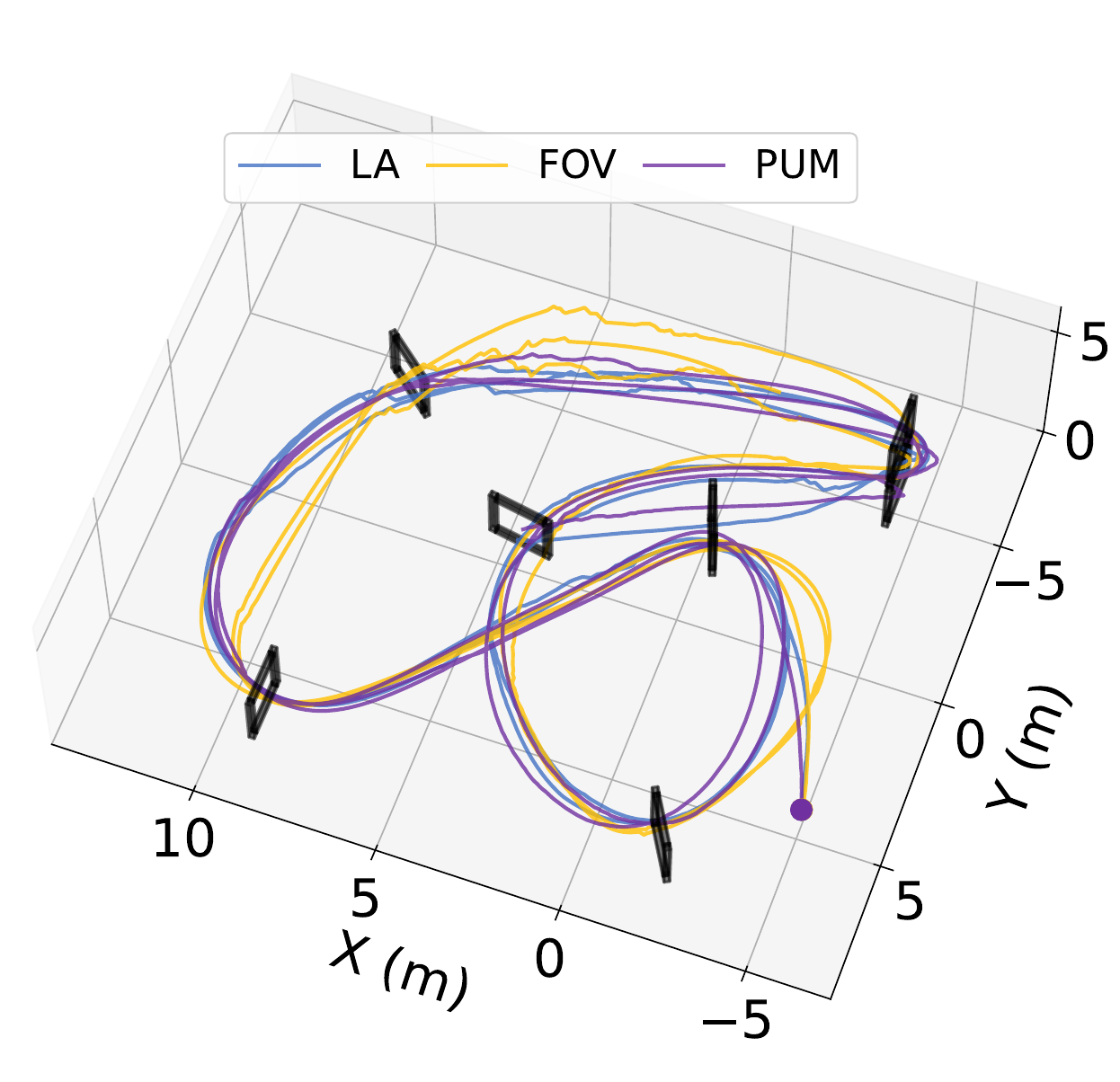}};
\end{tikzpicture}
\caption{Closed-loop execution of planned trajectories with LA, FOV, and PUM objectives, tracked using the proposed \ac{MPCTC}. We observe that the \ac{PUM} (purple) yields the smoothest tracking trajectory.}
\label{fig:closed_loop_2d}
\vspace{-0.2cm}
\end{figure}

To directly verify the effectiveness of the proposed perception objectives in fully onboard performance, we incorporate a vision-based localization module into the simulation for closed-loop evaluation. We adopt the framework from \cite{foehn2022alphapilot} for the implementation of the localization module, which utilizes an \ac{EKF} as the core fusion algorithm. One key difference is that we directly fuse simulated noisy \ac{IMU} readings with the pose estimates from the \ac{PnP} rather than using external visual-inertial odometry. This approach allows us to isolate the influence of other features in the scene, which is not the focus of this paper. We employ RotorS for the quadrotor dynamics simulation \cite{Furrer2016}, and the \ac{MPCTC} is the default tracking controller.

\begin{table}[!htbp]
\centering
\caption{Close-loop execution benchmarking.}\label{tab:closed_loop_benchmarking}
\resizebox{\linewidth}{!}{
    \begin{tabular}{lccc}
        \toprule
        Methods & Pos. Uncertainty (m) & Track. Error (m) & Success Rate \\
        \midrule
        \backgroundcolor
        \tikz[baseline=-0.6ex] \node[circle, draw={rgb,1:red,0.247; green,0.369; blue,0.710}, fill={rgb,1:red,0.247; green,0.369; blue,0.710}, line width=0.8pt, inner sep=0pt, minimum size=6pt] {}; TOWP-Ours & 0.3590 & 0.4192 & 55\% \\
        \tikz[baseline=-0.6ex] \node[regular polygon, regular polygon sides=3, rotate=0, draw={rgb,1:red,0.247; green,0.369; blue,0.710}, fill={rgb,1:red,0.247; green,0.369; blue,0.710}, line width=0.8pt, inner sep=0pt, minimum size=6pt] {}; LA  & 0.3541 & 0.4612 & 67\% \\
        \backgroundcolor
        \tikz[baseline=-0.6ex] \node[regular polygon, regular polygon sides=4, rotate=0, draw={rgb,1:red,0.247; green,0.369; blue,0.710}, fill={rgb,1:red,0.247; green,0.369; blue,0.710}, line width=0.8pt, inner sep=0pt, minimum size=6pt] {}; FOV & 0.4634 & 0.3577 & 71\%\\
        \tikz[baseline=-0.6ex] \node[regular polygon, regular polygon sides=5, rotate=0, draw={rgb,1:red,0.247; green,0.369; blue,0.710}, fill={rgb,1:red,0.247; green,0.369; blue,0.710}, line width=0.8pt, inner sep=0pt, minimum size=6pt] {}; PUM  & 0.2228 & 0.3239 & 100\%\\
        \bottomrule
    \end{tabular}
}
\end{table}

The resulting trajectories are displayed in Fig.~\ref{fig:closed_loop_2d}. Qualitatively, we observe that the PUM objective yields a significantly smoother tracking trajectory compared to those of LA and FOV. Tab.~\ref{tab:closed_loop_benchmarking} provides quantitative results regarding the position uncertainty of the planned trajectories, the tracking error in RMSE, and the success rate across 20 closed-loop flight tests. We notice that the trajectory optimized with PUM yields the smallest position uncertainty of 0.2228 m, resulting in the lowest tracking error of 0.3239 m. Furthermore, it is the only perception objective that achieves a 100\% success rate. These results demonstrate that incorporating PUM into trajectory optimization can significantly boost closed-loop flight performance and robustness. 


\subsection{Real-World Validation}

\begin{figure}[!htbp]
    \centering
\tikzstyle{every node}=[font=\footnotesize]
\begin{tikzpicture}[>=stealth]
\node [inner sep=0pt, outer sep=0pt] (img1) at (0,0)
{\includegraphics[width=8.0cm]{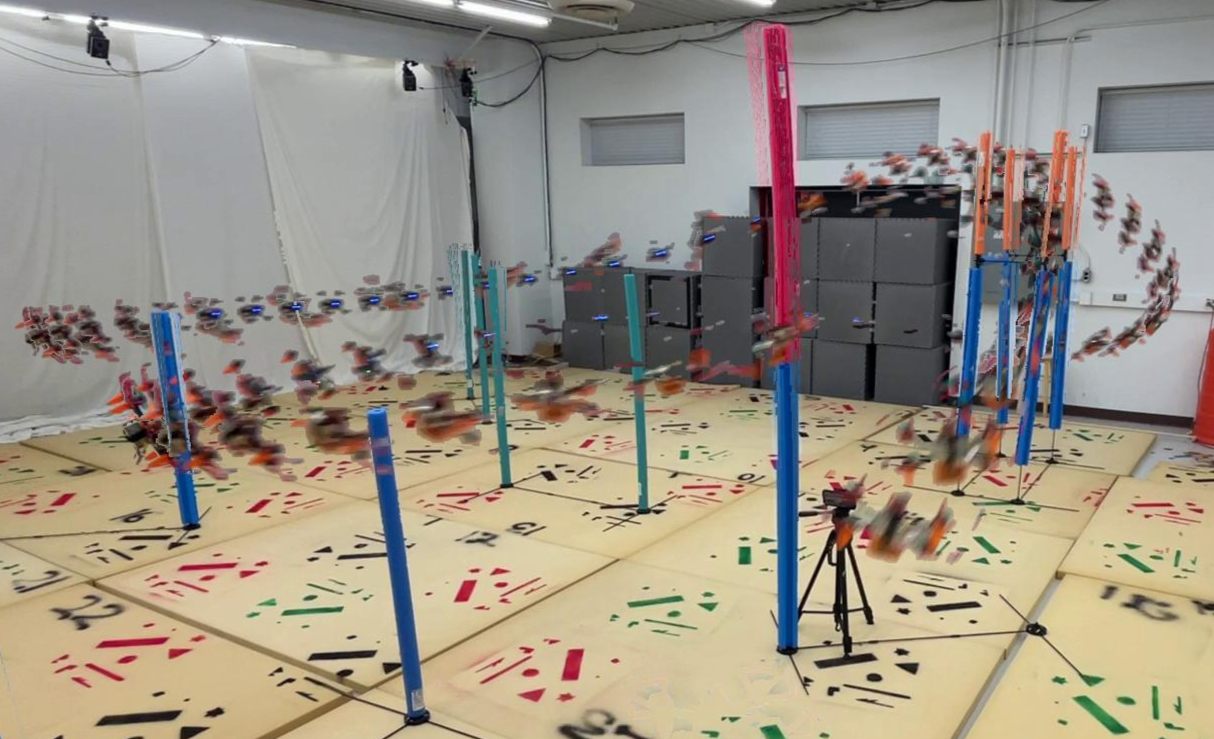}};
\node [inner sep=0pt, outer sep=0pt, below=1mm of img1] (img2)
{\includegraphics[width=8.0cm]{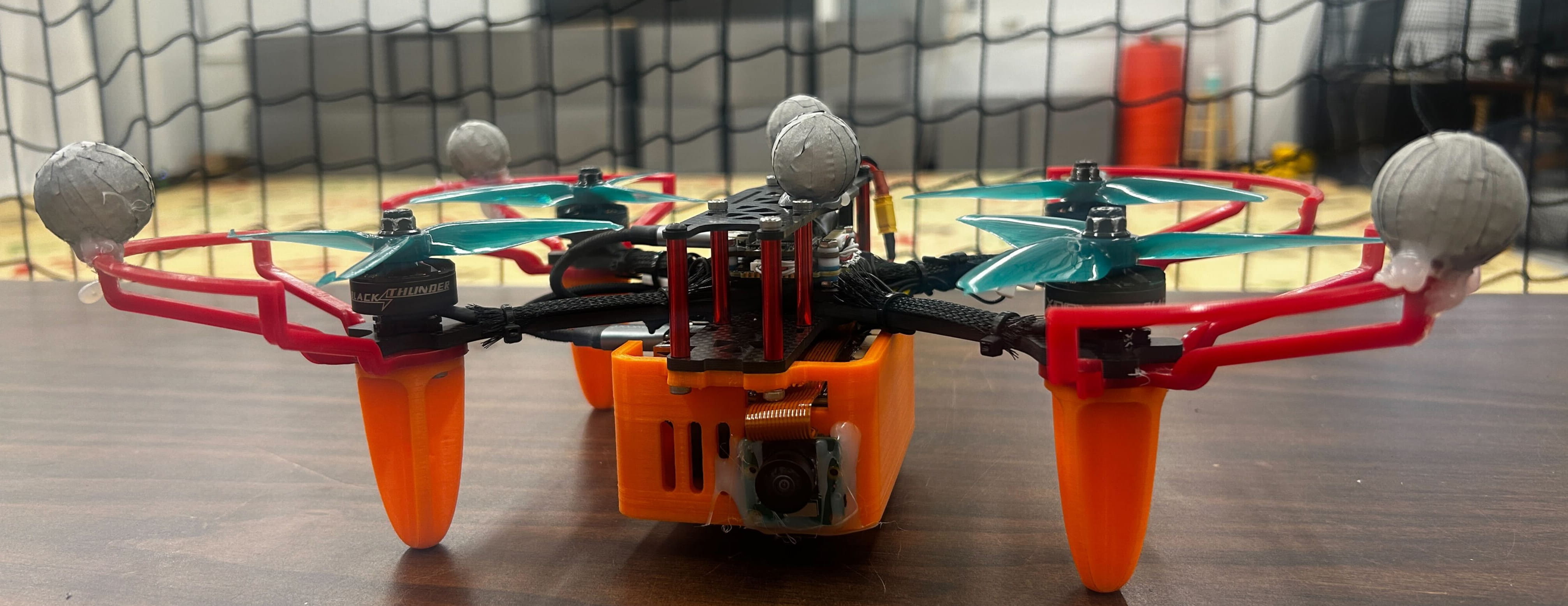}};
\end{tikzpicture}
\caption{Indoor Split-S track for experiment validation. The quadrotor platform is equipped with an NVIDIA Jetson Orin NX 16GB onboard computer and an IMX219 fisheye camera for image capture. Betaflight firmware is employed for low-level body-rate control. }
\label{fig:fig_drone}
\vspace{-0.1cm}
\end{figure}

We aim to validate the proposed planning and control system in real-world time-critical tasks. Please find our hardware specifications in Fig. \ref{fig:fig_drone} and the associated quadrotor parameters in Tab.~\ref{tab:quad_config}. Note that in the following experiments, we obtain the online state estimation from the motion capture system and perform the uncertainty evaluation offline based on the executed trajectory. This choice does not weaken our conclusion because the motion capture system provides a high-fidelity ground truth that is necessary to validate our theoretical uncertainty models. By using this setup, we decouple the performance of the planner from the potential failures of a specific online vision pipeline. This allows us to demonstrate that the planned trajectories inherently lead to better visual quality and lower position uncertainty when they are tracked with high accuracy.

\subsubsection{Controller Evaluation: \ac{MPC} vs. \ac{MPCTC}}

We begin by comparing the performance of \ac{MPC} and the proposed \ac{MPCTC} in tracking purely time-optimal trajectories shown in Fig~\ref{fig:fig_tracking}. Note that in this classical cascaded planning and control framework, it is common practice to limit the maximum thrust to between 70\% and 90\% of its actual capacity during the planning stage. This ensures sufficient actuation margin in the tracking stage to account for external disturbances, unexpected tracking errors, and model mismatch. 

\begin{figure}[!htbp]
    \centering
\tikzstyle{every node}=[font=\footnotesize]
\begin{tikzpicture}[>=stealth]
\node [inner sep=0pt, outer sep=0pt] (img1) at (0,0)
{\includegraphics[width=8.0cm]{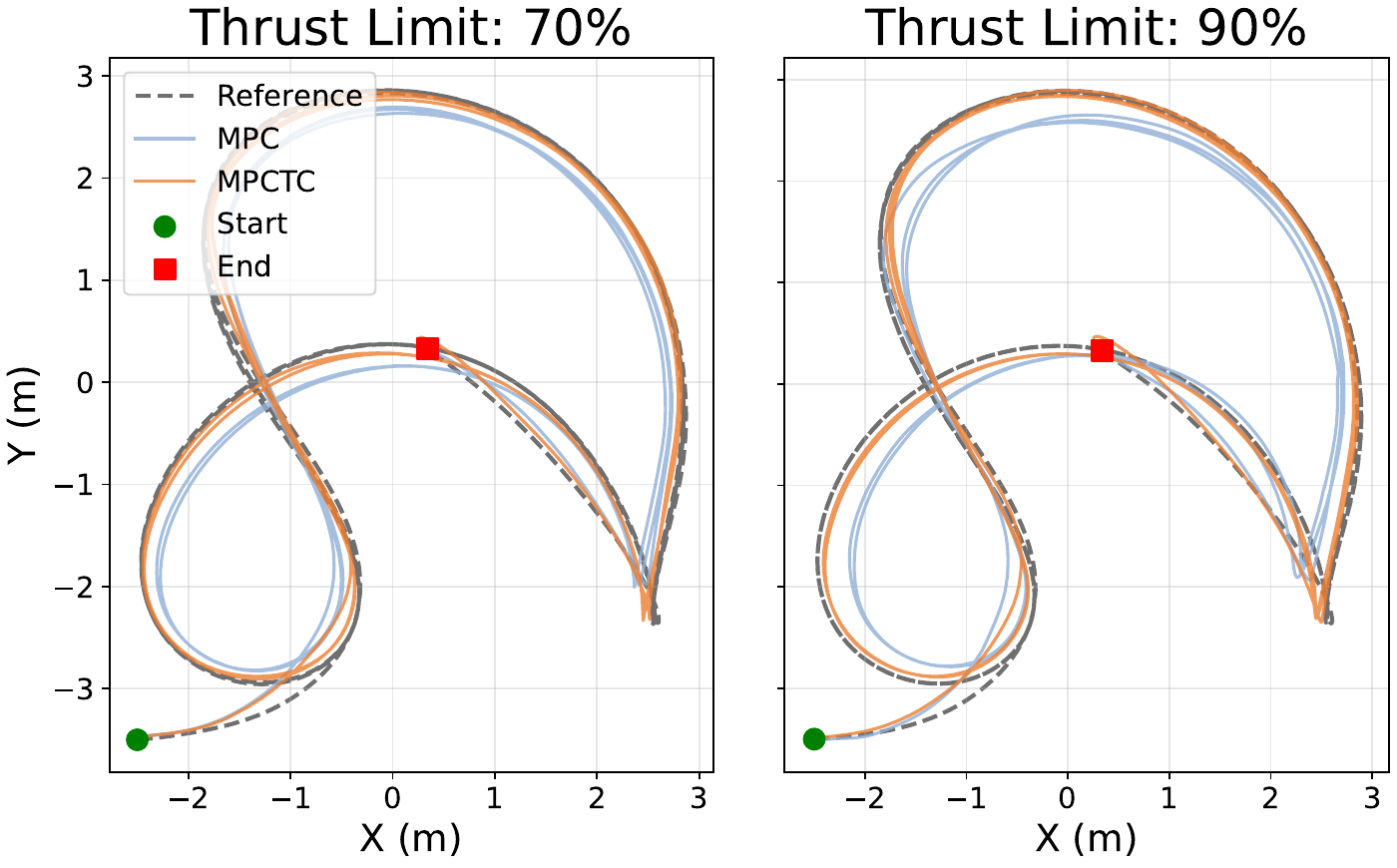}};
\node [inner sep=0pt, outer sep=0pt, below=1mm of img1] (img2)
{\includegraphics[width=8.0cm]{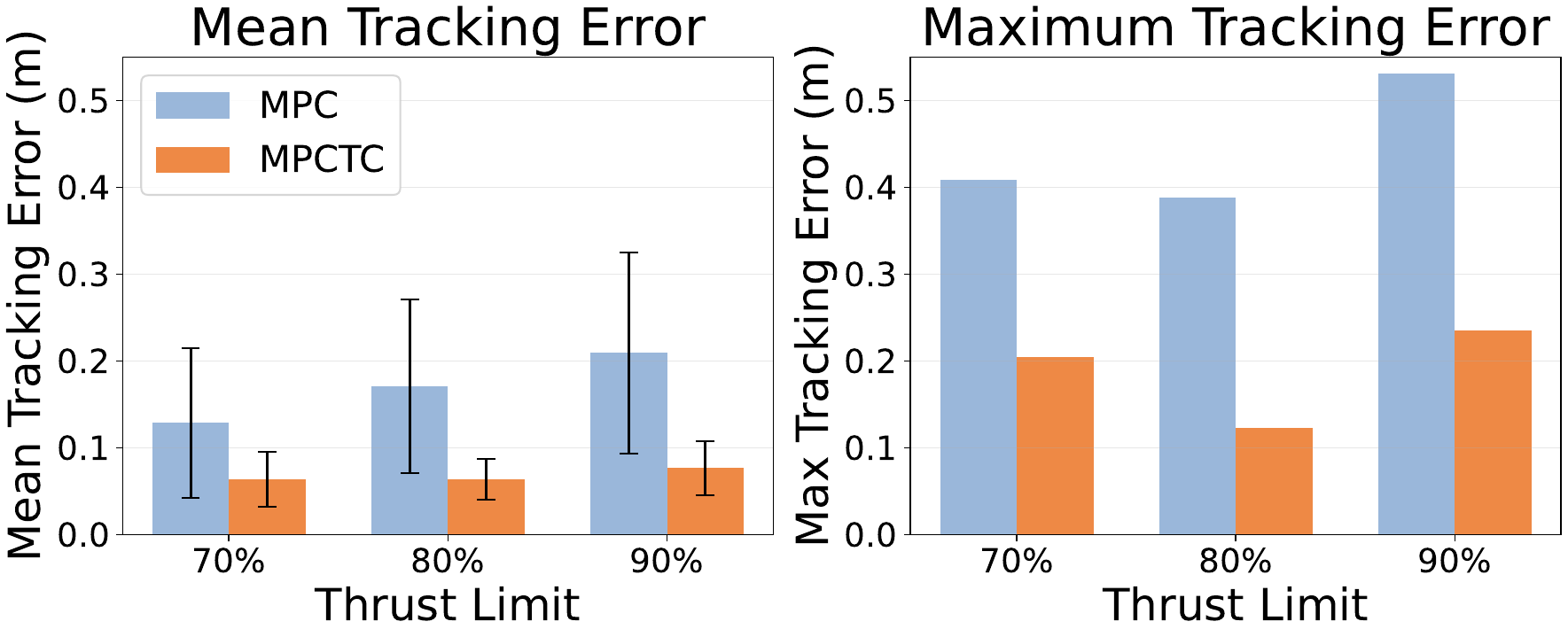}};
\end{tikzpicture}
\caption{Tracking trajectories of standard MPC versus our MPCTC against reference time-optimal trajectories planned with (a) 70\% and (b) 100\% thrust limit. The standard MPC performance degrades significantly when insufficient margin is reserved to compensate for tracking errors or disturbances}
\label{fig:fig_tracking}
\end{figure}

From the results, it is evident that \ac{MPCTC} outperforms \ac{MPC} in terms of tracking accuracy under all thrust limit configurations. The executed trajectory aligns closely with the reference, regardless of the thrust limit setups, maintaining a tracking error of 0.07 m on average with a maximum below 0.23 m even at flight speeds up to 9.8 m/s. In contrast, \ac{MPC} behaves more inconsistently, exhibiting significantly higher tracking error as the thrust limit increases.

Furthermore, we found that the ``corner-cutting'' behavior that frequently affects \ac{MPC} in high-speed flight is overcome by this new formulation. This is an important feature particularly in perception-aware flight, as you want to precisely execute the planned trajectory instead of some improvisation on the fly that breaks the coordination between the translation and rotation to meet certain perception needs. Also, this feature is critical for the \ac{TOGT} flight. When the planned trajectory passes near the edges of gates, the behavior of corner cutting will put the quadrotor in a severe risk of missing the gate or even colliding with the gate. To conclude, \ac{MPCTC} is a desirable tracking controller for perception-aware, time-critical missions due to its high adherence to the given reference trajectory.

\begin{figure}[!htbp]
    \centering
\tikzstyle{every node}=[font=\footnotesize]
\begin{tikzpicture}[>=stealth]
\node [inner sep=0pt, outer sep=0pt] (img1) at (0,0)
{\includegraphics[width=8.0cm]{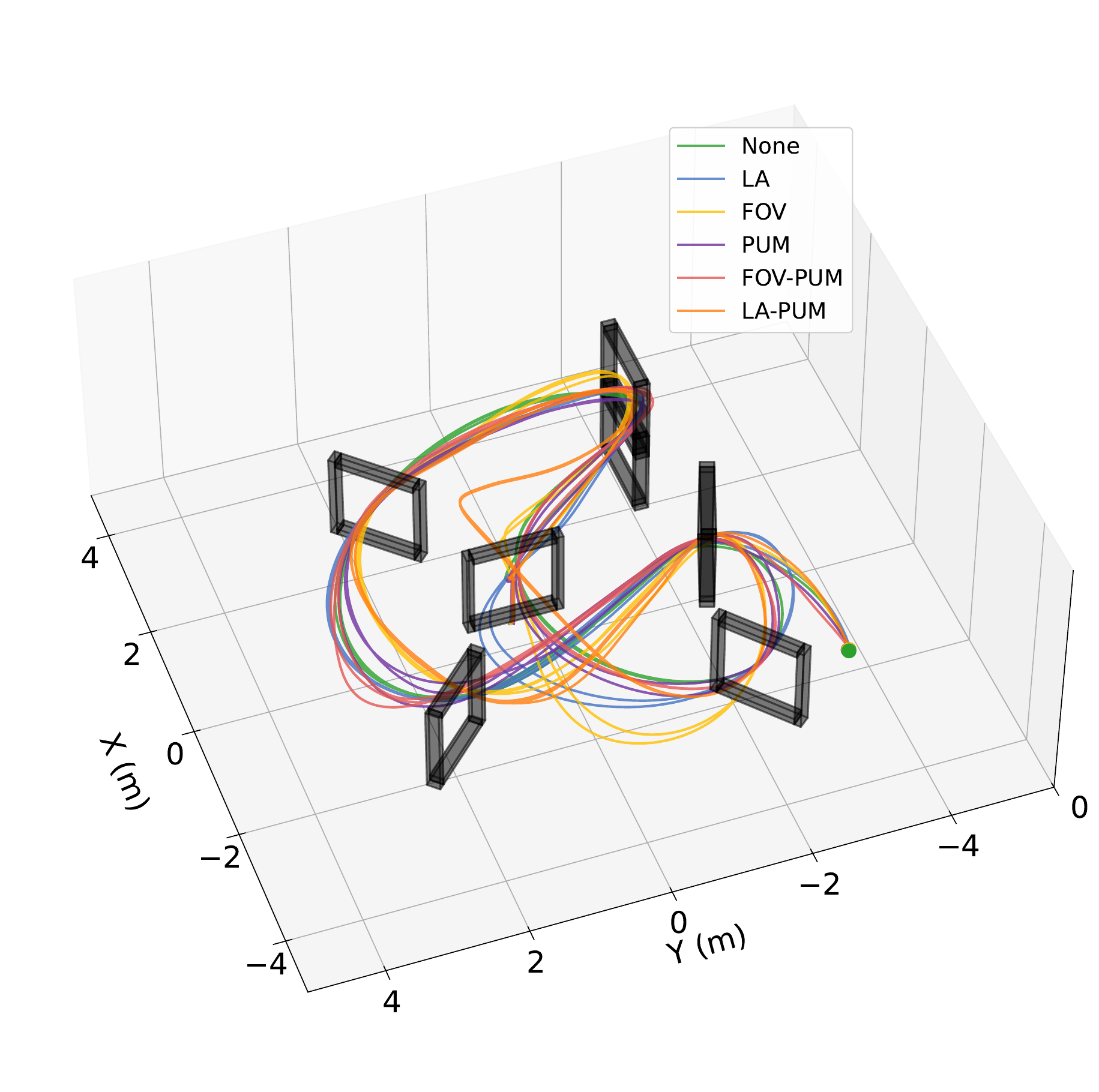}};
\end{tikzpicture}
\caption{Executed trajectories with various perception objectives on a real-world Split-S track. }
\label{fig:exp_trajs}
\vspace{-0.2cm}
\end{figure}


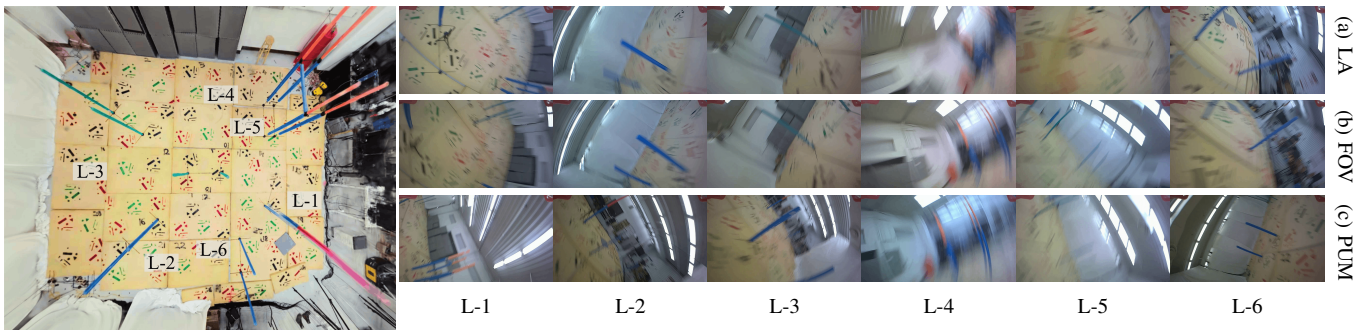
\begin{figure*}[!htbp]
    \centering
\tikzstyle{every node}=[font=\footnotesize]
\begin{tikzpicture}[>=stealth]

\node [inner sep=0pt, outer sep=0pt] (img0) at (0,0)
{\includegraphics[width=5.2cm]{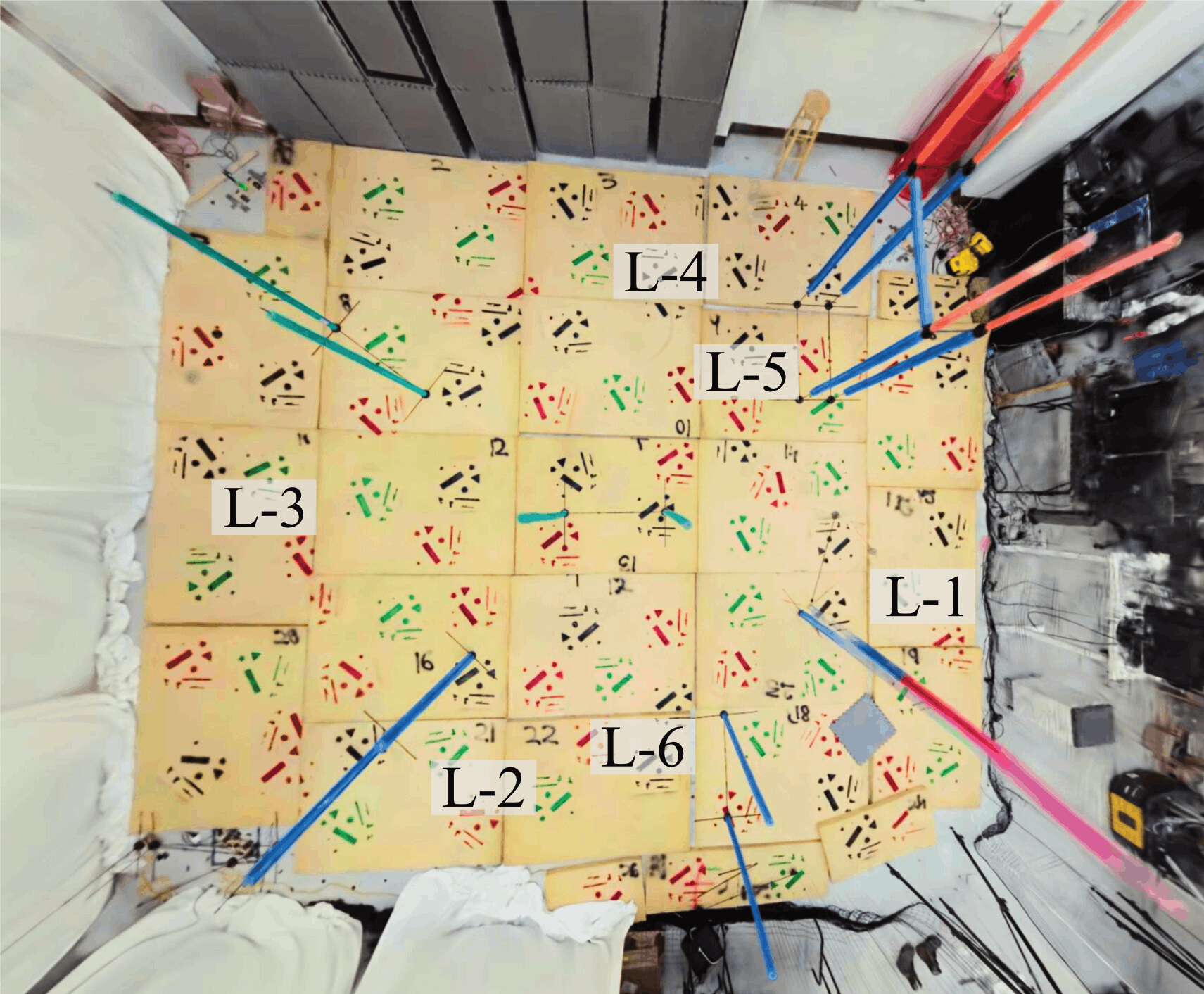}};

\node [inner sep=0pt, outer sep=0pt, right=0.5mm of img0.north east, anchor=north west] (img1)
{\includegraphics[width=2.05cm]{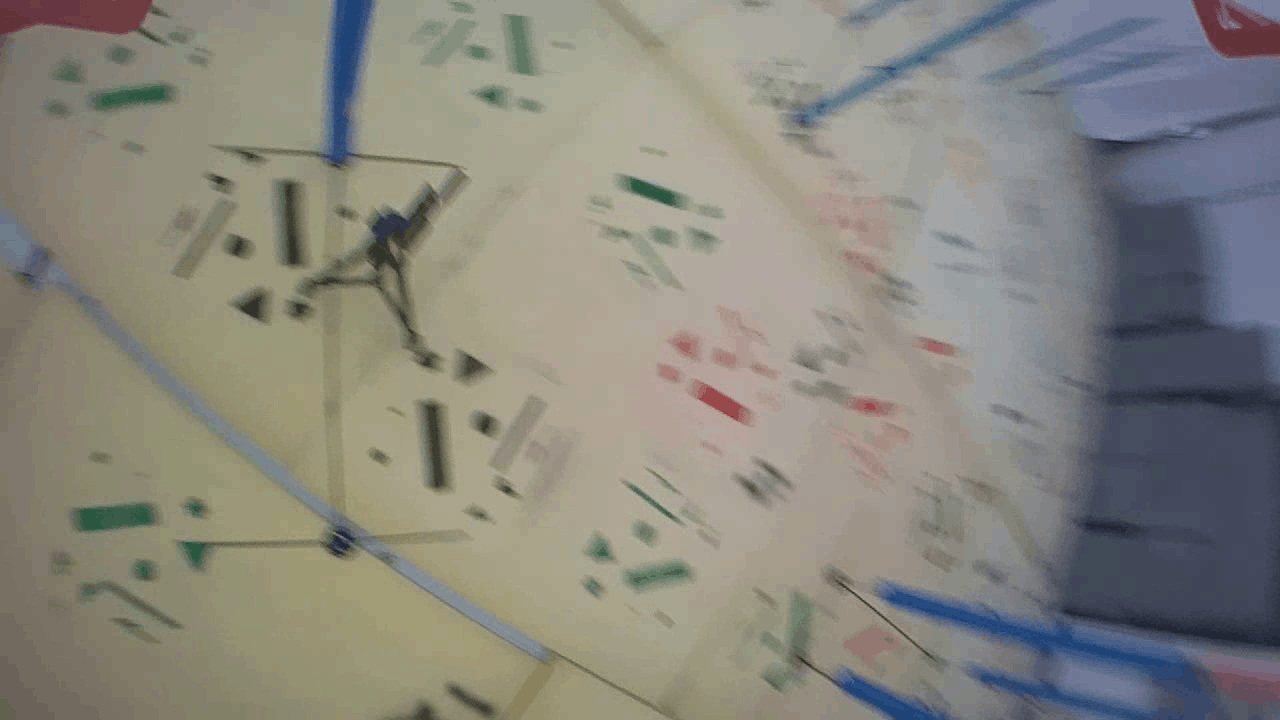}};
\node [inner sep=0pt, outer sep=0pt, right=0.0mm of img1] (img2) {\includegraphics[width=2.05cm]{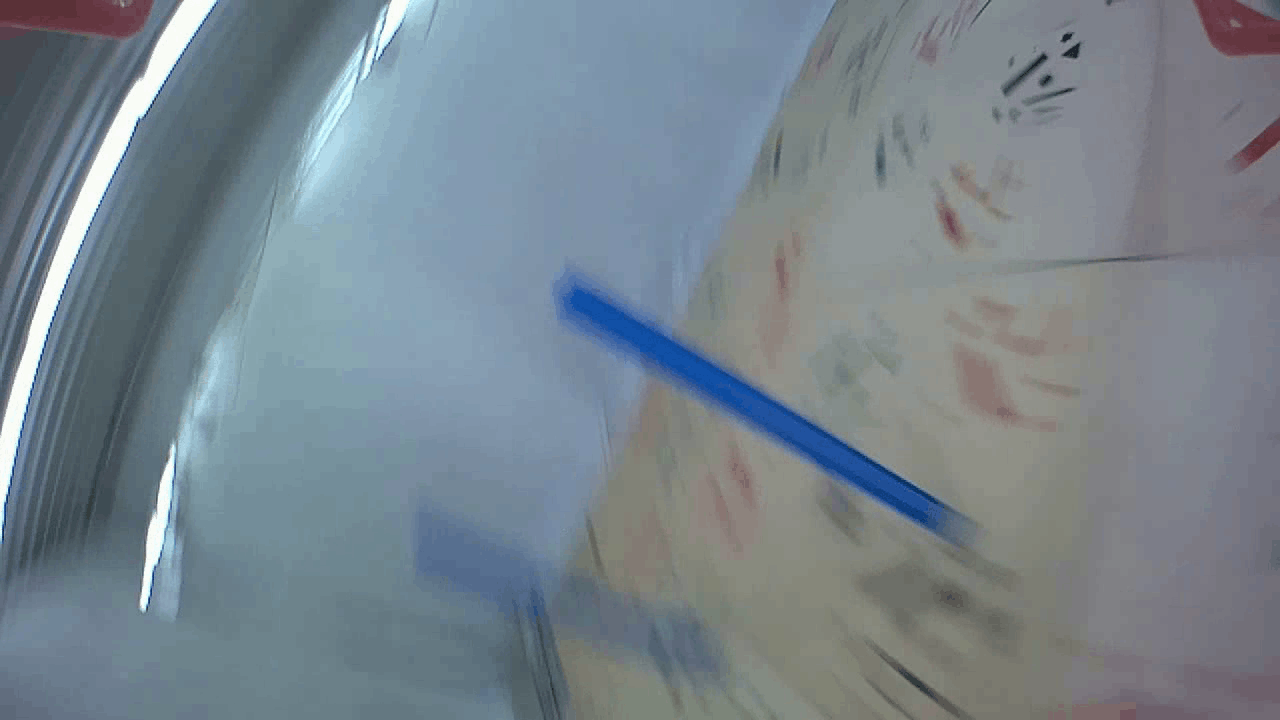}};
\node [inner sep=0pt, outer sep=0pt, right=0.0mm of img2] (img3) {\includegraphics[width=2.05cm]{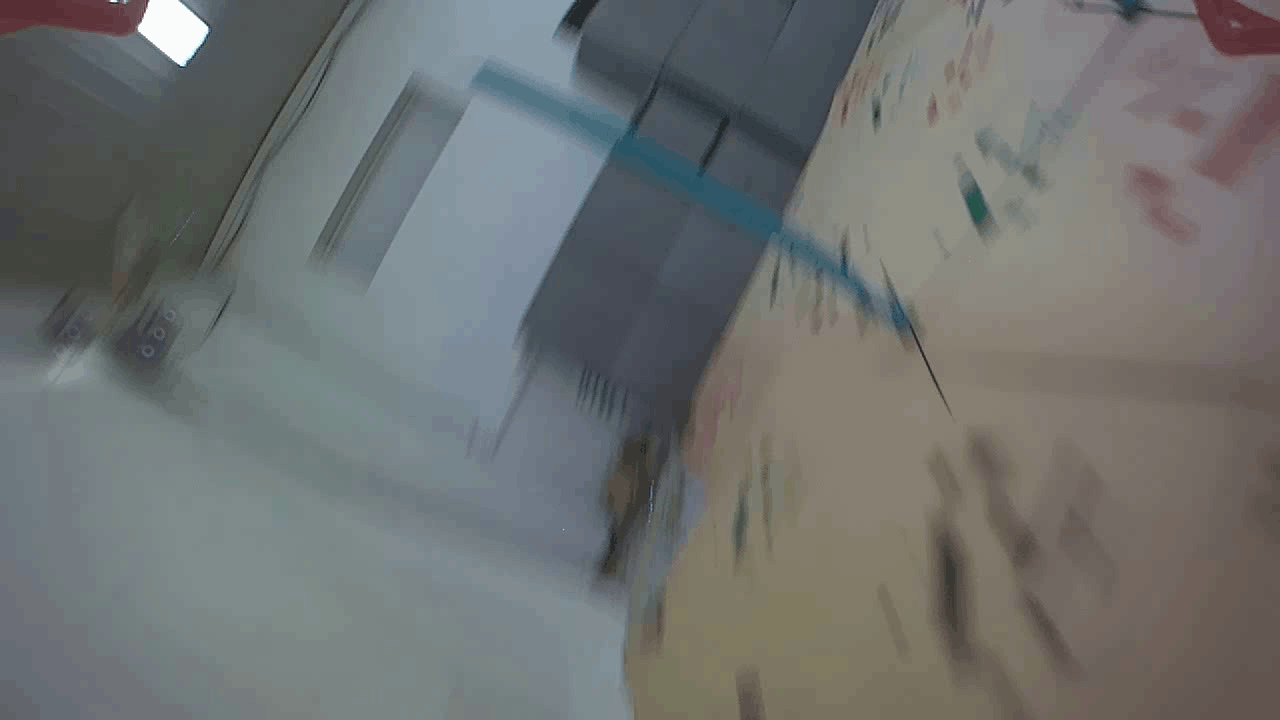}};
\node [inner sep=0pt, outer sep=0pt, right=0.0mm of img3] (img4) {\includegraphics[width=2.05cm]{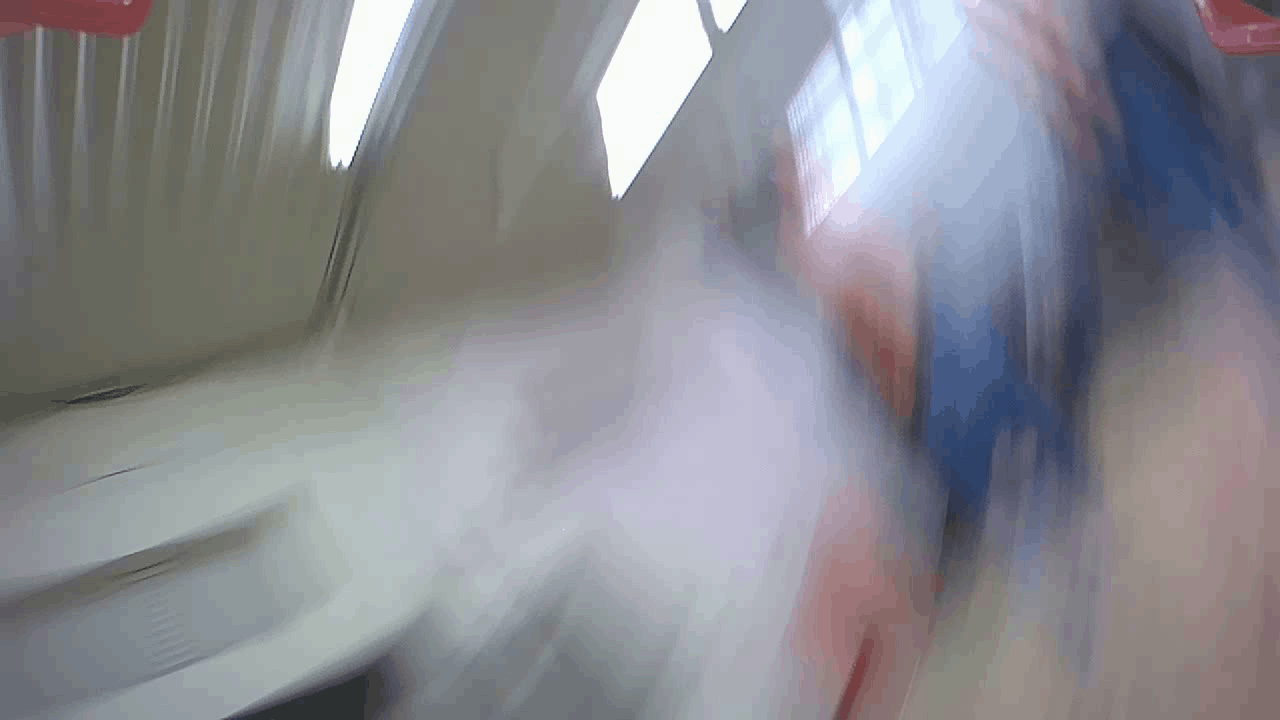}};
\node [inner sep=0pt, outer sep=0pt, right=0.0mm of img4] (img5) {\includegraphics[width=2.05cm]{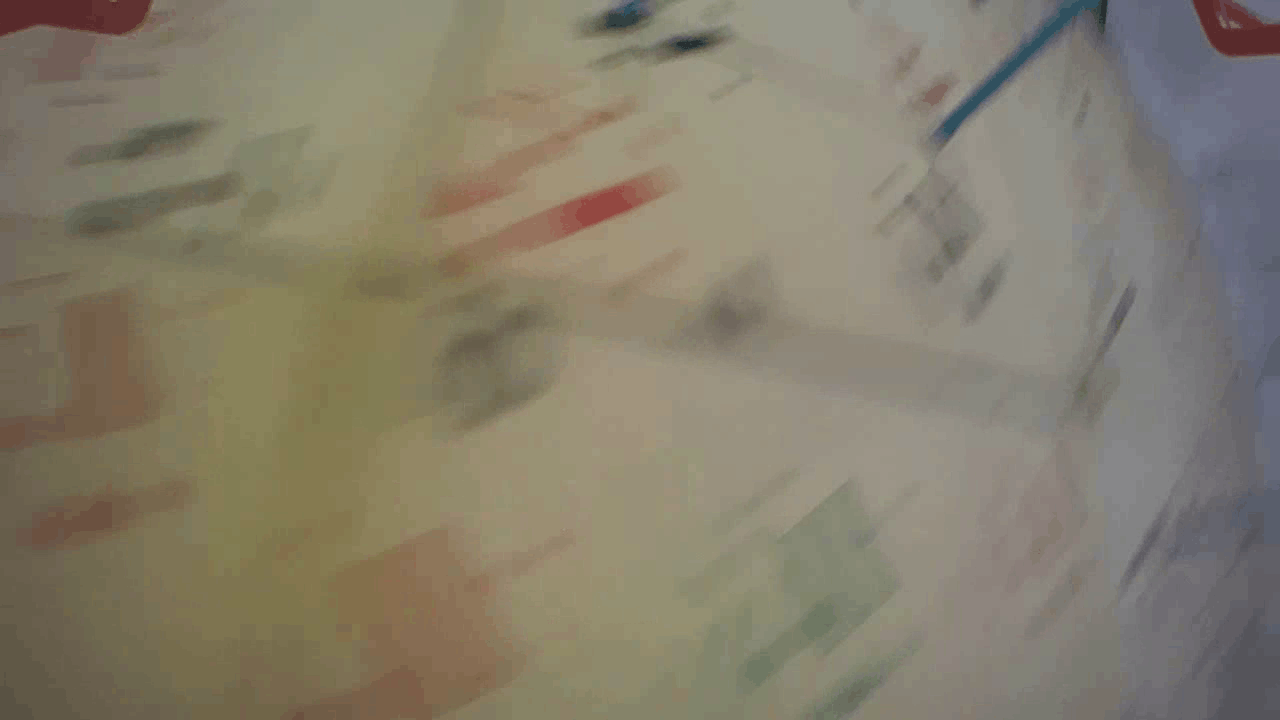}};
\node [inner sep=0pt, outer sep=0pt, right=0.0mm of img5] (img6) {\includegraphics[width=2.05cm]{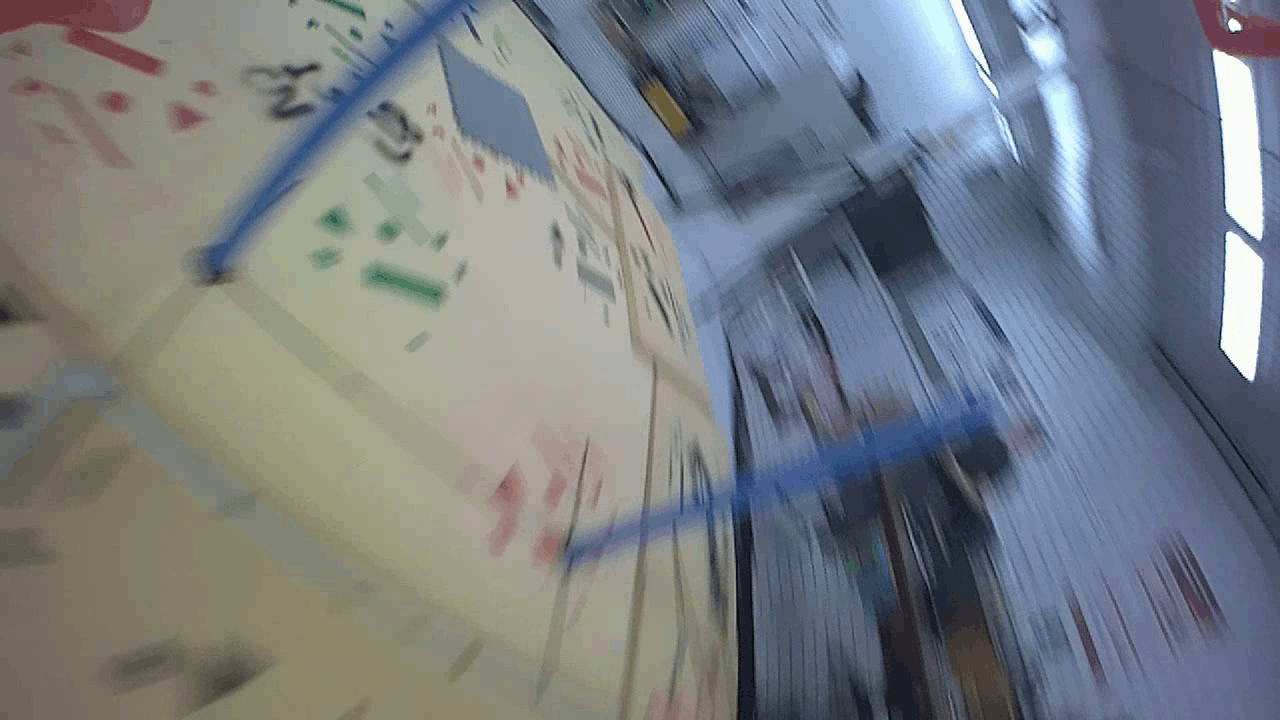}};

\node [inner sep=0pt, outer sep=0pt, below=1mm of img1] (img7) {\includegraphics[width=2.05cm]{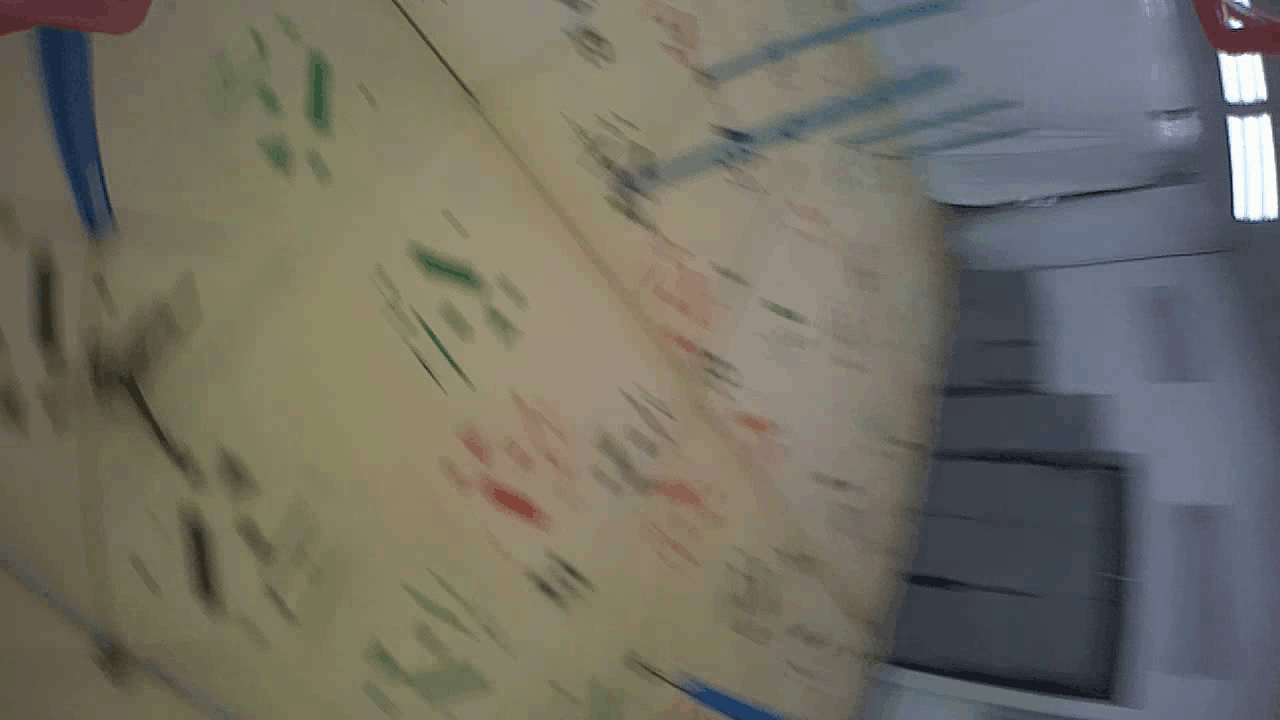}};
\node [inner sep=0pt, outer sep=0pt, right=0.0mm of img7] (img8) {\includegraphics[width=2.05cm]{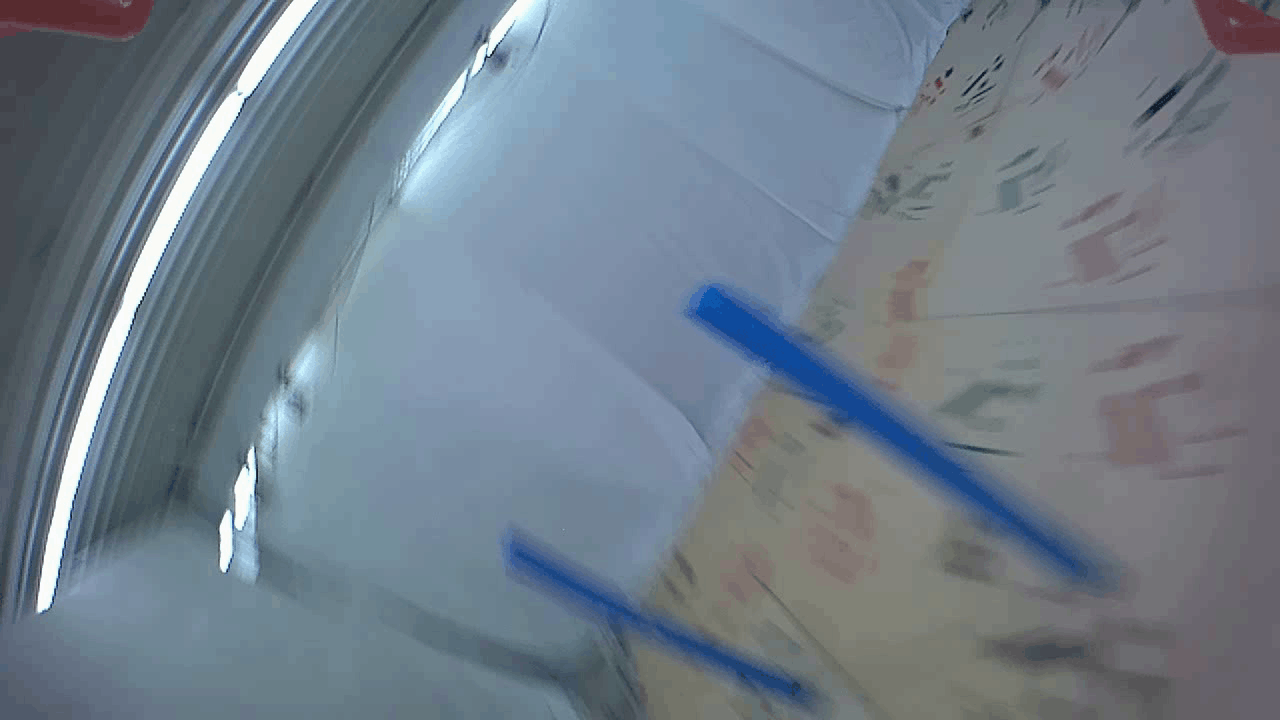}};
\node [inner sep=0pt, outer sep=0pt, right=0.0mm of img8] (img9) {\includegraphics[width=2.05cm]{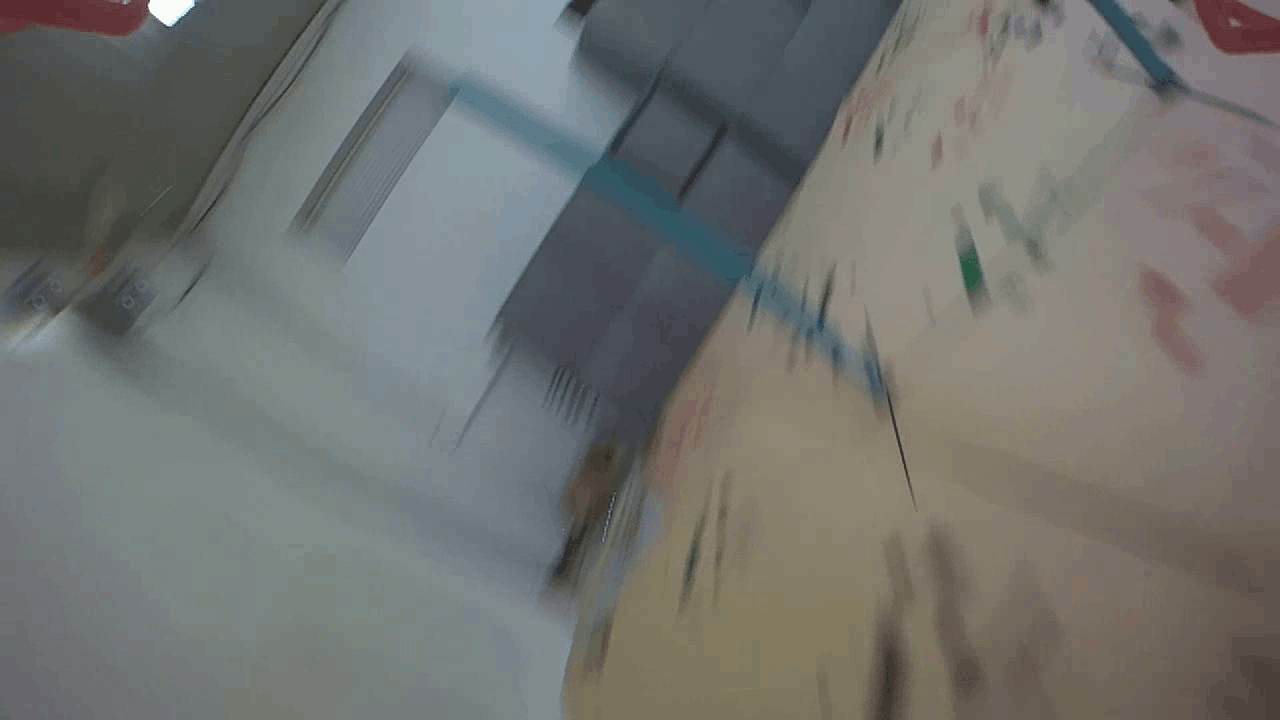}};
\node [inner sep=0pt, outer sep=0pt, right=0.0mm of img9] (img10) {\includegraphics[width=2.05cm]{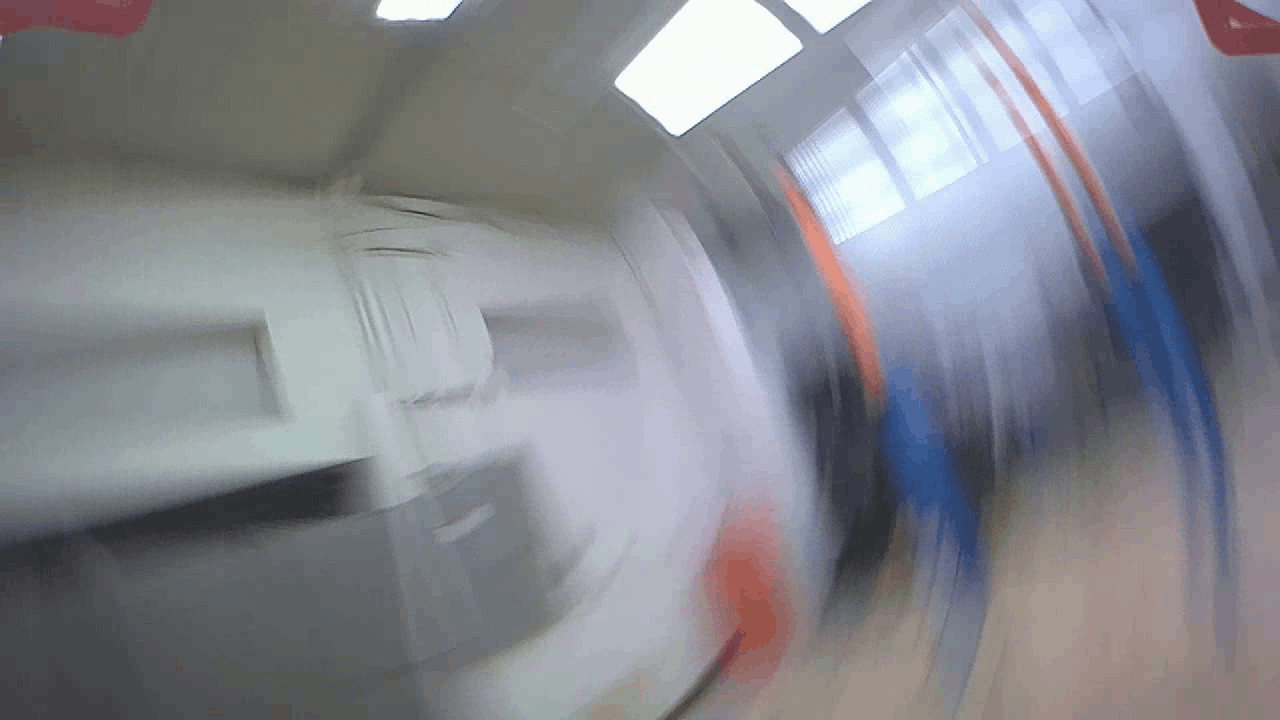}};
\node [inner sep=0pt, outer sep=0pt, right=0.0mm of img10] (img11) {\includegraphics[width=2.05cm]{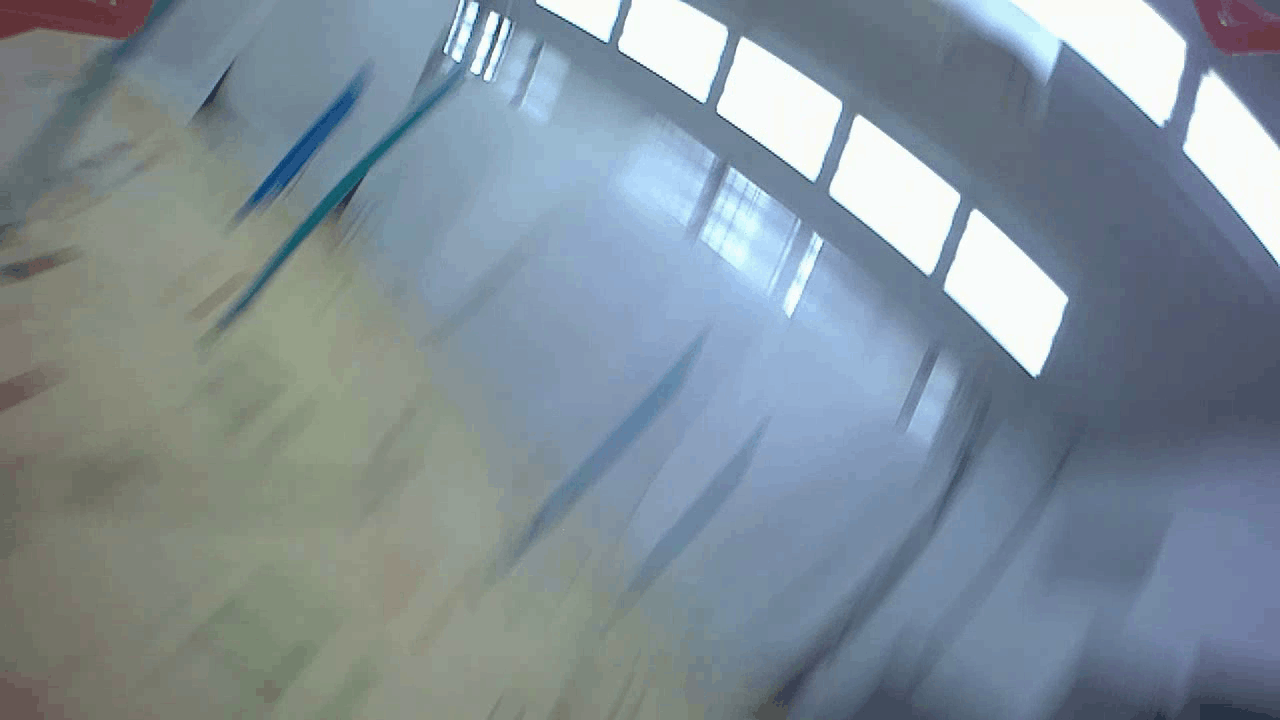}};
\node [inner sep=0pt, outer sep=0pt, right=0.0mm of img11] (img12) {\includegraphics[width=2.05cm]{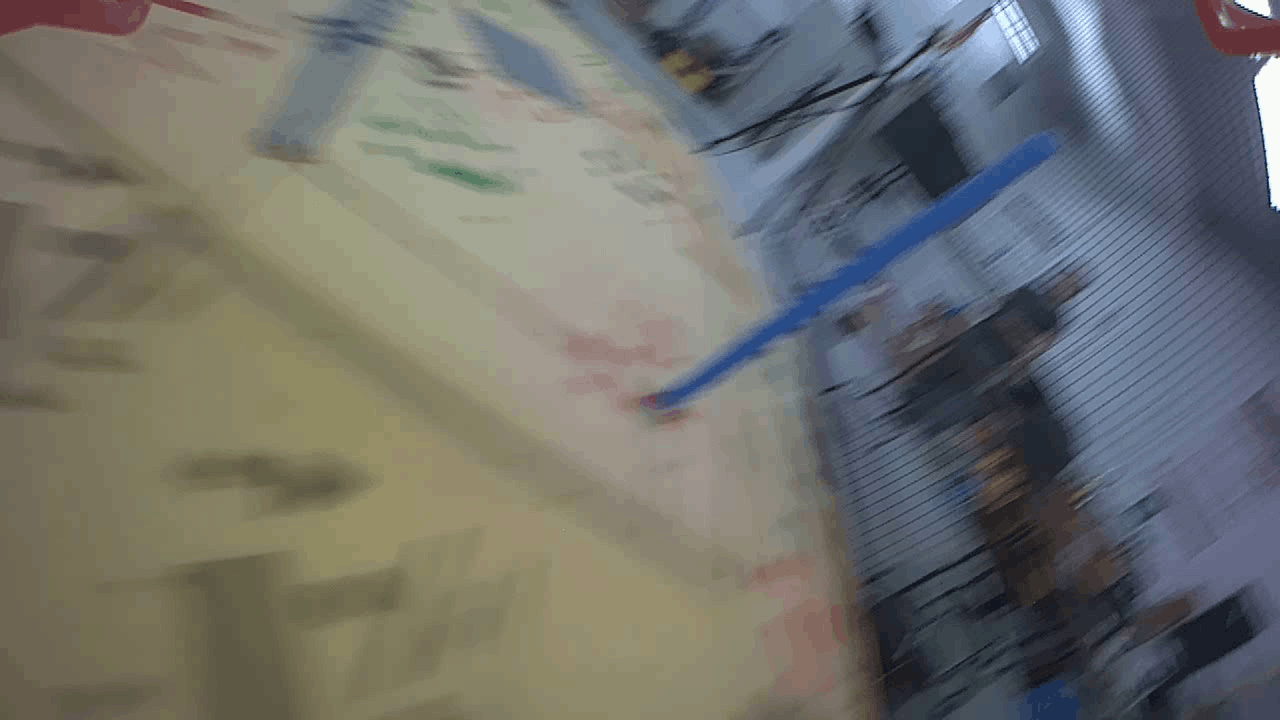}};

\node [inner sep=0pt, outer sep=0pt, below=1mm of img7] (img13) {\includegraphics[width=2.05cm]{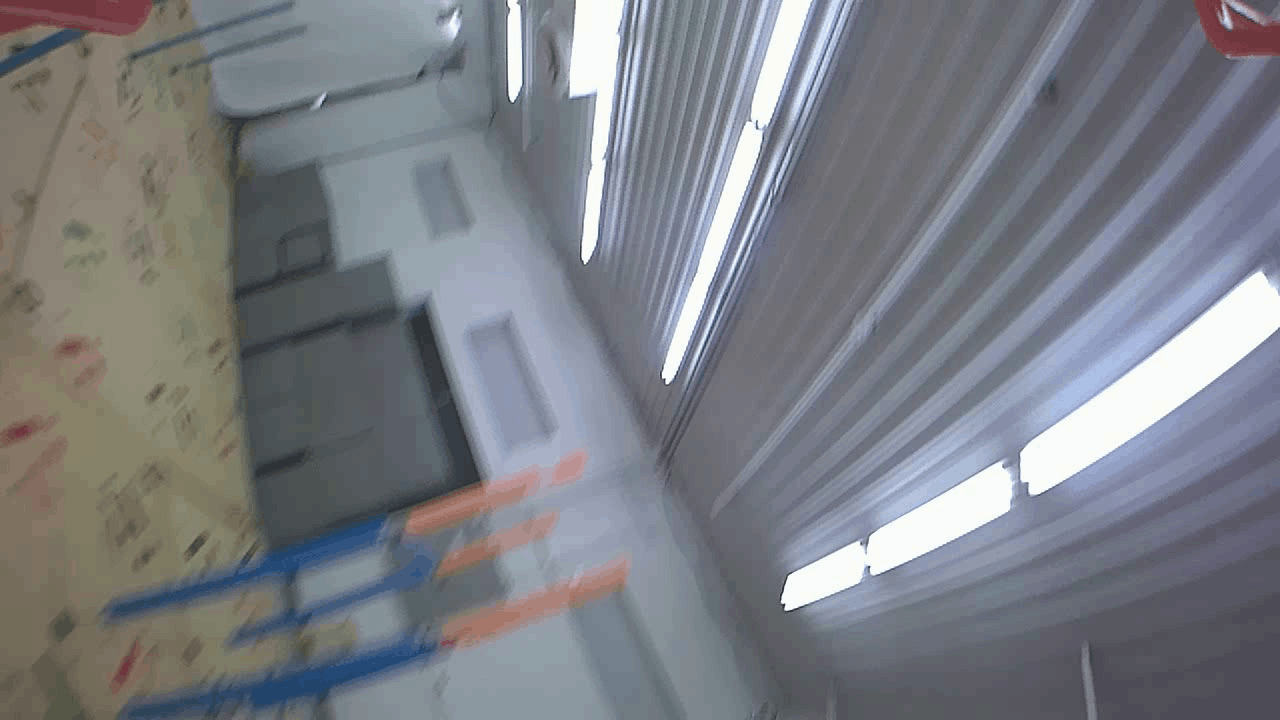}};
\node [inner sep=0pt, outer sep=0pt, right=0.0mm of img13] (img14) {\includegraphics[width=2.05cm]{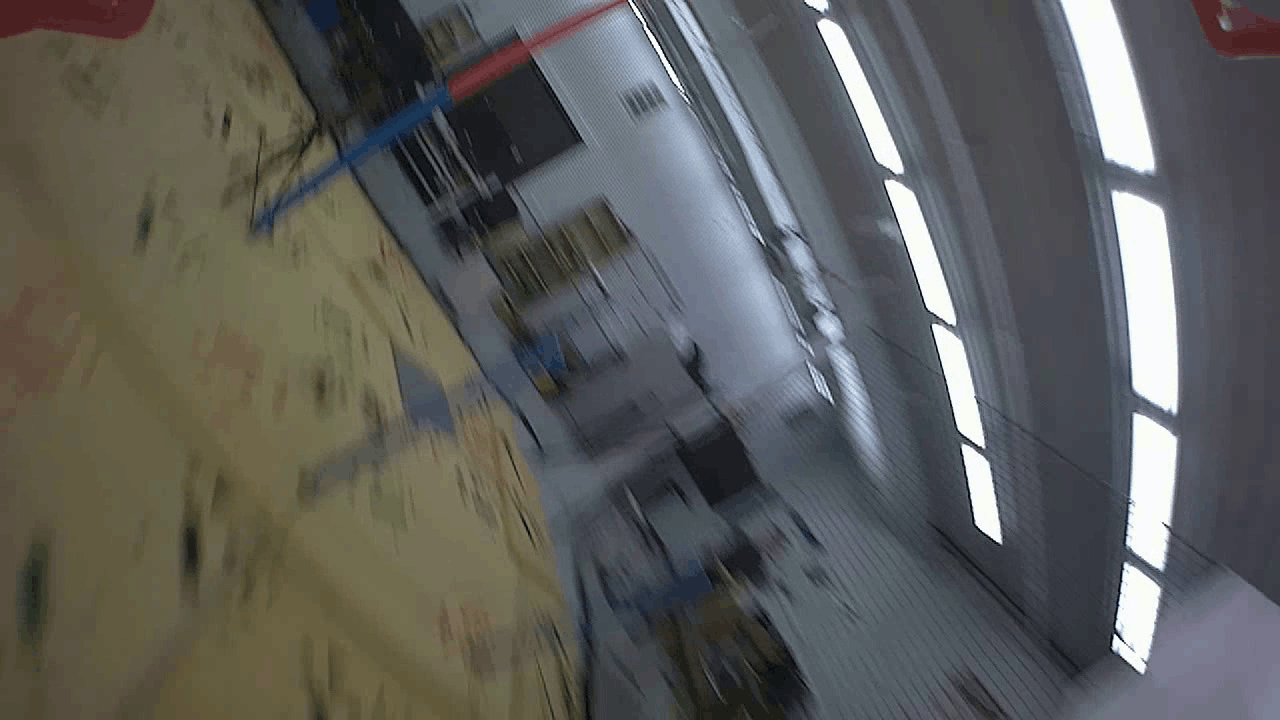}};
\node [inner sep=0pt, outer sep=0pt, right=0.0mm of img14] (img15) {\includegraphics[width=2.05cm]{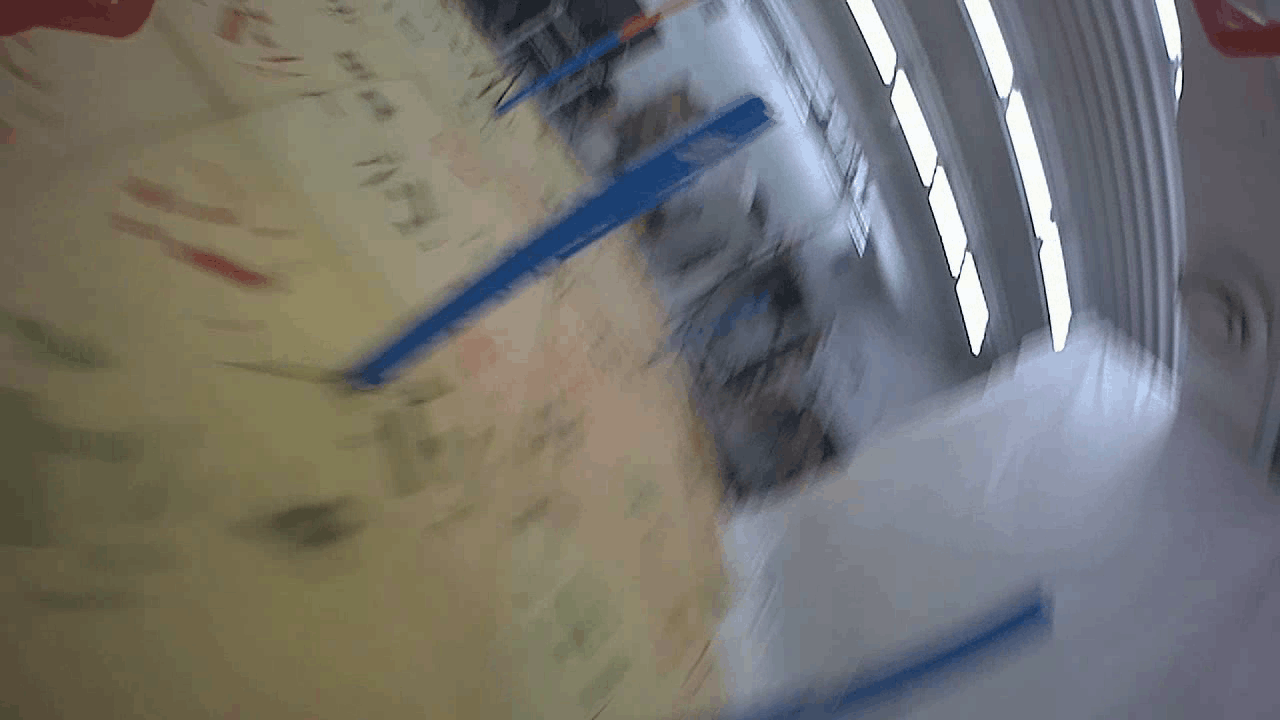}};
\node [inner sep=0pt, outer sep=0pt, right=0.0mm of img15] (img16) {\includegraphics[width=2.05cm]{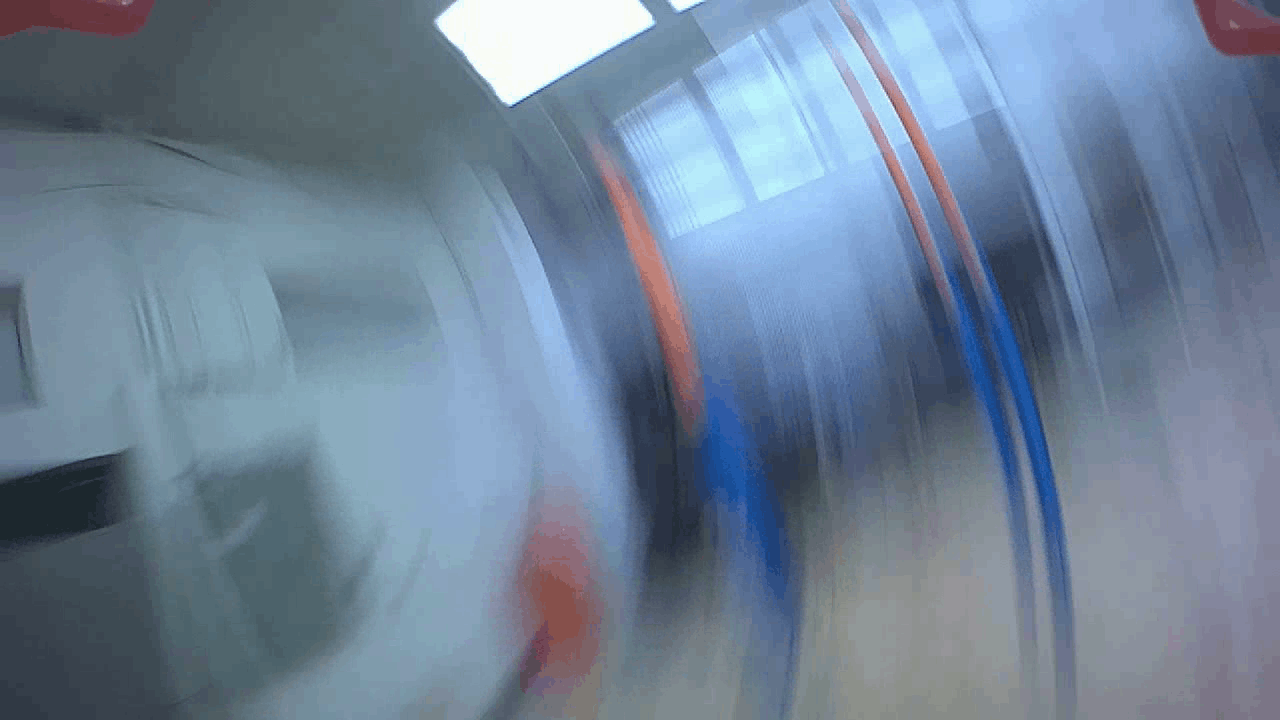}};
\node [inner sep=0pt, outer sep=0pt, right=0.0mm of img16] (img17) {\includegraphics[width=2.05cm]{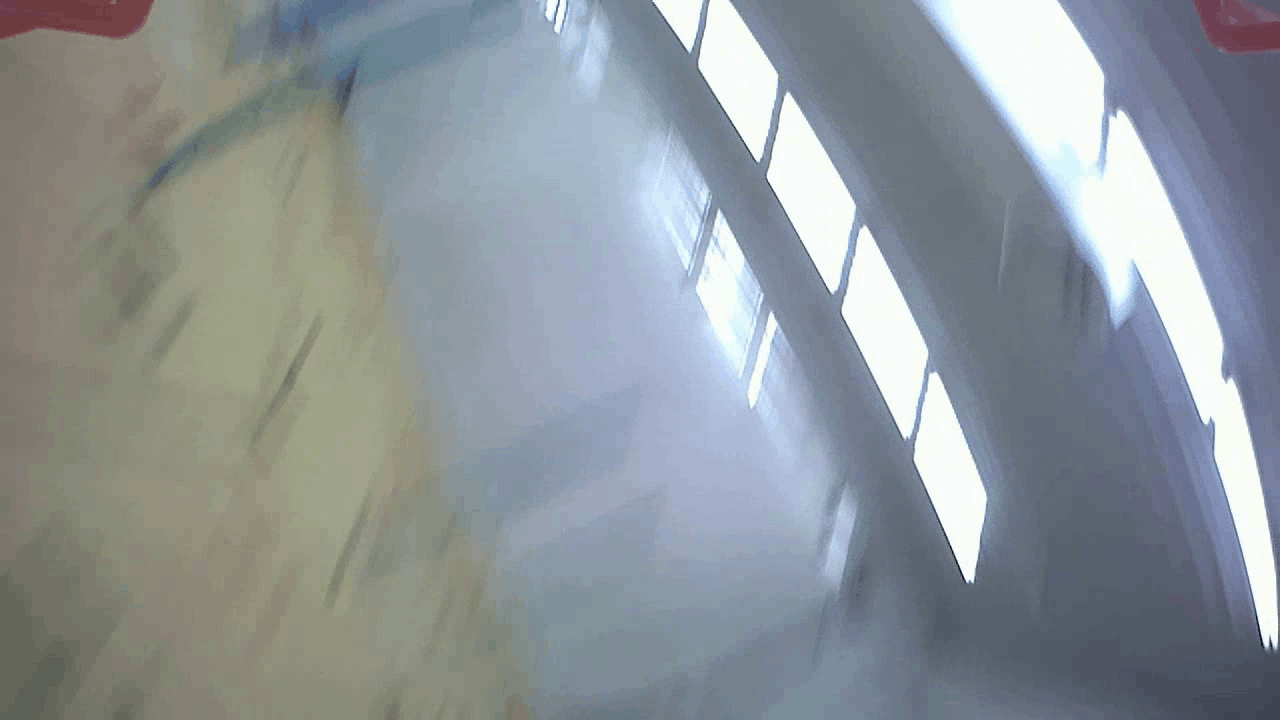}};
\node [inner sep=0pt, outer sep=0pt, right=0.0mm of img17] (img18) {\includegraphics[width=2.05cm]{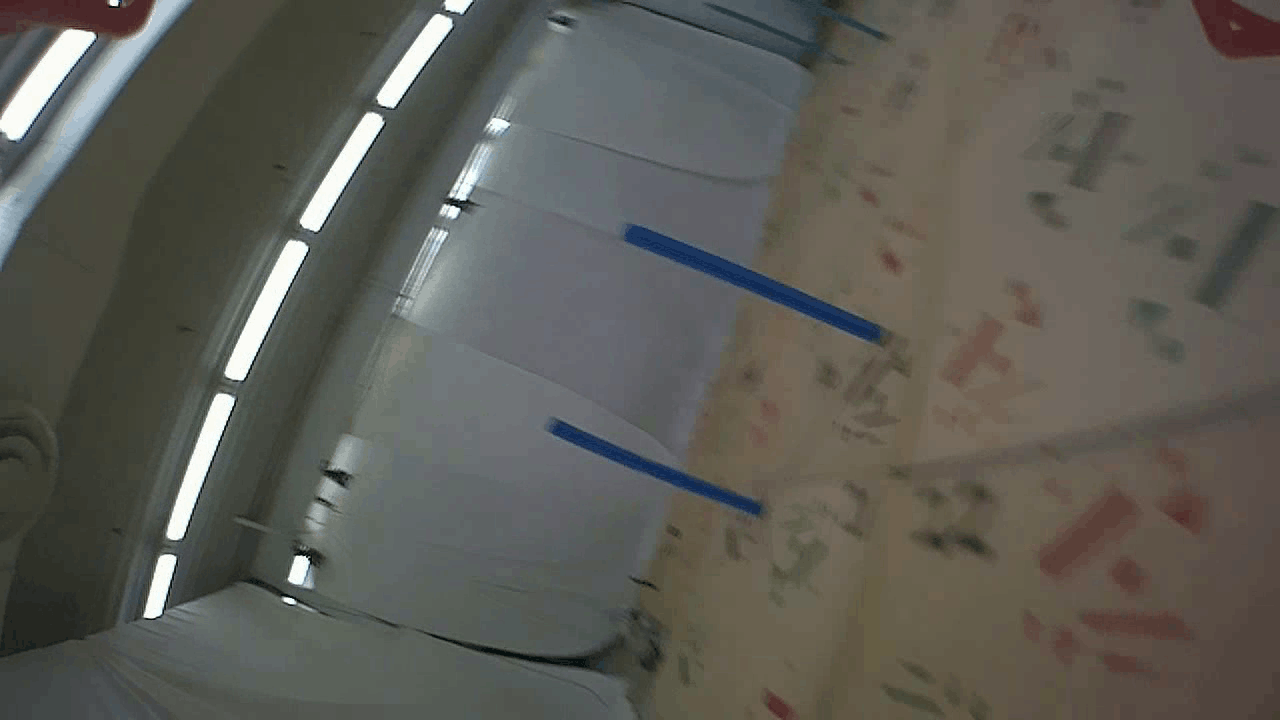}};

\node[below=1mm of img13](text1){L-1};
\node[below=1mm of img14](text2){L-2};
\node[below=1mm of img15](text3){L-3};
\node[below=1mm of img16](text1){L-4};
\node[below=1mm of img17](text2){L-5};
\node[below=1mm of img18](text3){L-6};

\node [rotate=-90, right=0.5cm of img6.north east, anchor=north west, xshift=-0.0cm](rowlabel1){(a) LA};
\node [rotate=-90, right=0.5cm of img12.north east, anchor=north west, xshift=0.0cm](rowlabel1){(b) FOV};
\node [rotate=-90, right=0.5cm of img18.north east, anchor=north west, xshift=0.0cm](rowlabel1){(c) PUM};

\end{tikzpicture}
\caption{\ac{FPV} images of perception-aware time-optimal flights at six different locations. While \ac{LA} and \ac{FOV} exhibit similar views, \ac{PUM} differs from them by looking toward directions where multiple gates are visible instead of only the upcoming one, such as locations L-1 and L-6.}
\label{fig:exp_fpv_footages}
\end{figure*}

\subsubsection{Position Uncertainty Evaluation}

In this experiment, we execute the trajectories with the \ac{LA} \tikz[baseline=-0.6ex] \node[regular polygon, regular polygon sides=3, rotate=0, draw={rgb,1:red,0.247; green,0.369; blue,0.710}, fill={rgb,1:red,0.247; green,0.369; blue,0.710}, line width=0.8pt, inner sep=0pt, minimum size=6pt] {};, \ac{FOV} \tikz[baseline=-0.6ex] \node[regular polygon, regular polygon sides=4, rotate=0, draw={rgb,1:red,0.247; green,0.369; blue,0.710}, fill={rgb,1:red,0.247; green,0.369; blue,0.710}, line width=0.8pt, inner sep=0pt, minimum size=6pt] {};, \ac{PUM} \tikz[baseline=-0.6ex] \node[regular polygon, regular polygon sides=5, rotate=0, draw={rgb,1:red,0.247; green,0.369; blue,0.710}, fill={rgb,1:red,0.247; green,0.369; blue,0.710}, line width=0.8pt, inner sep=0pt, minimum size=6pt] {};, FOV-PUM \tikz[baseline=-0.6ex] \node[regular polygon, regular polygon sides=6, rotate=0, draw={rgb,1:red,0.247; green,0.369; blue,0.710}, fill={rgb,1:red,0.247; green,0.369; blue,0.710}, line width=0.8pt, inner sep=0pt, minimum size=6pt] {};, and LA-PUM \tikz[baseline=-0.6ex] \node[regular polygon, regular polygon sides=8, rotate=0, draw={rgb,1:red,0.247; green,0.369; blue,0.710}, fill={rgb,1:red,0.247; green,0.369; blue,0.710}, line width=0.8pt, inner sep=0pt, minimum size=6pt] {}; objectives, respectively, using the \ac{MPCTC}, followed by comparing the corresponding position uncertainty along the trajectory. As shown in Fig.~\ref{fig:exp_trajs}, all trajectories are executed with a high precision, passing through exactly the center of the gate.

\begin{figure}[!htbp]
    \centering
\tikzstyle{every node}=[font=\footnotesize]
\begin{tikzpicture}[>=stealth]
\node [inner sep=0pt, outer sep=0pt] (img1) at (0,0)
{\includegraphics[width=8.0cm]{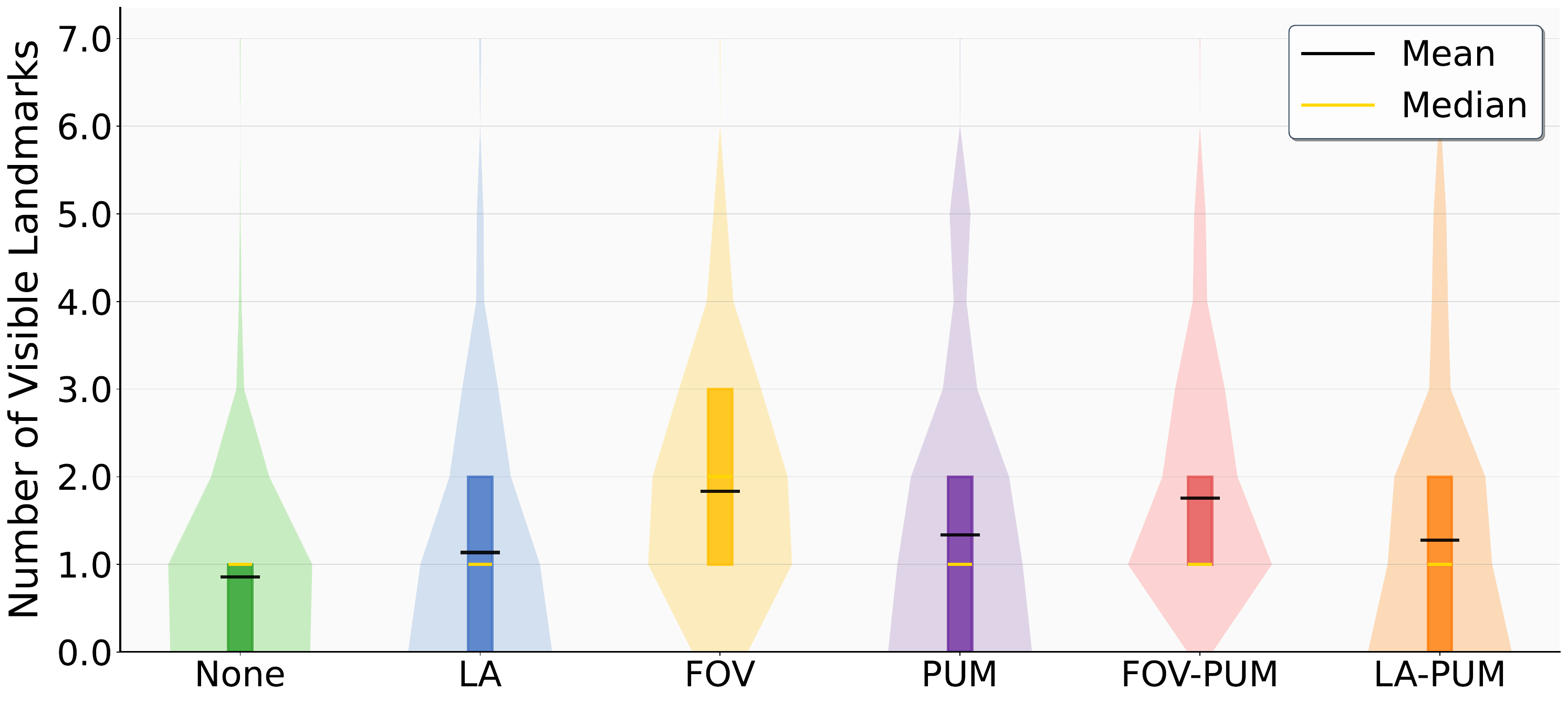}};
\node[below=1mm of img1](text1){(a) Landmark Visibility};
\node [inner sep=0pt, outer sep=0pt, below=1mm of text1] (img2)
{\includegraphics[width=8.0cm]{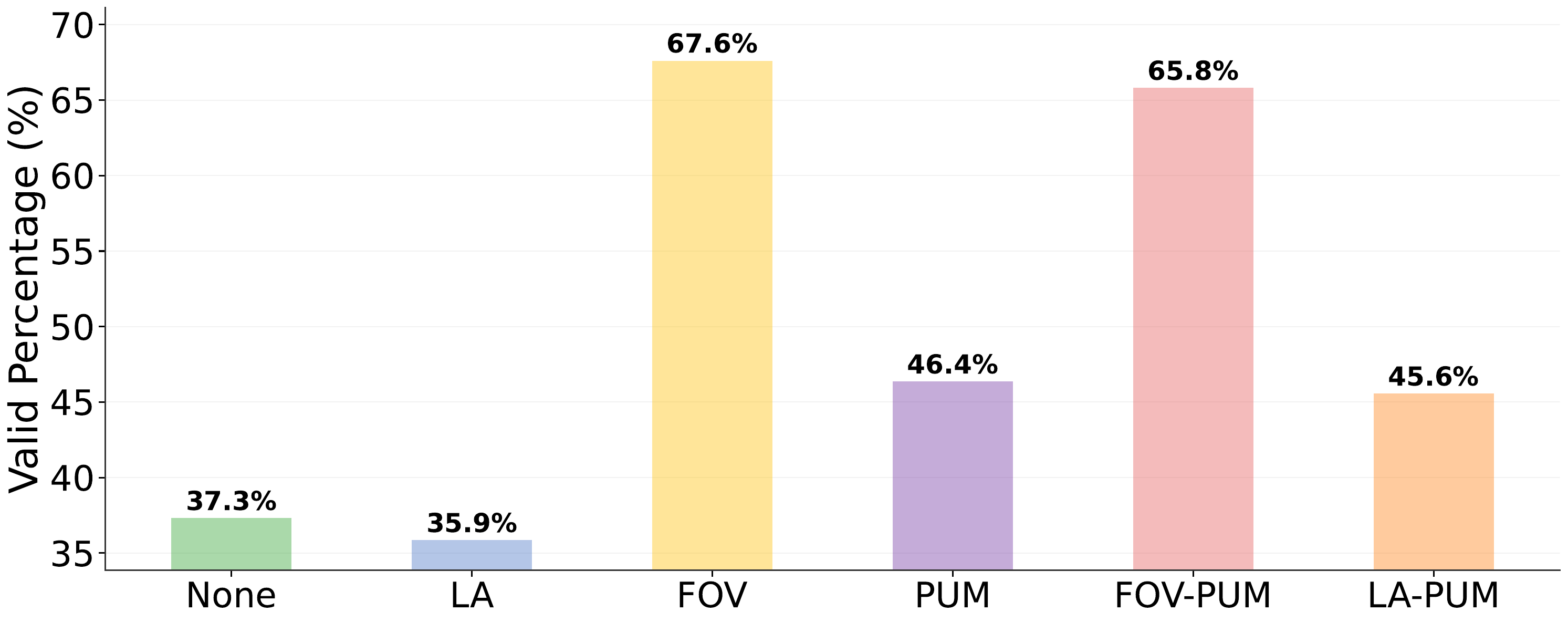}};
\node[below=1mm of img2](text2){(b) \ac{PnP} Pos. Availability};
\end{tikzpicture}
\caption{Landmark visibility and position estimate availability in the real-world experiments.}
\label{fig:exp_uncertainty}
\vspace{-0.2cm}
\end{figure}

We see from Fig.~\ref{fig:exp_fpv_footages} that since it is a complex Split-S track squeezed into a relatively small indoor environment, the quadrotor has to perform aggressive rotations to finish the track in the shortest possible time, thus creating severe motion blur in the captured images and making keeping the gate visible very challenging. In such a challenging scenario, the \ac{LA} and \ac{FOV} still manage to maintain good visibility of the upcoming gate, albeit with different objective formulations. As expected, the \ac{PUM} demonstrates a different flight strategy, looking sideways at L-1 and backward at L-2, L-3, and L-6 to see more gates. Please visit our website for detailed footages in this experiments.

Fig.~\ref{fig:exp_uncertainty} provides a quantitative evaluation of perception quality. It can be observed that in such extremely aggressive flights, \ac{FOV} turns out to be the strategy that provides the best visibility of gates, although it only considers the upcoming gate in the optimization. It reaches the lowest interval where no gates are visible, which is 12.1\%. In the FOV-PUM case, this number further reduces to 7.7\%. As a result, the percentage of sampled frames where the \ac{PnP} solver provides a reasonable output reaches a high of 67.6\% with \ac{PUM} and 65.8\% with FOV-PUM. One plausible cause is that in such a small race track where the gate distribution is dense and each trajectory segment between gates is short, the quadrotor must produce much more intense aggressive rotations than usual to keep the lap time low, which comes at the cost of undermined perception quality. In this scenario, while the \ac{FOV} constraint formulation remains effective due to its constraint nature, those cost-function-based formulations, such as \ac{LA} and \ac{PUM}, end up losing authority in determining the final maneuver, leading to unsatisfactory visual quality.

It is worth mentioning that \ac{PUM} still provides a salient improvement in perception quality over \ac{LA}. Interestingly, \ac{LA} reaches an even lower valid percentage than the baseline with no perception awareness. A direct reason for this is that \ac{LA} results in slightly longer periods with no gates visible. This highlights a fundamental limitation of \ac{LA}: the heuristic of orienting toward the future path does not guarantee sufficient gate visibility across all track geometries. In tracks involving intense and tight turns, \ac{LA} may not be a suitable strategy for guiding trajectory optimization. Meantime, the result also suggests that while \ac{PUM} may be theoretically optimal for uncertainty minimization, the final position estimate may not reach the theoretical optimum. Several factors contribute to this discrepancy, including missing key feature observations during actual trajectory execution or noise in gate corner detection that causes unanticipated large biases. Therefore, in practice, it is essential to intelligently combine these perception objectives to achieve desirable outcomes.

\section{Discussion And Conclusions}\label{sec:conclusion}

\subsection{Limitations}

This work presents two primary limitations. First, the framework currently considers only position uncertainty, assuming that orientation measurements remain accurate. While roll and pitch angles can be estimated from IMU measurements without significant bias, the yaw angle is susceptible to accumulated drift, particularly when the landmark layout is ill-conditioned or only a single, distant landmark is visible. In future work, we plan to incorporate yaw uncertainty metrics into the perception objective design.

The second limitation stems from the cascaded control structure, which cannot perfectly reproduce the planned trajectory and necessitates a trade-off between completion time and tracking accuracy. Increasing the progress weight reduces the time gap relative to the planned duration but inevitably increases tracking error. Consequently, this requires careful parameter tuning, particularly when traversing small gates or when the intended path grazes the edges of the gate frame.

\subsection{Our Approach vs. \ac{RL}}
While both this work and various \ac{RL} approaches \cite{romero2025actor, verraest2025skydreamer, song2023reaching, kaufmann2023champion} target the same task of autonomous drone racing, they address fundamentally different research problems.

\textbf{Theoretical Limits and Benchmarking:}
In essence, this paper focuses on identifying the theoretical control limits of a quadrotor using a nominal model in minimum-time tasks. This limit applies to all control strategies, including RL, provided they operate on the same model. Consequently, one of our primary contributions is a benchmarking methodology for evaluating controller performance rather than a system intended to outfly RL in the real world. 

\textbf{Generalizability and Explainability:}
RL often requires extensive retraining and hyperparameter fine-tuning when switching to a new race track. In contrast, our approach generalizes by nature and can provide a solution for any track configuration almost immediately. Furthermore, our method is explainable, while \ac{RL} operates in a black-box manner. Explainability allows for the direct inspection of the underlying optimization landscape. This transparency is crucial for identifying the specific causes of failure and for providing formal guarantees of performance

\textbf{Multi-Objective Flexibility:}
Our approach offers greater flexibility in handling diverse objectives, such as the perception-aware metrics introduced in this paper. While \ac{RL} can also incorporate a variety of reward functions, doing so increases the risk of trapping the policy in poor local minima and inevitably results in slower convergence. In comparison, our formulation inherits the advantages of traditional optimization-based planning methods, maintaining a clear path to convergence even when subjected to complex perception requirements.

In conclusion, our approach and \ac{RL} excel in distinct contexts and are not mutually exclusive. If a researcher needs to rapidly determine the theoretical minimum lap time for a track with known gate geometries and layouts, our method provides an immediate solution with an optimality guarantee. Furthermore, it can provide the expert demonstrations required to train imitation learning models, which subsequently facilitate the convergence of \ac{RL} policies.

\subsection{Conclusions}

In this paper, we presented a unified perception-aware time-optimal planning and control framework for vision-based \ac{UAV}s. By integrating three distinct perception objectives, look-ahead, sequential field-of-view constraints, and position uncertainty minimization, we demonstrated that explicitly accounting for perceptual quality significantly reduces state estimation uncertainty and enhances closed-loop robustness. Our findings highlight that FOV and PUM are particularly effective in racing scenarios, as it enables the drone to maximize the utility of environmental visual features to improve localization performance. Extensive numerical and real-world experiments are performed to validate the proposed autonomous system. These results provide a robust foundation for a wide range of time-critical applications, from high-speed drone racing to autonomous infrastructure inspection.

\section*{ACKNOWLEDGMENT}

The authors would like to acknowledge the Learning Systems \& Robotics Lab (formerly the Dynamic Systems Lab) at UTIAS, led by Dr. Angela Schoellig, for providing facility support for the laboratory flight experiments.
This work was supported by the European Union’s Horizon Europe Research and Innovation Programme under grant
agreement No. 101120732 (AUTOASSESS) and the European
Research Council (ERC) under grant agreement No. 864042
(AGILEFLIGHT).


%
\bibliographystyle{IEEEtran}
\bibliography{IEEEabrv,main}

\appendices

\section{Uncertainty in the Bearing Vector Measurements}
\label{sec:appendix_a}

The standard deviation of the bearing vector has the following relationship to that of the image coordinates:
\begin{equation}
\boldsymbol{\sigma}_{\rho}=|\frac{\partial\boldsymbol{\rho}^{c}}{\partial[u,v]^{\top}}|\left[\begin{array}{c}
  \sigma_{u}\\
  \sigma_{v}
  \end{array}\right].
\end{equation}
Expanding the Jacobian matrix by using the chain rule yields:
\begin{equation}
\frac{\partial\boldsymbol{\rho}^{c}}{\partial[u,v]^{\top}}=\frac{\partial\boldsymbol{\rho}^{c}}{\partial[x_{n},y_{n}]^{\top}}\frac{\partial[x_{n},y_{n}]^{\top}}{\partial[u,v]^{\top}}.
\end{equation}
We calculate their expressions separately:
\begin{equation}
\begin{aligned}
  \frac{\partial\boldsymbol{\rho}^{c}}{\partial x_{n}}&=\left[\begin{array}{c}
    \cos(\theta)\cdot(\frac{x_{n}}{\theta})^{2}+\sin(\theta)\cdot\left(\frac{1}{\theta}-\frac{x_{n}^{2}}{\theta^{3}}\right)\\
    \cos(\theta)\cdot\frac{x_{n}y_{n}}{\theta^{2}}-\sin(\theta)\cdot\frac{x_{n}y_{n}}{\theta^{3}}\\
    -\sin(\theta)\cdot\frac{x_{n}}{\theta}
    \end{array}\right],\\
  \frac{\partial\boldsymbol{\rho}^{c}}{\partial y_{n}}&=\left[\begin{array}{c}
    \cos(\theta)\cdot\frac{x_{n}y_{n}}{\theta^{2}}-\sin(\theta)\cdot\frac{x_{n}y_{n}}{\theta^{3}}\\
    \cos(\theta)\cdot(\frac{y_{n}}{\theta})^{2}+\sin(\theta)\cdot\left(\frac{1}{\theta}-\frac{y_{n}^{2}}{\theta^{3}}\right)\\
    -\sin(\theta)\cdot\frac{y_{n}}{\theta}
    \end{array}\right],
  \end{aligned}   
\end{equation}  
\begin{equation}\label{eq:bearing_vector_jacobian_right}
\frac{\partial[x_{n},y_{n}]^{\top}}{\partial[u,v]^{\top}}=
\begin{bmatrix}
  \frac{1}{f_{x}} & 0 \\
  0 & \frac{1}{f_{y}}
\end{bmatrix}.
\end{equation} 
When $\theta$ is small, the left matrix can be approximated as:
\begin{equation}\label{eq:approx_bearing_vector_jacobian}
\frac{\partial\boldsymbol{\rho}^{c}}{\partial[x_{n},y_{n}]^{\top}}\approx
\begin{bmatrix}
  1 & 0 \\
  0 & 1 \\
  -x_{n} & -y_{n}
\end{bmatrix}.
\end{equation}
It is noticed that (\ref{eq:approx_bearing_vector_jacobian}) indeed corresponds to the maximum entry values for $|\frac{\partial\boldsymbol{\rho}^{c}}{\partial[x_{n},y_{n}]^{\top}}|$; as $\theta$ increases from 0, the absolute values will decrease rapidly, approaching zero as $\theta$ approaches infinity. As a result, multiplying the above result with (\ref{eq:bearing_vector_jacobian_right}) gives an upper bound of the uncertainty in the bearing vector:
\begin{equation}\label{eq:bearing_jacobian}
\boldsymbol{\sigma}_{\rho} \leq [|\frac{\sigma_{u}}{f_{x}}|,|\frac{\sigma_{v}}{f_{y}}|,|\frac{\sigma_{u}x_{n}}{f_{x}}+\frac{\sigma_{v}y_{n}}{f_{y}}|]^{\top}. 
\end{equation}
Since $(c_x,c_y)$ generally represents the coordinate of the image center and hence $u_{\text{max}}\approx2c_{x}$ and $v_{\text{max}}\approx2c_{y}$, the following inequality holds:
\begin{equation}
|\frac{\sigma_{u}x_{n}}{f_{x}}+\frac{\sigma_{v}y_{n}}{f_{y}}| \leq |\frac{\sigma_{u}c_{x}}{f_{x}^{2}}+\frac{\sigma_{v}c_{y}}{f_{y}^{2}}|.
\end{equation}
Substituting this into (\ref{eq:bearing_jacobian}) yields:
\begin{equation}
\boldsymbol{\sigma}_{\rho}\leq[|\frac{\sigma_{u}}{f_{x}}|,|\frac{\sigma_{v}}{f_{y}}|,|\frac{\sigma_{u}c_{x}}{f_{x}}+\frac{\sigma_{v}c_{y}}{f_{y}}|]^{\top},
\end{equation}
which completes the proof.

\section{Proof of Proposition 1}
\label{sec:appendix_b}
Assume that $\boldsymbol{\Sigma}_\rho$ is a isotropic matrix. Inserting $\boldsymbol{\rho}^{c}=\mathbf{R}_{cw}\boldsymbol{\rho}^{w}$ into (\ref{eq:jacobian_bearing_vector}) gives:
\begin{equation}
\begin{aligned}
  \mathbf{J}&=\frac{1}{d}(\mathbf{I}_{3}-(\mathbf{R}_{cw}\boldsymbol{\rho}^{w})(\mathbf{R}_{cw}\boldsymbol{\rho}^{w})^{\top})(-\mathbf{R}_{cw})\\ 
  &=\mathbf{R}_{cw}\underset{\mathbf{A}}{\underbrace{(\mathbf{I}_{3}-\boldsymbol{\rho}^{w}(\boldsymbol{\rho}^{w})^{\top})/d}}\cdot \underset{\mathbf{I}_{3}}{\underbrace{\mathbf{R}_{cw}^{\top}\mathbf{R}_{cw}}}\\ 
  &=\mathbf{R}_{cw}\mathbf{A}.
  \end{aligned}
\end{equation}
Therefore, we can rewrite the FIM as:
\begin{equation}
\mathbf{I} = \mathbf{J}^{\top} \boldsymbol{\Sigma}_{\rho} \mathbf{J} = \mathbf{A}^{\top} (\mathbf{R}_{cw}^{\top} \boldsymbol{\Sigma}_{\rho} \mathbf{R}_{cw}) \mathbf{A} = \mathbf{A}^{\top} \boldsymbol{\Sigma}_{\rho} \mathbf{A},
\end{equation}
where the identity $\mathbf{R}_{cw}^{\top} \boldsymbol{\Sigma}_{\rho} \mathbf{R}_{cw} = \boldsymbol{\Sigma}_{\rho}$ holds true because $\boldsymbol{\Sigma}_{\rho}$ is a diagonal matrix with all diagonal entries being the same value, and $\mathbf{R}_{cw}$ is an orthogonal matrix.

\section{Proof of Proposition 2}
\label{sec:appendix_c}

We begin by considering the following geometric relationship from the camera projection model:
\begin{equation}
\frac{L}{(l/\cos\phi)}=\frac{d}{f_{x}} \Rightarrow d=\frac{L\cos\phi\cdot f_{x}}{l}.
\end{equation}
To analyze the uncertainty in depth estimation, we consider the standard error propagation formula:
\begin{equation}\label{eq:depth_uncertainty}
\sigma_{Z}=|\frac{dZ}{dl}|\cdot\sigma_{l},
\end{equation}
where $\sigma_{l}$ is the standard deviation of the measurement noise in the observed landmark size. Taking the derivative of $Z$ with respect to $l$, we obtain:
\begin{equation}
\frac{dZ}{dl}=\frac{-L\cos\phi\cdot f_{x}}{l^{2}}=\frac{-Z^{2}}{L\cos\phi\cdot f_{x}}.
\end{equation}
Now let us examine the uncertainty $\sigma_{l}$. The landmark size in the image is typically computed from the difference between two detected image features:
\begin{equation}
l=|(u_{1}+n_{u_{1}})-(u_{2}+n_{u_{2}})|
\end{equation}
where each feature location $u_{i}\sim\mathcal{N}(0,\sigma_{u})$, and $n_{u_{i}}$ denotes the associated noise. Assuming independent Gaussian noise, the standard deviation of $l$ becomes:
\begin{equation}
\sigma_{l}=\sqrt{\sigma_{u}^{2}+\sigma_{u}^{2}}=\sqrt{2}\sigma_{u}.
\end{equation}
Putting it into (\ref{eq:depth_uncertainty}), we obtain the final expression for the depth uncertainty:
\begin{equation}
\sigma_{Z}=|\frac{\sqrt{2}Z^{2}}{L\cos\phi\cdot f_{x}}|\cdot \sigma_{u}.
\end{equation}

\end{document}